\documentclass{article}
\usepackage[margin=0.75in]{geometry}
                        
\usepackage[american]{babel}

\usepackage{mathtools} 
\usepackage{booktabs} 
\usepackage{tikz} 

\usepackage{graphicx, soul, color, url}
\usepackage{multirow, amsmath, amssymb, amsthm, mathtools}
\usepackage{xcolor}
\usepackage{mdframed}
\usepackage[
backend=biber,
style=alphabetic,
sorting=ynt,
maxbibnames=99
]{biblatex}
\bibliography{arxiv_main/bib}
\usepackage{hyperref}
\usepackage{subcaption}
\usepackage{tikz}

\usepackage{amsmath, amssymb, mathtools, algorithm, soul}
\usepackage[noend]{algorithmic}

\newtheorem{theorem}{Theorem}[section]
\newtheorem{definition}[theorem]{Definition}
\newtheorem{corollary}[theorem]{Corollary}
\newtheorem{lemma}[theorem]{Lemma}
\newtheorem{example}[theorem]{Example}

\newtheorem{fact}{Fact}[section]

\newtheorem{remark}[theorem]{Remark}
\usepackage{enumitem}

\renewcommand{\Pr}[2]{\mathbb{P}_{#1}\left[  #2 \right]}
\newcommand{\E}[2]{\mathbb{E}_{#1}\left[  #2 \right]}

\title{Algorithm Design and Stronger Guarantees for \\ the Improving Multi-Armed Bandits Problem}
\author{Avrim Blum\footnote{Toyota Technological Institute at Chicago}, Marten Garicano\footnote{University of Chicago}, Kavya Ravichandran\footnote{Toyota Technological Institute at Chicago}, Dravyansh Sharma\footnote{IDEAL Institute, Toyota Technological Institute at Chicago}}
\date{}

\begin{document}
\maketitle

\begin{abstract}
    The improving multi-armed bandits problem is a formal model for allocating effort under uncertainty, motivated by scenarios such as investing research effort into new technologies, performing clinical trials, and hyperparameter selection from learning curves.
 Each pull of an arm provides reward that increases monotonically with diminishing returns. A growing line of work has designed algorithms for improving bandits, albeit with somewhat pessimistic worst-case guarantees. Indeed, strong lower bounds of $\Omega(k)$ and $\Omega(\sqrt{k})$ multiplicative approximation factors are known for both deterministic and randomized algorithms (respectively) relative to the optimal arm, where $k$ is the number of bandit arms. In this work, we propose two new parameterized families of bandit algorithms and bound the sample complexity of learning the near-optimal algorithm from each family using offline data. We also perform empirical evaluations on standard hyperparameter tuning benchmarks. The first family we define includes the optimal randomized algorithm from prior work. We show that an appropriately chosen algorithm from this family can achieve stronger guarantees, with optimal dependence on $k$, when the arm reward curves satisfy additional properties related to the strength of concavity. Our second family contains algorithms that both guarantee best-arm identification on well-behaved instances and revert to worst-case guarantees on poorly-behaved instances.\looseness-1 
\end{abstract}

\section{Introduction}

The multi-armed bandits problem has numerous practical applications, and a large body of literature is devoted to studying its various aspects. A natural use case is for modeling situations in which investment in an option increases its payoff. Consider the development of a new technology, where investing in research increases its efficacy, albeit with diminishing returns. Similarly, for machine learning models, more training time (e.g.\ more epochs in stochastic gradient descent) leads to increased training accuracy. The {\em improving multi-armed bandits} (IMAB) framework, originally formalized by \cite{heidari_tight_nodate, patil_mitigating_2023}, captures this setting.\looseness-1 \par

Recent work \cite{pmlr-v272-blum25a} has developed nearly-tight approximation guarantees for  this problem. Formally, the problem consists of $k$ bandit arms, each of which has associated with it a reward function that increases the more the arm is pulled. Work on the problem has traditionally assumed that the reward functions are concave (i.e., satisfy diminishing returns), and lower bounds  involve at least some arms that have minimally concave (i.e., linear) reward functions. However, in many practical settings, we would expect the reward functions to not only be concave but also satisfy some stronger condition on the growth rate. In this work, we study the problem of designing algorithms for the improving bandits problem that achieve stronger performance guarantees on more benign problem instances. In particular, we design families of algorithms parameterized by some parameter(s) $\alpha\,,$ and we wish to learn the ``best'' algorithms from this family.\looseness-1 \par

Consider the example of tuning hyperparameters such as learning rates for neural networks. Here each bandit arm corresponds to a value of the hyperparameter, an arm pull corresponds to running an additional training epoch for the corresponding value of the hyperparameter, and the reward function corresponds to the learning curves that capture the training accuracy as a function of the epochs (for the different hyperparameter values). Often the number of arms is very large (e.g.\ a grid of multi-dimensional hyperparameters), so existing approximation factors may be too pessimistic here. Furthermore, in many practical settings, we have access to historical data consisting of learning curves for training similar models on related tasks or datasets (similarly, we may have access to past clinical records, or data regarding click-through-rates for online advertising and recommendation). We can use these related previously-seen tasks to design our algorithm for the current instance. Formally, we assume we have access to multiple IMAB instances drawn from an unknown distribution. We seek to perform as well as the best parameter $\alpha$ in the defined algorithm family, on average over the distribution.\looseness-1

In this work, we first develop stronger approximation guarantees that depend on the strength of concavity of the reward functions. Our guarantees are achieved by a family of algorithms, parameterized by a parameter $\alpha$ that corresponds to the strength of concavity. We show that by setting $\alpha$ appropriately, we achieve the optimal approximation guarantees for every strength of concavity. In other words, if we know how ``nice'' our improving bandits instance is, we can use the appropriate algorithm from our family. On the other hand, if we do not have this information, we resort to a data-driven approach to learn the best algorithm parameter from the data. We obtain bounds on the sample complexity, that is, the number of IMAB instances sampled from the distribution that suffice to learn the best algorithm. \par

Next, we turn our attention to the best-arm identification (BAI) task. Approaches from prior work either guarantee that the best arm is exactly recovered on a sufficiently ``nice'' instance (using a notion of niceness that is different from the strength of concavity~\cite{mussi2024best}), or select an arm such that the reward of the selected arm is a good approximation to the reward of the best arm on any (worst-case) instance~\cite{pmlr-v272-blum25a}. However, the former may suffer from sub-optimal approximation ratios on more challenging instances, while the latter may fail to recover the exact best arm on the nicer instances. We propose a hybrid algorithm which resolves this gap in the literature and obtains a best-of-both-worlds guarantee. That is, on a nice instance, our algorithm will recover the best arm, while still guaranteeing the optimal approximation factor (up to constants) on the worst-case instance. We further show how to tune the parameters of our hybrid algorithm to obtain the near-optimal parameters on typical instances for a fixed problem domain, by giving bounds on the sample complexity of data-driven algorithm design.

\subsection{Our Contributions}
\begin{enumerate}
[leftmargin=*,topsep=0pt,partopsep=1ex,parsep=1ex]\itemsep=-4pt
    \item We introduce a parameter $\beta$ to measure the ``strength'' of concavity of reward functions and develop algorithms with optimal approximation ratios for each $\beta$ (Sections~\ref{subsec:algfam-ptrr} and \ref{sec:sharper CR}). For parameter $\beta\in(0,1]$, we show that the optimal approximation ratio is $O(k^{\beta/(1+\beta)})$, which gives an improvement over the worst-case optimal bounds of $O(\sqrt{k})$ from prior work whenever $\beta<1$.\looseness-1
    \item Our algorithms that achieve the optimal guarantees constitute a parameterized family. In Section~\ref{sec:sample-complexity}, we show how to tune the algorithm parameter, given access to similar instances drawn from a fixed but unknown distribution over IMAB instances. We employ techniques from data-driven algorithm design to bound the sample complexity of configuring our algorithm given reward curves of ``training'' instances. 
    As in typical data-driven algorithm design settings, our results are optimal on average over the distribution. However, we can additionally reason about instances on which we can get strong per-instance guarantees due to the relationship between our algorithm family and sufficient conditions for stronger guarantees.\looseness-1 
    
    Our approach of designing a parameterized family of assumptions (that includes worst-case instances for a fixed value of the parameter) and a corresponding tunable family of algorithms that give optimal algorithms for each value of the assumption parameter is novel in the context of data-driven algorithm design, and may be useful to design algorithms with powerful expected-case as well as per-instance worst-case guarantees in the context of other algorithm design problems.
    \item In Section~\ref{sec:bothworlds}, we study best-arm identification (BAI) for improving bandits. Previous literature has a gap for this problem. Namely, there are algorithms that achieve exact BAI on certain ``nice'' instances but with sub-optimal worst-case performance~\cite{mussi2024best}, and the algorithms that achieve the near-optimal worst-case approximation~\cite{pmlr-v272-blum25a} fail to identify the best arm on these easier instances. In Sections~\ref{sec:hybriddef} and \ref{sec:hybridguarantees}, we propose a hybrid approach that switches between an algorithm that explores many arms and our parameterized algorithm family above and obtains best-of-both-worlds guarantees. We also  bound the sample complexity of simultaneously tuning the switching time of our hybrid approach and the concavity-strength  parameter.
\end{enumerate}

\textbf{Related Work.} We improve the tight worst-case bounds on $\tilde{\Theta}(\sqrt{k})$ on the competitive ratio of IMAB by exploiting the strength of concavity of instances. Prior work \cite{metelli_stochastic_2022,mussi2024best} also gives regret bounds for more benign IMAB instances. We develop algorithms that achieve a best-of-both-worlds guarantee by simultaneously achieving low regret on benign instances and asymptotically optimal competitive ratios on worst-case instances. We also extend techniques from the data-driven algorithm design framework~\cite{balcan2020data,sharma2025offline} to the IMAB setting to bound the sample complexity to tuning the parameters in our algorithms. See Appendix \ref{appendix:additional-related} for a more detailed discussion on the related prior literature.

\section{Preliminaries}
\label{sec:preliminaries}
We formally define the improving multi-armed bandits (IMAB) framework introduced in \cite{heidari_tight_nodate}. 

\begin{definition} \label{defn:imab}
    An instance $I \in \mathcal{I}$ of the improving multi-armed bandits problem consists of $k$ arms, each of which has associated with it a reward function $f_i$ that is nondecreasing 
    as a function of $t_i\,,$ which is the number of times that arm has been pulled so far. Following \cite{heidari_tight_nodate, patil_mitigating_2023, pmlr-v272-blum25a}, we further assume the arm rewards follow {\em diminishing-returns} (also referred to as concavity of reward functions), i.e., $f_i(t+1)-f_i(t) \le f_i(t)- f_i(t-1)$ for each reward function $f_i\,.$\looseness-1
\end{definition}

\noindent 
 Let $f^*$ denote the best arm at horizon $T$ (known) in terms of cumulative reward.
Let
 $\text{OPT}_t := \sum_{n=1}^t f^*(n).$ Note that the optimal strategy for an instance of this problem is to play the arm with highest cumulative reward for all time.

Typically, in bandits problems, we are interested in minimizing regret compared to an optimal policy, i.e., the difference between the reward due to the used policy and the optimal policy. 
For this problem there are simple instances that show that getting non-trivial regret is impossible, 
and therefore we instead study the competitive ratio of the reward. 


\begin{definition}
    Suppose on a fixed instance an algorithm accrues reward $ALG$ and the reward of the optimal policy is $OPT.$ Then, an algorithm is said to achieve competitive ratio $c$ if $OPT/ALG \le c\,$ for every IMAB instance.\looseness-1
\end{definition}

\noindent Prior work provides upper bounds (algorithms) and lower bounds (hard instances) for this problem when optimizing for competitive ratio assuming the reward functions are concave. In particular, \cite{patil_mitigating_2023} consider deterministic algorithms and show tight upper and lower bounds at $\Theta(k)\,.$ In a recent work, \cite{pmlr-v272-blum25a} show that randomization helps achieve a sharper $O(\sqrt{k}\log k)$ competitive ratio ($O(\sqrt{k})$  when $f^*(T)$ is known to the algorithm), which they complement with an $\Omega(\sqrt{k})$ lower bound.

\paragraph{Data-driven Algorithm Design Perspective and Problem Statement.}
In Sections~\ref{sec:sample-complexity} and \ref{sec:hybridcomplexity}, we take a {\em data-driven algorithm design} perspective, i.e., we show that we can use samples from a distribution over instances to {\em adapt} to the setting at hand and get better guarantees. First, we describe the context and motivation for this perspective. 

\begin{figure}
    \centering
    \includegraphics[width=0.9\linewidth]{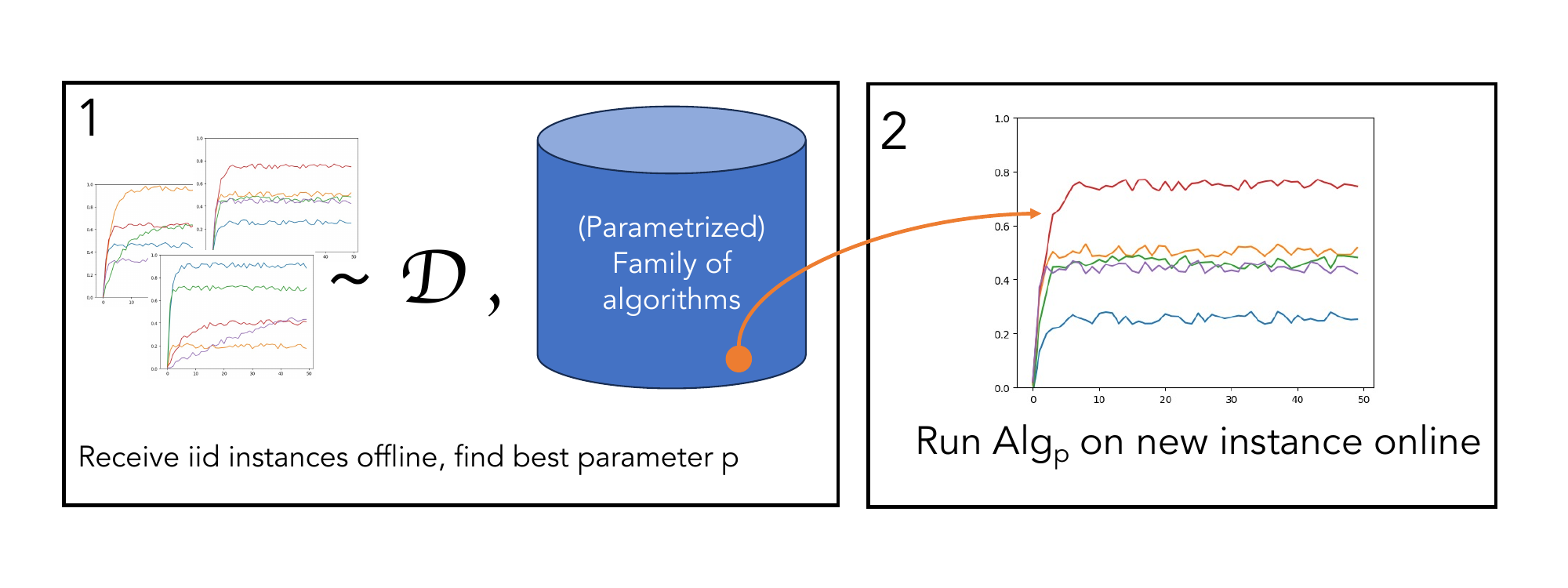}
    \caption{This figure summarizes the framework studied in this work. In phase 1 (left), the learning algorithm receives instances sampled iid from some distribution. It uses its {\em offline} access to these instances to select the best algorithm from a parameterized family of algorithms. This algorithm is the version with parameter $p\,.$ In the phase 2 (right), the algorithm with the value of the parameter set to $p$ is run {\em online} on a new instance. Thus, we can see why we call this framework ``offline-to-online transfer.''}
    \label{fig:placeholder}
\end{figure}

\paragraph{Framework.}
Suppose that we are in a setting where we must train and deploy a new machine learning model for a prediction problem each month. Each month, we must select hyperparameters and then train a deployable model with those hyperparameters. How can we leverage historical data to solve the hyperparameter selection problem? We neither want to assume that the best model on a given day is the best model on another day, nor even that the best hyperparameters for one day are the best hyperparameters for another day. Instead, let us simply assume that the algorithm we use for hyperparameter selection should be similar across days, and we wish to learn the best such algorithm. Now suppose we have historical data from January to June, and we wish to use that {\em offline} data to learn which hyperparameter selection algorithm is most appropriate to deploy in July. Further, assume we can inspect the available data {\em completely}, i.e., for each model and hyperparameter setting trained in January through June, we have access to the {\em full} learning curves. Then, our goal will be to learn a good algorithm from a family of algorithms. In this work, we identify relevant families of algorithms to solve the problem of {\em identifying the best bandit algorithm for future hyperparameter tuning} and study what happens when we choose an algorithm from this family by using empirical risk minimization (ERM) with respect to the offline instances. Finally, with this ``best'' algorithm identified, we deploy it on future hyperparameter tuning problems. We conclude that if future instances are drawn from the same distribution as the training instances, then in expectation over the distribution, we will do well on the future instances.\looseness-1

Now, we formally define the preliminaries for our study of the improving multi-armed bandits problem in the data-driven algorithm design framework following prior work on stochastic bandits~\cite{sharma2025offline}.
We consider loss functions of an algorithm in an algorithm family on an instance, defining $l(I, \pmb{\alpha})$ as the loss of the algorithm with the parameters fixed to $\pmb{\alpha} \in \mathcal{P}$ on instance $I\,.$ We also define the dual of a loss function.\looseness-1


\begin{definition}
    A loss function $\ell: \mathcal{I} \times \mathcal{P} \rightarrow \mathbb{R}$ is {\em piecewise-$H$-bounded} if it is piecewise-constant and has a  bounded range $[0, H]$.
    We fix the instance and consider the loss on a fixed instance as a function of the algorithm. Since the algorithm arises from a parameterized family, the loss is  a function of the parameter vector. In particular, the {\em dual of the loss function} is given by: $l_T^{I}(\pmb{\alpha}) = \ell_T(I, \pmb{\alpha})\,.$ 
\end{definition}

\noindent Finally, we formally define the problem we study. 

\begin{definition}[Hyperparameter Transfer Setting] \label{defn:hypertransf}
    Suppose we have a distribution $\mathcal{D}$ over $\mathcal{I}\,,$ the space of instances $I$ of the improving multi-armed bandits problem as defined in Definition~\ref{defn:imab}. Consider a family of algorithms parameterized by a vector of parameters $\pmb{\alpha} \in \mathcal{P}$. Finally, consider a piecewise-$H$-bounded loss, $l.$ We achieve sample complexity $N(\epsilon,\delta)$ in the {\em Hyperparameter Transfer Setting} if, for any $\epsilon,\delta\in(0,1)$, given $N(\epsilon,\delta)$ instances sampled iid from $\mathcal{D}\,,$ we identify 
    $
    \hat{\pmb{\alpha}}$ such that with probability $1-\delta$,  
    \begin{equation} \label{eqn:goal} 
    \left |\E{I \sim \mathcal{D}}{l_T(I, \hat{\pmb{\alpha}})} - \min_{\pmb{\alpha} \in \mathcal{P}} \E{I \sim \mathcal{D}}{l_T(I, \pmb{\alpha})} \right | < \epsilon\,.
    \end{equation}
\end{definition}

\section{Using Strength of Concavity
}\label{sec:intro3}

In this section, we introduce a family of algorithms designed to address the question of providing sharper bounds for the improving multi-armed bandits problem. The design of this family is motivated by two factors: (1) the family must contain the algorithm that is known to provide the worst-case near-optimal guarantee of \cite{pmlr-v272-blum25a}; (2) there must exist an algorithm in the family that performs better than the general worst-case if instances satisfy a stronger regularity condition. In Section~\ref{subsec:algfam-ptrr}, we define the family. Then, in Section~\ref{sec:sharper CR}, we define a strengthening of concavity and show that a variant of \cite{pmlr-v272-blum25a} achieves optimal competitive ratio guarantees on instances that satisfy the strengthened property. In Section~\ref{sec:sample-complexity}, we show that we can learn the best algorithm from this family for instances arising from a distribution with polynomially-many samples. Finally, in Section ~\ref{sec:empirical}, we present empirical evidence on real learning-curve data that different instances prefer different values of $\alpha$, corroborating the intuition that adapting the choice of this parameter to data is valuable.\par

The data-driven perspective we take in this section differs from standard worst case guarantees in several ways. On the one hand, we do not have to make stronger assumptions to get the guarantees and can instead {\em adapt} to the distribution of data. On the other hand, we must be in a setting where such a distributional assumption holds and we have access to other instances from the distribution. \looseness-1

\subsection{Algorithm Family}
\label{subsec:algfam-ptrr}

\noindent 
 Each algorithm in the family, $\textit{PTRR}_\alpha$ for fixed $\alpha$, is a slightly modified version of Algorithm $1$ in \cite{pmlr-v272-blum25a}. While playing arm $i$, we keep pulling it as long as 
$
f_i(t_i)\geq m\Bigl(\frac{t_i}{\tau}\Bigr)^\alpha,
$
where $\tau$ is an internal horizon parameter 
(set to $T-k$ when $T$ is known), $m$ is an approximation of the maximum final pull, and $t_i$ denotes the number of pulls of $i$ so far. 
When the inequality first fails, we abandon $i$ and move to a uniformly random new arm, repeating until time $T$. In this section, we focus on the cumulative reward $R$ accumulated by the algorithm. We will study the estimated best arm $\hat{i}$ in Section~\ref{sec:bothworlds}.\looseness-1

\begin{algorithm}[H]
\caption{$\textit{PTRR}_\alpha$} \label{alg:ptrr}
\begin{algorithmic}[1]
            \REQUIRE estimated max final pull $m$,  internal horizon  $\tau$
\STATE $t \gets 0$, $R \gets 0$, $S \gets \{1, \ldots, k\}$
\WHILE{$t < T$}
  \STATE \textbf{sample} $i$ uniformly from $S$
  \STATE $S \gets S \setminus \{i\}$,  $t_i \gets 0$
  \WHILE{$f_i(t_i) \geq m\Bigl(\frac{t_i}{\tau}\Bigr)^\alpha$}
    \STATE \textbf{pull} arm $i$
    \STATE $t_i \gets t_i + 1$, $t \gets t + 1$, $R \gets R + f_i(t_i)$
  \ENDWHILE
\ENDWHILE
\RETURN $R, \hat i\gets \text{argmax}_i f_i(t_i)$ 
\end{algorithmic}
\end{algorithm}

\begin{figure}
    \centering
    \begin{subfigure}{0.45\textwidth}
    \includegraphics[width=0.9\textwidth]{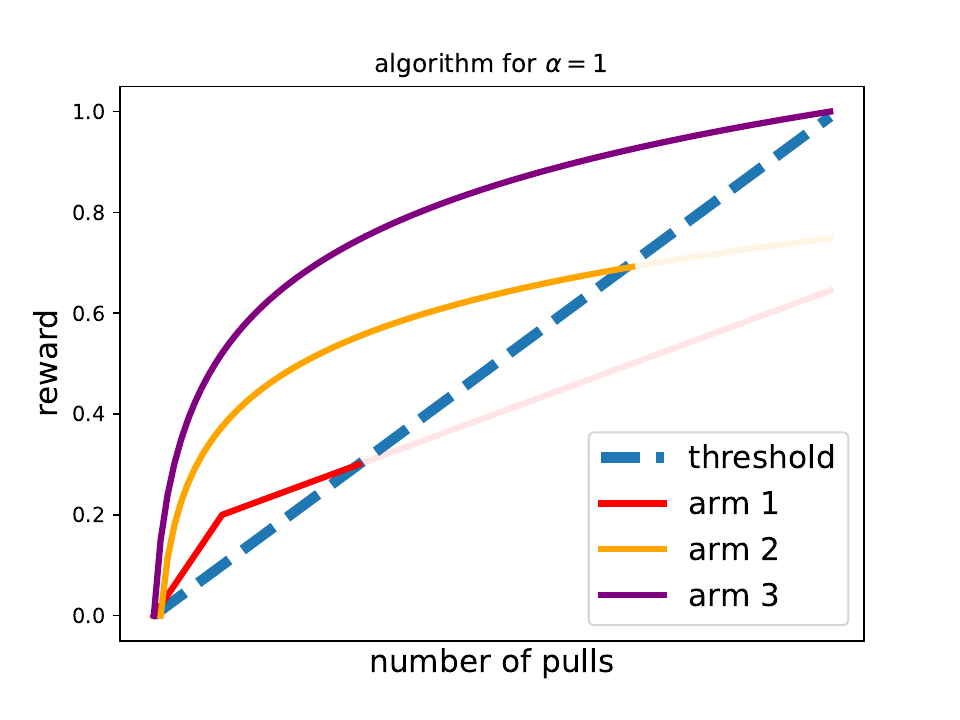} 
    \end{subfigure}
    \begin{subfigure}{0.45\textwidth}
        \includegraphics[width=0.9\textwidth]{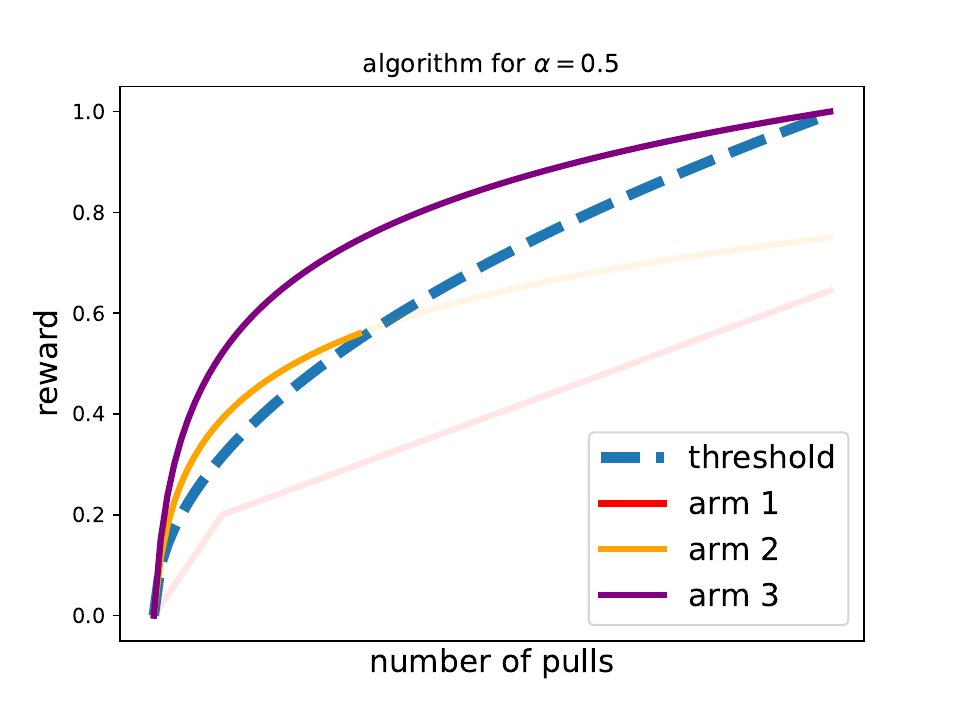} 
    \end{subfigure}
    \caption{This figure shows a snapshot of running PTRR with differing $\alpha$s on the same instance. In the case where $\alpha = 1\,,$ recovering the algorithm of \cite{pmlr-v272-blum25a}, we see that each arm is discarded once it crosses the linear lower bound. Note that the best arm (purple) is never discarded. On the right, we see the run of the algorithm when $\alpha=0.5$. Here, since the best arm satisfies the CEE with $\beta=0.5\,,$ we never discard the best arm. However, we discard worse arms faster. Thus, we can see the motivation for using $PTRR_\alpha$ for the largest possible $\alpha$ such that we still ensure we do not discard the best arm.}
    \label{fig:alg-snapshot}
\end{figure}

\begin{definition} \label{defn:alg-fam}
    Define the family of algorithms $\textit{PTRR}$ as $ \{ \textit{PTRR}_\alpha: \alpha \in (0,1] \},$ where $\textit{PTRR}_\alpha$ ($\alpha$-Power-Thresholded Round Robin) is Algorithm~\ref{alg:ptrr}.
\end{definition}

\subsection{Sharper Competitive Ratio
}\label{sec:sharper CR}

We will now argue that it is natural to consider the simple one–parameter family \textit{PTRR} from  Definition~\ref{defn:alg-fam}.
First, we observe that our family contains the algorithm from \cite{pmlr-v272-blum25a} that is known to be optimal in the unrestricted case. We then show that there are algorithms in \textit{PTRR} that improve the competitive ratio on every instance that satisfies a mild growth condition.
A matching lower bound shows that these algorithms are optimal on such instances.

For simplicity, we assume that both $T$ and $f^*(T)$ are known to the algorithm. In reality, our analysis---mirroring \cite{pmlr-v272-blum25a}---only requires that $m$ lies within a constant factor of $f^*(T-k)$, which we achieve by setting $m:=\frac{T-k}{T} \cdot f^*(T)$. We can avoid this requirement altogether by spending half our time learning $m$, thereby incurring only an extra $O(\log k)$ factor in our competitive ratio. This method is described in detail in section $5$ of \cite{pmlr-v272-blum25a}. When $T$ is also unknown, we embed the same half-explore/half-exploit method into a standard doubling schedule, preserving this $O(\log k)$ overhead. A full proof is included in Appendix \ref{appendix:unknownT}.\looseness-1

We first define a slightly stronger version of concavity.\looseness-1

\begin{definition}[Per-arm $\beta$-Lower Envelope, LE($\beta$)]\label{def:per-arm floor}
For any $\beta \in (0,1]$, we say that arm $i$ satisfies LE($\beta)$ if\looseness-1 $$f_i(t) \geq f_i(T) \left(\frac{t}{T}\right)^\beta \quad \text{for all } t \leq T.$$
\end{definition}

\begin{definition}[Concavity Envelope Exponent, $\beta_I$]\label{def:inst}
For every instance $I$,  define its Concavity Envelope Exponent $\beta_{I}$ as\looseness-1
$$
\beta_{I} \coloneqq \inf \left\{ \beta \in (0,1] : \text{ every arm in $I$ satisfies } \mathrm{LE}(\beta) \right\}.
$$
\end{definition}

\begin{remark}
    Note that smaller $\beta_I$ indicate larger early rewards, making learning easier. Moreover, note that every non-decreasing concave function with $f(0) \geq 0$ satisfies LE$(1)$, which implies that $1$ is an upper bound on the CEE. When $\beta_I$ approaches $1$ (an instance contains near‑linear arms), the problem reverts to the $\Theta(\sqrt{k})$ regime.
\end{remark}

\noindent We will now evaluate the performance of our family. 
Note that setting $\alpha=1$ recovers the Random Round Robin algorithm from \cite{pmlr-v272-blum25a}, which ensures that $\textit{PTRR}$ preserves the $O(\sqrt{k})$ guarantee in the unrestricted case.

Now suppose that an instance has $\beta_I < 1$, which occurs whenever it has no (arbitrarily close to) linear arms. Under this assumption, we will prove that choosing any $\alpha\in (\beta_I,1)$ yields the strictly smaller competitive ratio $O\!\bigl(k^{\alpha/(\alpha+1)}\bigr)=o(\sqrt{k})$. Letting $\alpha$ approach $\beta_I$ gives $O\!\bigl(k^{\beta_I/(\beta_I+1)}\bigr)$.

\begin{theorem} Given an IMAB instance $I$ with Concavity Envelope Exponent $\beta_I$. If $T \ge 2k$, then $PTRR_\alpha$ for $\alpha \in (\beta_I, 1)$  with $\tau =T-k$, $m = \frac{\tau}{T} f^*(T)$  achieves
$$
\mathbb{E}[\text{reward from PTRR}_\alpha] \ge \frac{1}{2^{\alpha+3}(\alpha + 1)} \cdot \frac{\text{OPT}_T}{(k + 1)^{\frac{\alpha}{\alpha + 1}}}, 
$$
Equivalently, the competitive ratio is $O(k^{\alpha/(\alpha+1)})$. Setting $\alpha = 1$ recovers the $O(\sqrt{k})$ bound.\label{thm:ptrr-alpha}
\end{theorem}

\begin{proof}[Proof Sketch]
Two simple facts drive the analysis of our algorithm. First, the optimal arm is never abandoned when $\alpha > \beta_I$.  Second, if we ever abandon a non‑optimal arm at time $t$, the cumulative reward collected on that arm is at least the “area” under the lower envelope up to $t$.\looseness-1

These two facts yield a recurrence that trades off (i) the chance we picked the optimal arm now, versus (ii) the value we get after spending $t$ pulls and moving on. At any state $(\tau', k')$, we either draw the best arm now (probability $1/k'$) and then keep it forever, earning $\mathrm{OPT}_{\tau'+k'}$, or we draw a bad arm. If it is bad and we abandon it at a pessimistic time $t$, we have already earned at least the threshold-area up to $t$ and we continue from the smaller state $(\tau' - t, k' - 1)$. Induction on $(\tau', k')$ yields the desired result. A full proof is included in Appendix~\ref{appendix:upperbd-proof} (see Theorem \ref{thm:upper}).
\end{proof}

Our updated keep-test affects exactly three steps of the proof in \cite{pmlr-v272-blum25a}. (i) ``Never drop \(f^\star\)'' now needs the explicit safety range \(m \le f^\star(T)(\tau/T)^\alpha\). (ii) The area-before-abandonment bound scales with \((t - 1)^{\alpha+1}\) (not quadratic). (iii) The recurrence minimization shifts from a quadratic to minimizing \(u\,y^{\alpha+1} + v(1 - y)^{\alpha+1}\), which requires a new balancing inequality that yields the factor \((k+1)^{-\gamma}\) with \(\gamma = \alpha/(\alpha + 1)\).

We next provide a matching lower bound by showing that there is a distribution on instances with the same $\beta_I$ on which no algorithm beats $\Omega(k^{\beta_I/(\beta_I+1)})$. It follows that the exponent is optimal. \textit{PTRR} attains this when $\alpha = \beta_I$.

\begin{figure}
    \centering
    \begin{subfigure}{0.45\textwidth}
    \includegraphics[width=0.9\textwidth]{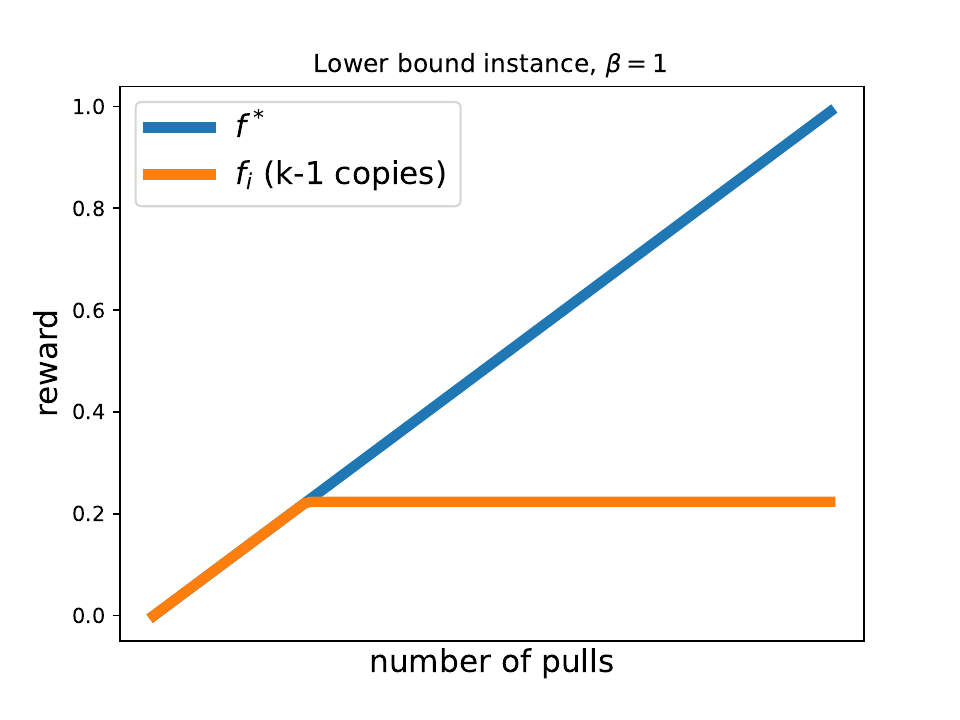} 
    \end{subfigure}
    \begin{subfigure}{0.45\textwidth}
        \includegraphics[width=0.9\textwidth]{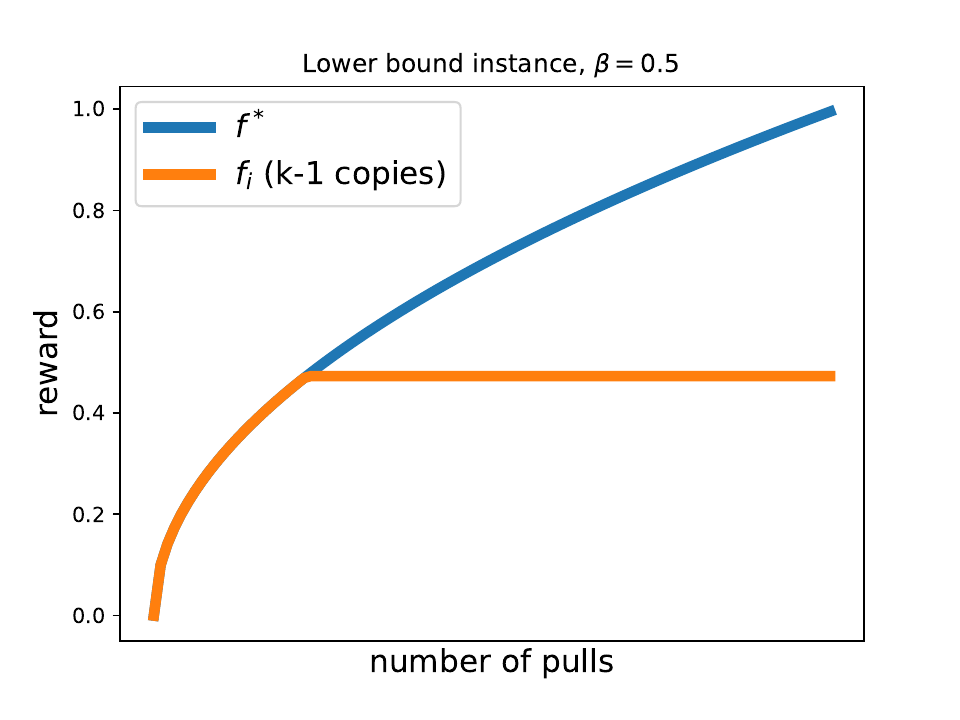} 
    \end{subfigure}
    \caption{\textbf{Lower bound instances:} The spirit of the lower bound instances is that most arms flatten after $T/\sqrt{k}$ pulls but one arm keeps increasing. However, since the algorithm needs to play any arm for a while before figuring out whether it is the good arm or one of the regular arms, the expected reward of the algorithm cannot exceed a certain amount. On the left, we reproduce the instance for $\beta = 1\,,$ recovering the lower bound of \cite{pmlr-v272-blum25a}, and on the right we show the instance for $\beta = 0.5\,.$}
    \label{fig:placeholder}
\end{figure}

\begin{theorem}[Lower Bound]
Fix $\beta \in (0, 1]$ and $k \ge 2$. 
For every (possibly randomized) algorithm, for $T$ sufficiently large, there exists an instance with Concavity Envelope Exponent $\beta_I = \beta$ such that
$
\frac{\mathbb{E}[\text{ALG}_T]}{\text{OPT}_T} \le C_\beta k^{-\beta/(\beta+1)},
$
with $C_\beta = \frac{3}{2}(\beta + 1)^2 \left[\beta(\beta + 1)\right]^{-\beta/(\beta+1)}$. Equivalently, any such algorithm has an $\Omega(k^{\beta/(\beta+1)})$ competitive ratio. 
\end{theorem}

\begin{proof}[Proof Sketch] A full proof is in Appendix \ref{appendix:lowerbd-proof}.
We construct an instance similar  to \cite{pmlr-v272-blum25a}. Define a ``good'' arm as the power curve $g(t) = m(t/T)^\beta$, choose this arm at random, and let the other $k - 1$ arms match $g$ exactly for the first $t'$ pulls and then flatten at $g(t')$. Every arm satisfies $\text{LE}(\beta)$, and the good arm violates any stricter floor, so $\beta_I = \beta$. 
At a high-level, we carefully pick $t'$ such that any deterministic algorithm must suffer bad competitive ratio on the distribution over instances. By Yao's principle, this results in a lower bound for randomized algorithms.
\end{proof}

Again, we will briefly contextualize our proof. The hard instance in \cite{pmlr-v272-blum25a} uses a linear ``good'' arm, while ours uses a power curve \(g(t) = m(t/T)^\beta\). This change is propagated in two places. (i) The generous upper bound becomes \(h(x) = 1/(kx) + (\beta + 1)x^\beta\) (instead of \(1/(kx) + 2x\)). (ii) the calculus minimizer moves to \(x^\star = [k\beta(\beta + 1)]^{-1/(\beta+1)}\), giving the exponent \(k^{-\beta/(\beta+1)}\) instead of \(k^{-1/2}\). The rounding step and use of Yao's principle are unchanged.

\subsection{Sample Complexity of Learning $\alpha$}
\label{sec:sample-complexity}

In this section, we analyze the number of samples required to learn a near-optimal algorithm from the algorithm family \textit{PTRR} (Definition~\ref{defn:alg-fam}). Previously, we considered a fixed $\alpha$ and showed its optimality for a fixed $\beta.$ However, in practice, we may not know the true $\beta$ and therefore cannot pick the optimal $\alpha.$ Now, suppose we have access to historical examples of improving multi-armed bandits problems (e.g., learning curves for the hyperparameter tuning problem on previous time periods of data). 
Then, we could hope to {\em learn} the best possible $\alpha$ for this distribution, provided we have sufficiently many samples. To analyze the sample complexity of this task, we extend techniques developed by \cite{sharma2025offline} for stochastic multi-armed bandits.
\par

\begin{remark}
    A related approach for meta-learning bandits is the {\em corralling} framework of \cite{agarwal2017corralling, arora2021corralling, luo2022corralling}. However, existing corralling techniques can only handle a finite number of hyperparameters. We further show that meta-learning even with a finite band of algorithms may not be possible without further assumptions in the improving bandits setting (see Appendix~\ref{appendix:corralling}).
\end{remark}

\paragraph{Derandomized Dual Complexity and result of \cite{sharma2025offline}.}

We start by explaining how to use results from \cite{sharma2025offline}. In their work, they defined a relevant quantity, the {\em derandomized dual complexity}, $Q_\mathcal{D}$, which accounted for randomness in data sampling and any additional randomness after fixing the instance. Since our family of algorithms includes randomized algorithms, we must also derandomize this additional source of randomness. We do this by defining $\mathcal{D'}$ as an extension of the original distribution. For each element in the support of the original distribution, we define $k!$ copies, each with a different permutation $\pi_k$ of $[k]$ associated with it. Each new augmented instance $(I, \pi_k)$ has probability $\mathcal{D'}_I \coloneqq \mathcal{D}_{I}/k!$ associated with it, where $\mathcal{D}_I$ is the probability associated with that instance originally. 
Wherever the results of \cite{sharma2025offline} refer to the distribution $\mathcal{D}$ over instances, we instead consider the distribution over instances and permutations as defined above. By derandomizing in this way, we can immediately apply their results\footnote{For our purposes, it suffices to handle randomness in the algorithm in this way. It would be interesting and valuable to extend their results to general randomized algorithms in future work.\looseness-1}.

From the results of \cite{sharma2025offline} (reproduced in Appendix~\ref{appendix:ss25results}), we know that if we have sufficiently many samples, as a function of $Q_\mathcal{D}, H$ and the accuracy and success probability parameters, then with high probability the empirical loss will be close to the population loss. 
We know that if we have enough samples to achieve uniform convergence, we will find a hyperparameter value that achieves near-optimal performance on future instances from the distribution (see Appendix~\ref{appendix:ucpoploss}). Thus, if we pick the value of $\alpha$ with the smallest empirical loss (i.e., run empirical risk minimization (ERM)), we can guarantee that the objective in Eqn.~\ref{eqn:goal} is satisfied. It remains to (1) bound $Q_\mathcal{D}$ for our problem and (2) investigate various reasonable piecewise-$H$-bounded losses.\looseness-1 

\paragraph{Bounding derandomized dual complexity in our setting.}
Now, we compute a bound on the value of $Q_\mathcal{D}\,,$ by bounding the number of possible behaviors of algorithms from $\mathcal{A}\,.$ To do so, we identify a sufficient statistic for the loss of a fixed algorithm on an instance then count the number of possible values of the statistic.

\begin{lemma} \label{lemma:bdqd}
    For the family $\mathcal{A}$ defined in Defn.~\ref{defn:alg-fam}, the improving multi-armed bandits problem, and {\em any} piecewise-constant loss function, $Q_\mathcal{D} \le kT$\,.
\end{lemma}



\newcommand{\algalpha}{\mathcal{A}_\alpha}
\begin{proof}
    We show this lemma by defining a instance\footnote{again, augmented instance}- and algorithm-specific tuple $R_\alpha\,.$ We define $R_\alpha$ such that it contains sufficient information to evaluate the loss of the algorithm on the fixed (augmented) instance. That is, since one value of the tuple gives rise at most one value of loss, we can bound the number of possible values of loss by counting the number of such tuples. \par

    First, fix an algorithm $\algalpha\,.$ This algorithm proceeds by first generating a curve to which any chosen arm is compared, namely $c_\alpha(t) \coloneq\left(\frac tT \right)^\alpha f^\star(T).$ Next, call the first arm chosen by the algorithm $f_1(t).$ The algorithm plays $f_1$ until the first time $t^{(1)}_\text{stop} \in \{ 1, 2, \hdots \, T \}$ such that $f_1(t^{(1)}_\text{stop}) < \left( \frac{t^{(1)}_\text{stop}}{T} \right)^\alpha f^*(T).$
    Similarly, we can calculate a $t^{(2)}_\text{stop}$ for the second arm played by the algorithm and so on. This provides us a tuple $R_\alpha \coloneqq (t^{(1)}_\text{stop}, t^{(2)}_\text{stop}, \hdots, t^{(k)}_\text{stop})$ that provides enough information to compute the loss of the algorithm on the instance.

    Now, we observe that as we vary $\alpha\,,$ the number of possible tuples that can be generated is upper bounded by the total number of possible tuples. Thus, we simply need to count the number of possible such tuples $R_\alpha\,.$ Since each $t^{(i)}_\text{stop}$ takes on a discrete value in $[T]\,,$ there are at most $T$ values an element in the tuple could take on, and since there are $k$ elements in the tuple, the total number of such tuples is {\em at most} $kT\,.$
\end{proof}

\paragraph{Sample complexity.}
We present sample complexity results for generic piecewise-$H$-bounded loss functions by extending Theorem~\ref{thm:ss25main} and Lemma~\ref{lemma:bdqd}. In Appendix~\ref{appendix:avgregret}, we instantiate this for a concrete loss function.

\begin{theorem}
    The sample complexity of Hyperparameter Transfer in family $\mathcal{A}$ for the improving multi-armed bandits problem with generic piecewise-$H$-bounded loss is $N = O\left( \left(\frac{H}{\epsilon}\right)^2 (\log kT + \log \frac 1 \delta \right).$
\end{theorem}

While the above result guarantees sample efficiency for learning the best $\alpha,$ it is also interesting to consider the question of computational efficiency. In Appendix~\ref{appendix:erm-alg}, we provide a framework for implementing ERM 
that is efficient if the number of arms is a small constant. In Appendix~\ref{appendix:approx-learning-alpha}, we propose and analyze a potentially more practical algorithm to select a value of $\alpha$ that provides a good approximation for most instances in the distribution.\looseness-1

\subsection{Empirical Evaluation} 
\label{sec:empirical}

\begin{figure}[h]
    \centering
    \includegraphics[width=0.5\linewidth]{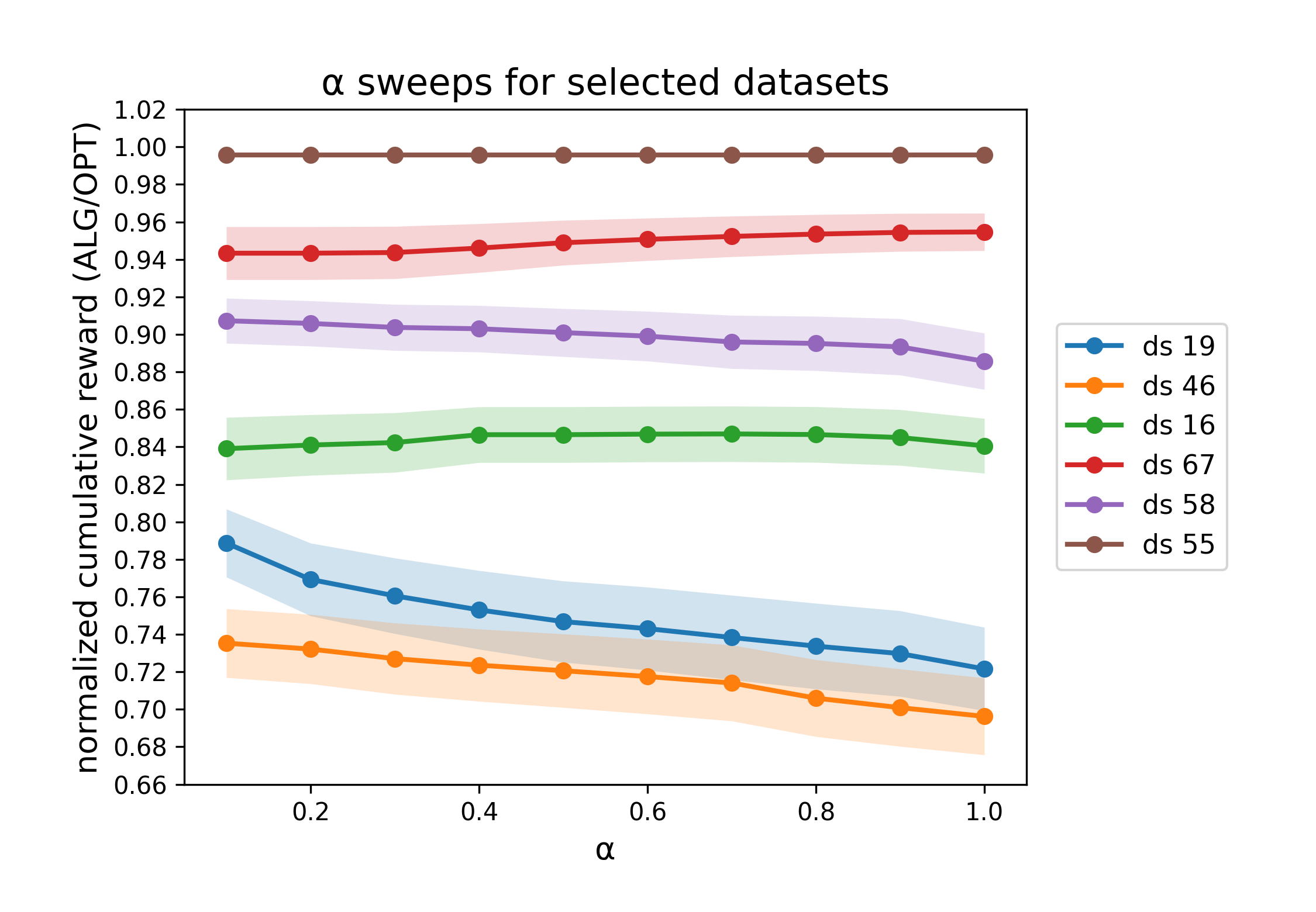}
    \caption{\textbf{Sensitivity of $\mathrm{PTRR}_\alpha$ to $\alpha$ on selected LCDB instances ($T=44$, $k=22$).}
Each curve corresponds to one CC-18 dataset $d$ from LCDB~1.1 and reports the normalized cumulative reward
$\mathbb{E}[\mathrm{PTRR}_\alpha(d)]/\mathrm{OPT}(d)$ as a function of $\alpha\in\{0.1,0.2,\dots,1.0\}$, where
$\mathrm{OPT}(d)=\max_i\sum_{t=1}^T r_{i,d}(t)$ is the cumulative reward of the best fixed-arm policy in hindsight under the same horizon (which is not necessarily the optimal policy due to absence of monotonicity, but it is a useful proxy.) 
and
$r_{i,d}(t)=1-\mathrm{err}_{i,d}(t)$ is the mean (over cross-validation) reward at anchor $t$ for arm $i$.
For each $(d,\alpha)$, $\mathbb{E}[\mathrm{PTRR}_\alpha(d)]$ is estimated by averaging over $200$ random arm orderings. The shaded regions are pointwise 95\% Student-$t$ confidence intervals across the 200 runs (mean $\pm\, t_{0.975,199}\cdot \mathrm{sd}/\sqrt{200}$).
The displayed datasets are selected to illustrate the range of sweep shapes observed across the full benchmark (near-flat curves, monotone trends, and interior maxima). For most datasets, performance differences across $\alpha$ are small relative to the confidence intervals, while a minority show a significant trend across $\alpha$ on this grid. Complete sweeps over all $27$ usable datasets are included in Appendix~\ref{appendix:empirical}.}
    \label{fig:6 sweeps}
\end{figure}

\begin{figure*}[h]
    \centering
    \includegraphics[width=1\linewidth]{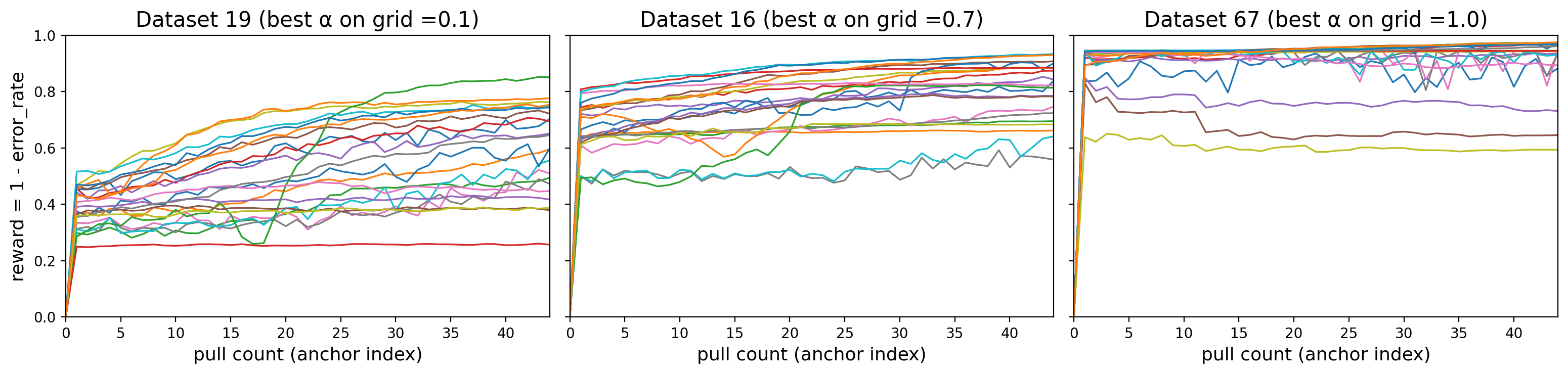}
    \caption{\textbf{Mean LCDB reward curves for three datasets with distinct best $\alpha$ values on the grid.} Each panel overlays the mean reward curves $r_{i,d}(t)=1-\mathrm{err}_{i,d}(t)$ across anchors $t$ for all $k=22$ arms on a single CC-18 dataset $d$. The title of each panel reports the value of $\alpha\in\{0.1,0.2,\dots,1.0\}$ that maximizes the estimated normalized cumulative reward $\mathbb{E}[\mathrm{PTRR}_\alpha(d)]/\mathrm{OPT}(d)$ at horizon $T=44$ on that dataset. 
    We selected these datasets to illustrate diverse best $\alpha$ values on the grid. Qualitatively, these plots suggest a mechanism consistent with the influence of $\alpha$ on $\mathrm{PTRR}_\alpha$, where datasets preferring smaller $\alpha$ tend to exhibit early separation between `good' and `bad' arms (making aggressive abandonment beneficial).\looseness-1
    \label{fig:learning curves}}
\end{figure*}

We provide empirical evidence that the value of $\alpha$ used in running $PTRR$ can affect the cumulative reward using data from a learning curve dataset, LCDB~1.1 (CC-18 benchmarks,~\cite{yanlcdb}). We map each dataset in the LCDB dataset to an IMAB instance as follows: the $k$ arms are the LCDB learners, the time index
corresponds to the number of samples seen, and the rewards arise from 
inverting the error rates through $r_{i,d}(t)=1-\mathrm{err}_{i,d}(t)$.
We average over cross-validation and retain only datasets for which the curve is defined for all $t \in \{1, \dots, T\}\,.$ This yields 27 datasets with $k = 22$ arms each and a time horizon of $T=44.$ We run $PTRR_\alpha$ for each $\alpha \in \{0.1, 0.2, \ldots, 1.0\}$ and report normalized cumulative reward $\mathbb{E}[\mathrm{PTRR}_{\alpha}(d)] / \mathrm{OPT}(d)$ 
(the reciprocal of the competitive ratio in Definition~2.2) averaged over 200 seeds. Figure \ref{fig:6 sweeps} shows that the $\alpha$ value maximizing this quantity varies across datasets.
Figure \ref{fig:learning curves} plots mean reward curves for three datasets to help visualize the underlying reward dynamics. Full details of the  setup and results are provided in Appendix~\ref{appendix:empirical}.\looseness-1

$
$

\section{Achieving Best-of-both-worlds Best Arm Identification}\label{sec:bothworlds}

\begin{figure}[t]
    \centering

\tikzset{every picture/.style={line width=0.75pt}} 

\begin{tikzpicture}[x=0.75pt,y=0.75pt,yscale=-0.9,xscale=0.9]

\draw    (71,600.46) -- (312,600.46) ;
\draw [shift={(314,600.46)}, rotate = 180] [color={rgb, 255:red, 0; green, 0; blue, 0 }  ][line width=0.75]    (10.93,-3.29) .. controls (6.95,-1.4) and (3.31,-0.3) .. (0,0) .. controls (3.31,0.3) and (6.95,1.4) .. (10.93,3.29)   ;
\draw    (71,600.46) -- (69.02,382.46) ;
\draw [shift={(69,380.46)}, rotate = 89.48] [color={rgb, 255:red, 0; green, 0; blue, 0 }  ][line width=0.75]    (10.93,-3.29) .. controls (6.95,-1.4) and (3.31,-0.3) .. (0,0) .. controls (3.31,0.3) and (6.95,1.4) .. (10.93,3.29)   ;
\draw    (70,419.46) -- (268,419.46) ;
\draw  [dash pattern={on 4.5pt off 4.5pt}]  (71,600.46) -- (160,509.46) ;
\draw  [dash pattern={on 4.5pt off 4.5pt}]  (160,509.46) -- (269,509.46) ;
\draw    (391,600.46) -- (632,600.46) ;
\draw [shift={(634,600.46)}, rotate = 180] [color={rgb, 255:red, 0; green, 0; blue, 0 }  ][line width=0.75]    (10.93,-3.29) .. controls (6.95,-1.4) and (3.31,-0.3) .. (0,0) .. controls (3.31,0.3) and (6.95,1.4) .. (10.93,3.29)   ;
\draw    (391,600.46) -- (389.02,382.46) ;
\draw [shift={(389,380.46)}, rotate = 89.48] [color={rgb, 255:red, 0; green, 0; blue, 0 }  ][line width=0.75]    (10.93,-3.29) .. controls (6.95,-1.4) and (3.31,-0.3) .. (0,0) .. controls (3.31,0.3) and (6.95,1.4) .. (10.93,3.29)   ;
\draw    (391,599.46) -- (577,417.46) ;
\draw  [dash pattern={on 4.5pt off 4.5pt}]  (391,600.46) -- (441,551.46) ;
\draw  [dash pattern={on 4.5pt off 4.5pt}]  (441,551.46) -- (581,551.46) ;

\draw (122,617.46) node [anchor=north west][inner sep=0.75pt]   [align=left] {number of arm pulls};
\draw (31.88,522.09) node [anchor=north west][inner sep=0.75pt]  [rotate=-269.09] [align=left] {reward};
\draw (114,398.46) node [anchor=north west][inner sep=0.75pt]   [align=left] {optimal arm $\displaystyle i^{*}$};
\draw (55,418.46) node [anchor=north west][inner sep=0.75pt]  [font=\small] [align=left] {$\displaystyle 1$};
\draw (253,602.86) node [anchor=north west][inner sep=0.75pt]  [font=\small]  {$T$};
\draw (143,603.86) node [anchor=north west][inner sep=0.75pt]  [font=\small]  {$T/2$};
\draw (57,484.86) node [anchor=north west][inner sep=0.75pt]  [font=\footnotesize]  {$\frac{1}{2}$};
\draw (168,484.46) node [anchor=north west][inner sep=0.75pt]   [align=left] {all other arms};
\draw (31,642.46) node [anchor=north west][inner sep=0.75pt]   [align=left] {(a) \cite{pmlr-v272-blum25a} may be suboptimal by a factor of\\two on some ``nice" instances};
\draw (442,617.46) node [anchor=north west][inner sep=0.75pt]   [align=left] {number of arm pulls};
\draw (351.88,522.09) node [anchor=north west][inner sep=0.75pt]  [rotate=-269.09] [align=left] {reward};
\draw (444,438.46) node [anchor=north west][inner sep=0.75pt]   [align=left] {optimal arm $\displaystyle i^{*}$};
\draw (375,418.46) node [anchor=north west][inner sep=0.75pt]  [font=\small] [align=left] {$\displaystyle 1$};
\draw (573,602.86) node [anchor=north west][inner sep=0.75pt]  [font=\small]  {$T$};
\draw (428,602.86) node [anchor=north west][inner sep=0.75pt]  [font=\small]  {$T/k$};
\draw (376,532.86) node [anchor=north west][inner sep=0.75pt]  [font=\footnotesize]  {$\frac{1}{k}$};
\draw (488,529.46) node [anchor=north west][inner sep=0.75pt]   [align=left] {all other arms};
\draw (351,642.46) node [anchor=north west][inner sep=0.75pt]   [align=left] {(b) \cite{mussi2024best} may be suboptimal by a factor of\\$\displaystyle \sqrt{k}$ on worst-case instances};

\end{tikzpicture}
    
    \caption{Examples demonstrating the need for a best-of-both-worlds approach.}
    \label{fig:motivating-examples}
\end{figure}

In this section, we bridge a gap in the literature: so far, algorithms can either identify the best arm in an instance that is sufficiently benign or identify an approximate best arm but not {\em both}, i.e., identify the best arm if possible and achieve the approximate goal if the exact goal cannot be achieved. To this end,
we introduce a family of hybrid algorithms that address the task of identifying the arm with the highest final reward, the best arm identification (BAI) task\footnote{While the reader may be more familiar with BAI in the sense of maximum mean in stochastic bandits, we note that in that case, we could equivalently study cumulative reward from a single arm.}. Each algorithm in this family
guarantees best arm identification whenever an instance is sufficiently benign, while simultaneously achieving optimal (up to constants) multiplicative guarantees on worst-case instances. 
We first provide examples that show that no existing algorithm achieves near-optimal results on both nice and worst-case instances. We define our family in Section~\ref{sec:hybriddef}. In Section~\ref{sec:hybridguarantees}, we formalize a `niceness' condition under which it is possible to sample every suboptimal arm enough to safely discard it by the switch time. We then prove that $\text{Hybrid}_{1,T/2}$ is guaranteed to identify the best arm on instances that satisfy this condition and, on all other instances, returns an arm whose expected maximum pull at time $T$ competes optimally with the highest single pull in the instance. We also develop results for choosing $\alpha, B$ in a data-driven way.\looseness-1 


\textbf{Motivating Examples.} 
Prior work gives two distinct types of guarantees for improving bandits---optimal approximation factors for worst-case instances \cite{patil_mitigating_2023,pmlr-v272-blum25a}, and better guarantees (sublinear policy regret, or small sample complexity of best arm identification) for nicer instances \cite{heidari_tight_nodate,metelli_stochastic_2022,mussi2024best}.
It is natural to ask whether we can achieve the best-of-both-worlds, that is, get almost optimal results on nice instances, while being within the optimal approximation factors (up to constants) in the worst-case. It turns out that the algorithms proposed in prior literature fail to achieve such a guarantee. 

The following example shows that the algorithm of \cite{pmlr-v272-blum25a} has sub-optimal guarantees on nice instances.\looseness-1

\begin{example}
{\it The randomized algorithm of \cite{pmlr-v272-blum25a} may fail to identify the best-arm on instances where the UCB-style algorithms of \cite{metelli_stochastic_2022,mussi2024best} find the best-arm.} We set the reward function for the best arm as $f_{i^*}(t)=1$ for all $t$, and for any other arm $i\ne i^*$ as $f_{i}(t)=\min\{\frac{t}{T},\frac{1}{2}\}$, where $T$ is the time horizon. We assume that $k$ is large (in particular, $k\ge 4$). Now, the randomized round-robin algorithm selects the optimal arm as its first or second arm with probability at most $2/k$. Otherwise, it keeps playing a sub-optimal arm till time $T/2$, and a different sub-optimal arm for the rest of the time horizon. Thus, with probability at least $1-\frac{2}{k}$, the best arm identified by the algorithm is sub-optimal by at least a factor of 2 (with respect to both its cumulative reward and its final pull). On the other hand, for the UCB-based algorithms, we can upper bound the number of exploratory pulls of the sub-optimal arms, and the algorithm succeeds in identifying the best arm $i^*$.
\end{example}

\noindent The following example shows that the UCB-variant (R-UCBE) developed for the improving bandits BAI problem by~\cite{mussi2024best}  has sub-optimal worst-case performance.

\begin{example}
{\it The UCB-based algorithm of ~\cite{mussi2024best} for best-arm identification may output an arm with $\Omega(k)$ sub-optimal reward compared to the actual best arm on some instances.} We adapt the example used by \cite{pmlr-v272-blum25a} in their lower bound construction. Set $f_{i^*}(t)=t/T$ for all $t$ for the optimal arm $i^*$, and for any other arm $i\ne i^*$ as $f_{i}(t)=\min\left\{\frac{t}{T},\frac{1}{k}\right\}$. Due to the exploration term in UCB, while the arm rewards are identical, each arm gets pulled an equal number of times. By time $T$, each arm gets pulled $T/k$ times, and all the arms appear identical to the algorithm. Thus, the best arm learned by the algorithm may be sub-optimal (with respect to both its cumulative reward and its final pull) by a factor of $\Omega(k)$. In contrast, the PTRR algorithm family guarantees a worst-case competitive ratio of $O(\sqrt{k})$ or smaller.\looseness-1
\end{example}




\subsection{Hybrid Algorithm Family}\label{sec:hybriddef}
We now define the hybrid algorithm family. As before, we work in the standard improving‑bandits setting with $k$ arms, a known horizon $T$, and non‑decreasing concave reward functions $f_i$. Each algorithm Hybrid$_{\alpha,B}$ has two stages. Stage 1 uses a UCB-style envelope: At each step, the algorithm computes a lower bound $L_i(n)$ and terminal upper bound $U_i(n)$ on the final reward $f_i(T)$ of every arm and pulls the arm with the largest optimistic estimate $U_i$. If the lower bound of one arm dominates the terminal upper bound of every other arm, Hybrid$_{\alpha,B}$ commits to this arm. 
If no commit occurs by time $B$, Stage 2 runs PTRR$_\alpha$ and finds an arm whose expected terminal reward is at least a substantial fraction of the best arm’s. 

For the terminal envelope, we define $L_i(t) := f_i(t),
\triangle_i(t) :=  (T-t) \gamma_i(t - 1),$ and $
U_i(t) := L_i(t) + \triangle_i(t), 
$
where $\gamma_i (t-1) := f_i(t) - f_i (t-1)$. We set $U_i(0) : = \infty$ to ensure first pulls. Using concavity and monotonicity, it is straightforward to prove that $L_i(t) \leq f_i(T) \leq U_i(t)$ for all $i,t$ (see Lemma \ref{lem:envelope} in Appendix~\ref{appendix:hybridproofs}).

\begin{definition} \label{def:fullyhybridfam}
Define the family of algorithms $\textit{Hybrid} := \{\textit{Hybrid}_{\alpha, B} (Algorithm ~\ref{alg:hybrid}) : \alpha \in (0,1]$ and $B \in [T]$\}.
\end{definition}

\begin{algorithm}[t]
\caption{$\textit{Hybrid}_{\alpha, B}$ BAI}\label{alg:hybrid}
\begin{algorithmic}[1]
\REQUIRE maximum final pull $m$, horizon $\tau$
\STATE \textbf{Stage 1}: $t\gets0$ 
\FOR{each arm $i$}
\STATE $t_i\gets0$, $L_i\gets0$, $U_i\gets+\infty$
\ENDFOR
\WHILE{$t<B$}
  \FOR{each $i$ with $t_i\ge1$}
     \STATE $L_i \gets f_i(t_i)$ \hfill 
     \STATE $\gamma_i \gets f_i(t_i)-f_i(t_i-1)$ \hfill 
     \STATE $U_i \gets L_i + (\tau-t)\,\cdot \gamma_i$ 
  \ENDFOR
  \STATE $\hat i \gets \arg\max_i L_i$,\quad $U_{\mathrm{next}}\gets \max_{j\ne \hat i} U_j$
  \IF{$L_{\hat i} > U_{\mathrm{next}}$} 
     \RETURN $\hat i$.
  \ENDIF
  \STATE $i' \gets \arg\max_i (U_i - L_i)$
  \STATE \textbf{pull} $i'$;\; $t_{i'}\gets t_{i'}+1$, $t\gets t+1$
\ENDWHILE
\STATE \textbf{Stage 2}:
 $\tau' \gets (\tau - B) - k$, $m' \gets \left(\frac{\tau'}{T}\right) \cdot m$
\FOR{each $i$}
\STATE $g_i(s)\gets f_i(t_i + s)$
\ENDFOR
\RETURN $\hat i \gets \text{arm returned by } \textit{PTRR}_\alpha$ with parameters $(m', \tau')$ on $\{g_i\}$ for $T-B$ steps.

\end{algorithmic}
\end{algorithm}

\subsection{Best-of-Both-Worlds BAI guarantees}
\label{sec:hybridguarantees}

In this section, we argue that it is meaningful to study the two-parameter family of \textit{Hybrid} algorithms from Definition \ref{def:fullyhybridfam}. We will do this by showing that \textit{Hybrid} contains algorithms that simultaneously (i) guarantee best arm identification on sufficiently benign instances, and (ii) preserve tight (up to constants) multiplicative bounds for approximating the best arm
on adversarial instances.\footnote{With additional information about stronger concavity, we achieve sharper bounds by using PTRR$_\alpha$ with appropriate $\alpha$.}

We start by defining a class of `sufficiently benign' instances and prove that our algorithm is guaranteed to return the best arm on all members of this class. If an instance is not in this class, Stage~2 pursues {\em approximate} BAI and runs $PTRR_\alpha$ with $\alpha$ dependent on the strength of concavity ($\alpha = 1$ works for {\em all} instances). Having reverted to an approximate goal, we identify an arm $\hat{i}$ whose final reward satisfies
$\mathbb{E}[f_{\hat{i}}(T)] \ge \Omega\left(k^{\frac{-\alpha}{1+\alpha}}\right) f^{\star}(T)$. We assume that both $T$ and $f^*(T)$ are known to the algorithm, which we input as $\tau$ and $m$.\looseness-1 

We will now define a condition (Definition \ref{def:GCC1}) under which we can guarantee best arm identification. For any arm $i$ and $\epsilon>0$, the {\it terminal budget} $h_i$ of arm $i$ is defined as
$
h_i(\epsilon) := \min\{\,n \in \{2,\ldots,T\} : \triangle_i(n) \le \epsilon\,\},
$
where $\triangle_i(t) := (T-t) \, \gamma_i(t - 1)$.  For any instance $I$, its {\it Best Arm Gap} $\Delta_I$ is defined as
 $\Delta_I := f^\star (T) - \max_{j \ne i^*} f_j(T).$



\begin{definition}[Gap Clearance Condition, $GCC(B)$]\label{def:GCC1}
For any $B\le T$, we say that an instance satisfies the Gap Clearance Condition $GCC(B)$ if $\Delta_I > 0$ and
$
\sum_{i=1}^K h_i(\Delta_I / 3) \;\le\; B.
$
\end{definition}

Under concavity, $\triangle_i(n)$ is an upper bound on the remaining possible increase in terminal value of arm $i$ after $n$ pulls: it denotes the most additional reward an adversary can ``hide in the tail''  by continuing with the last observed slope. The per-arm budget $h_i(\epsilon)$ is therefore the minimal number of pulls needed to ensure its optimistic terminal value is close to current lower envelope. 
Taking $\epsilon = \Delta_I/3$, once a suboptimal arm has been pulled this many times, its best possible continuation still loses to the best arm’s lower envelope. 
Thus $\mathrm{GCC}(B)$ ensures that the total work needed to certify the best arm fits within budget $B$. 
If the sum exceeds $B$, concavity allows at least one suboptimal arm to remain plausibly optimal by time $B$, so no sound mid-horizon certificate can be guaranteed.
For a comparison of this requirement to other `niceness' conditions from the literature and a proof of Theorem \ref{thm:hybrid-single2}, see Appendix~\ref{appendix:PR}.

\begin{theorem}[best-of-both-worlds guarantees] \label{thm:hybrid-single2}
Suppose an instance $I\in\mathcal{I}$ has Concavity Envelope Exponent $\beta_I\in(0,1]$. Algorithm \ref{alg:hybrid} with $\alpha \in(\beta_I, 1]$  satisfies the following:
\begin{enumerate}
[leftmargin=*,topsep=0pt,partopsep=1ex,parsep=1ex]\itemsep=-4pt
\item[(1)] If the instance further satisfies $\Delta_I>0$, and $GCC(\theta)$ holds for $\theta\le T/2$, then the algorithm with $B\in (\theta,T/2]$  identifies and commits to the best arm $i^\star$ in Stage~1.
    \item[(2)] If Stage~1 does not certify a best arm by time $B$, then Stage~2 finds an approximate best arm $\hat i$ such that 
    $\mathbb{E}\big[f_{\hat i}(T)\big]\ \ge\ \Omega{\left(k^{\frac{-\alpha}{1+\alpha}}\right)}\; f^\star(T),$ where the expectation is  over the randomness of the algorithm. 
\end{enumerate}
\end{theorem}

\begin{proof}[Proof Overview]
(1) $\mathrm{GCC}(\theta)$ implies that the per–arm budgets $h_i(\Delta_I/3)$ sum to at most $\theta$. Since the algorithm pulls the arm with the largest slack, we know that these budgets are met within $B\ge\theta$ pulls. Then every suboptimal arm has $U_j\le f^*(T)-2\Delta_I/3$ and the best arm has $L_{i^\star}\ge f^*(T)-\Delta_I/3$, which implies that $L_{i^\star}\ge\max_{j\ne i^\star}U_j$. We commit to $i^\star$.
(2) If no certificate fires by $B\le T/2$, running $\textit{PTRR}_\alpha$ for $T_{\mathrm{rem}}$ steps yields an average reward within a $k^{\alpha/(1+\alpha)}$ factor of the residual optimum (up to constants). Using the fact that the best single pull is at least the average, and $\mathrm{OPT}^{\mathrm{res}}$ is at least a constant times $g^\star(T_{\mathrm{rem}})\,T_{\mathrm{rem}}$ by concavity, we get $\mathbb{E}[f_{\hat i}(T)]\ge \Omega\!\big(k^{-\alpha/(1+\alpha)}\big)\,f^\star(T)$. 
\end{proof}

Finally, in Appendix \ref{sec:hybridcomplexity}, we establish bounds on the sample complexity of tuning $\alpha,B$ in the {\it Hybrid} family. 
\section{Conclusion}
We provide beyond worst-case guarantees for the improving multi-armed bandits problem. 
Employing the lens of data-driven algorithm design, we show that using sampled instances from a distribution, we can find algorithms that are near-optimal for instances arising from {\em that} distribution. We introduce two novel and interesting families of algorithms. Our first family achieves stronger competitive ratios that depend on the strength of concavity of the reward curves, and the second achieves a best-of-both-worlds guarantee on benign as well as worst-case instances which prior techniques in the literature fail to achieve. \looseness-1

Our work opens up several directions for future research. It would be interesting to extend our notion of strength of concavity to more refined arm-dependent and piecewise structures with unknown switching times, where techniques for ``tracking regret'' \cite{herbster2001tracking,sharma2020learning} might be helpful. Our best-of-both-worlds results mirror known ``benign'' regimes where sublinear regret is possible, but it would be interesting to determine the precise necessary and sufficient conditions for achieving sub-linear regret in improving bandits. Computationally efficient implementation and sample complexity lower bounds are  concrete questions raised by our offline-to-online hyperparameter transfer results. Finally, it would be interesting to develop similar guarantees for variants of the improving bandits model studied in the literature, including stochastic (arm reward means are concave non-decreasing, with sub-Gaussian variance) and restless (arms evolve with time, not pulls) versions.

\printbibliography

\clearpage
\onecolumn
\appendix
\section{Additional Related Work}\label{appendix:additional-related}
\textit{Improving (Rising) Bandits.} Initially, \cite{heidari_tight_nodate} formulated the improving (and decaying) bandits problem where payoffs increase the more the arm is played, and obtain sub-linear policy regret. \cite{patil_mitigating_2023} show that a linear regret is unavoidable in the worst-case and obtain a tight $\Theta(k)$ on the competitive ratio of deterministic algorithms. 
Then, \cite{pmlr-v272-blum25a} extend the development on this problem to randomized algorithms, 
achieving an $O(\sqrt{k} \log k)$ upper bound, nearly matching their $\Omega(\sqrt{k})$ lower bound. A related line of work is that on (deterministic, rested) {\em rising} bandits \cite{metelli_stochastic_2022,mussi2024best}. This line of work bounds the instance-dependent policy regret on benign instances but otherwise considers similar increasing reward functions. We note that policy regret is a stronger notion of regret than {\em external regret}, where an algorithm is compared to the best action on {\em the same} set of rewards. In policy regret (and in the competitive ratios we consider), an algorithm is compared to the best policy for the sequence of rewards generated by {\em playing that policy}.\looseness-1

\textit{Data-Driven Algorithm Design.}  Data-driven algorithm design is emerging as a powerful approach for using machine learning to design algorithms that have strong performance on “typical” input instances, as opposed to worst-case or average-case analysis~\cite{gupta2016pac,balcan2017learning,balcan2018dispersion,balcan2024much}. The approach is known to be effective in both statistical and online learning settings \cite{balcan2020data}, and, under this paradigm, a growing line of research has successfully developed techniques for designing several fundamental  algorithms in machine learning and beyond (e.g.~\cite{blum2021learning,bartlett2022generalization,jinsample,khodak2024learning,balcan2024trees,cheng2024learning,sakaue2024generalization,sharma2024no,balcanalgorithm}). Recent work \cite{sharma2025offline} has shown how to tune hyperparameters in standard multi-armed bandit algorithms like UCB, LinUCB and GP-UCB under this paradigm. While they focus on stochastic bandits, we extend their techniques to tune hyperparameters in the non-stochastic improving bandits setting. Another related work \cite{khodak2023meta} shows how to tune hyperparameters for adversarial bandit algorithms in an online-with-online meta-learning setting. While we focus on statistical complexity of tuning our bandit algorithms, an interesting future direction is to give computationally efficient algorithms~\cite{balcan2024accelerating,chatziafratisaccelerating}.\looseness-1

\section{Full Proofs for Sharper Competitive Ratio}
\label{appendix:competitive-ratio}
In this Appendix, we give full proofs of our optimal sharper competitive ratio for each value of the Concavity Envelope Exponent $\beta_I$ satisfied by the instance $I$, along with a formal proof of the ``doubling trick'' applied in our context for completeness. Our analysis extends the techniques due to \cite{pmlr-v272-blum25a}.
\subsection{Upper Bound} \label{appendix:upperbd-proof}
Suppose an instance has Concavity Envelope Exponent $\beta_I \in (0,1)$,  and fix $\alpha \in (\beta_I, 1)$. Let $\gamma := \frac{\alpha}{\alpha+1}$.
For $\tau'\ge0$ and $k'\ge1$, let $V(\tau',k')$ denote the expected reward earned in the next $\tau'+k'$ pulls when $k'$ arms are still untouched. Let $V(\tau'',\cdots)=0$ for all $\tau''\le0$.

We analyze $\textit{PTRR}_\alpha$ through a recurrence that trades off two properties. First, the optimal arm is never abandoned when $\alpha > \beta_I$ (Lemma~\ref{lem:never-drop}). Second, abandoning any non‑optimal arm after $t$ pulls yields at least the “area under the envelope” up to $t$ (Lemma~\ref{lem:area}). These two facts give the value recurrence in Lemma~\ref{lem:rec}: with probability $1/k'$ we hit $f^*$ and earn $\mathrm{OPT}_{\tau'+k'}$; otherwise we bank the area from time $t$ and recurse from $(\tau'-t,k'-1)$ with the minimizer attained on $[0,\tau']$. We then solve the recurrence by normalizing $t$ to $y=t/\tau'$ and minimizing a one‑variable convex form $u\,y^{\alpha+1}+v(1-y)^{\alpha+1}$ exactly (Lemma~\ref{lem:exact-min}). A clean lower bound on this minimum (Lemma~\ref{lem:balance}) closes the induction and yields
$$
V(\tau',k') \ge \frac{m}{(\alpha+1)\tau^\alpha}\cdot \frac{(\tau')^{\alpha+1}}{2\,(k'+1)^\gamma}
$$
(Lemma~\ref{lem:solve}). Finally, substituting $(\tau',k')=(T-k,k)$ and $\mathrm{OPT}_T$’s comparison to $m$ gives Theorem~\ref{thm:upper}. A more detailed overview of the proof is included after the theorem.

Our analysis of $\mathit{PTRR}_\alpha$ holds for any parameter $m$ that satisfies
$$\frac{1}{c_2} f^*(T-k) \leq m \leq f^*(T) \left( \frac{T-k}{T} \right)^\alpha$$ for some constant $c_2 \geq 1$. In particular, setting $m := \frac{T - k}{T} f^{\star}(T)$ gives $c_2 = \frac{T}{T - k} \in [1,2]$ when $T \ge 2k$. Since $\frac{T-k}{T} \in (0,1]$ and $\alpha \in (0,1)$, we also know that $f^{\star}(T) \frac{T - k}{T} \leq f^*(T) \left(\frac{T-k}{T}\right)^\alpha$, and therefore that $\frac{1}{2} f^{\star}(T - k) \le m \le f^*(T) \left( \frac{T-k}{T} \right)^\alpha$ in this case. 

\begin{lemma}[Optimal arm is never dropped]\label{lem:never-drop}
$\textit{PTRR}_\alpha$ will never switch away from the optimal arm $f^*$. 
\end{lemma}
\begin{proof}
By definition of the CEE, we know that
$
f_{\star}(t) \ge\ f_{\star}(T)\left( \frac{t}{T} \right)^{\beta_I}
$
for all $0\le t\le T$. Moreover, recall that $m \leq f^*(T) \left( \frac{\tau}{T}\right)^\alpha$. Since $\frac{\tau}{T} \in (0,1]$ and $\alpha > \beta_I$, it follows that
$$
f^{\star}(t) \ge f^{\star}(T) \left( \frac{t}{T} \right)^{\beta_I}
\ge f^{\star}(T) \left( \frac{t}{T} \right)^\alpha = f^{\star}(T) \left( \frac{\tau}{T} \right)^{\alpha} \left( \frac{t}{\tau} \right)^{\alpha}
\ge  m \left( \frac{t}{\tau} \right)^{\alpha}.
$$
for all $t \le T$, and therefore that the test never fails once the optimal arm is reached.
\end{proof}

\begin{lemma}[Area before abandonment]\label{lem:area}
If the keep-test holds for $t_i-1$ pulls on arm $i$ and fails at $t_i$, we obtain a reward of at least
$$
 \frac{m}{(\alpha+1)\tau^\alpha}\,(t_i-1)^{\alpha+1}.
$$
\end{lemma}
\begin{proof}
For $1\le s<t_i$, we know that $f_i(s)\geq m(s/\tau)^\alpha$. Summing and using the monotonicity of $x^\alpha$ on $[0,\infty)$ gives
$$
 \sum_{s=1}^{t_i-1} f_i(s)\ \ge\ m\tau^{-\alpha}\sum_{s=1}^{t_i-1}s^\alpha \ \ge\ m\tau^{-\alpha}\int_{0}^{t_i-1} x^\alpha\,dx \ \geq \ \frac{m}{(\alpha+1)\tau^\alpha}\,(t_i-1)^{\alpha+1}.
$$
\end{proof}

\noindent With $k'$ untouched arms, we either hit the best arm now (probability $1/k'$) and then take the optimum path, or we spend $t$ pulls on a non-best arm, bank the area from Lemma \ref{lem:area}, and recurse on $(\tau' - t, k' - 1)$. To ensure a worst-case guarantee, we take the minimum over all feasible abandonment times $t$.

\begin{lemma}[Value recurrence]\label{lem:rec}
For all $\tau'\ge0$, $k'\ge1$,
\begin{equation}
V(\tau',k')\ \ge\ \frac{1}{k'}\,\mathrm{OPT}_{\tau'+k'} +\Bigl(1-\frac{1}{k'}\Bigr)\min_{0\leq t\leq \tau'+k'-1}\!\left[\frac{m}{(\alpha+1)\tau^\alpha}\,t^{\alpha+1}+V(\tau'-t,k'-1)\right]\!, \label{eq:rec}
\end{equation}
and the minimum is attained on $[0,\tau']$.
\end{lemma}
\begin{proof}
With probability $1/k'$ the next arm is $f^*$ and, by Lemma \ref{lem:never-drop}, the algorithm stays on it and earns $\mathrm{OPT}_{\tau'+k'}$. Else it plays a non-optimal arm. If it abandons this arm at time $t$, Lemma~\ref{lem:area} gives the earned area and the process recurses on $(\tau'-t,k'-1)$. If $t\geq \tau'$, then $V(\tau'-t,k'-1)=0$ while the area term increases in $t$, which implies that the minimizer lies in $[0,\tau']$.
\end{proof}

\noindent Below, we prove two technical results that aid our analysis of the recurrence. First, the recurrence reduces to minimizing $u\, y^p + v(1 - y)^p$ over $y \in [0,1]$. This convex function has a unique minimizer balancing the two terms. We write that minimum in closed form to keep constants explicit.

\begin{lemma}[One-variable minimum]\label{lem:exact-min}
For $p>1$ and $u,v>0$,
$$
 \min_{y\in[0,1]}\ \{\,u\,y^p+v(1-y)^p\,\} \ =\ \bigl(u^{-1/(p-1)}+v^{-1/(p-1)}\bigr)^{-(p-1)}.
$$
\end{lemma}
\begin{proof}
Let $f(y):=u\,y^p+v(1-y)^p$ on $[0,1]$. For $p>1$, we have
$$
 f''(y)=u\,p(p\!-\!1)y^{p-2}+v\,p(p\!-\!1)(1-y)^{p-2}>0\quad(y\in(0,1)),
$$
so $f$ is strictly convex. Moreover, note that
$$
 f'(y)=u p\,y^{p-1}-v p\,(1-y)^{p-1},\quad f'(0^+)=-vp<0,\ \ f'(1^-)=up>0,
$$
which implies that $f'$ has a unique global minimizer $y^\star\in(0,1)$. Solving $u\,{y^\star}^{p-1}=v\,(1-y^\star)^{p-1}$ gives
$$
 y^\star=\frac{v^{1/(p-1)}}{u^{1/(p-1)}+v^{1/(p-1)}}.
$$
Letting $a:=u^{1/(p-1)}$ and $b:=v^{1/(p-1)}$, it follows that
$$
 \min_{y\in[0,1]} f(y)=f(y^\star)=\frac{a^{p-1}b^p+b^{p-1}a^p}{(a+b)^p} =\frac{(ab)^{p-1}}{(a+b)^{p-1}} =\bigl(u^{-1/(p-1)}+v^{-1/(p-1)}\bigr)^{-(p-1)}.
$$
\end{proof}

\noindent We now lower-bound the exact minimum with a simple balancing inequality that cleanly shows the dependence on $k'$.

\begin{lemma}[Balancing inequality]\label{lem:balance}
Let $\alpha\in(0,1]$ and $\gamma:=\frac{\alpha}{\alpha+1}$. For all integers $k'\ge1$,
\begin{equation}\label{eq:balance-half}
 \frac{1}{k'}+\Bigl(1-\frac{1}{k'}\Bigr)\Bigl(1+\bigl(2\,k'^{\gamma}\bigr)^{1/\alpha}\Bigr)^{-\alpha} \ \ge\ \frac{1}{2\,(k'+1)^{\gamma}}.
\end{equation}
\end{lemma}
\begin{proof}
Let $A:=(2\,k'^{\gamma})^{1/\alpha}\ge0$. Note that the function $u\mapsto (1+u)^{-\alpha}$ is convex and decreasing on $[0,\infty)$ for every $\alpha>0$, which implies that
$$
 (1+A)^{-\alpha} =A^{-\alpha}\Bigl(1+\frac{1}{A}\Bigr)^{-\alpha} \ \ge\ A^{-\alpha}\Bigl(1-\frac{\alpha}{A}\Bigr) =\frac{1}{2\,k'^{\gamma}}-\frac{\alpha}{2^{1+1/\alpha}\,k'}
$$
for $A>0$. Substituting this lower bound into the left-hand side gives
$$
 \begin{aligned}
 \frac{1}{k'}+\Bigl(1-\frac{1}{k'}\Bigr)(1+A)^{-\alpha} &\geq \frac{1}{k'}+\Bigl(1-\frac{1}{k'}\Bigr)\!\left(\frac{1}{2\,k'^{\gamma}}-\frac{\alpha}{2^{1+1/\alpha}\,k'}\right)\\
 &= \frac{1}{2\,k'^{\gamma}}+\frac{1}{k'}-\frac{1}{2\,k'^{\gamma+1}} -\frac{\alpha}{2^{1+1/\alpha}\,k'}+\frac{\alpha}{2^{1+1/\alpha}\,k'^2}.
 \end{aligned}
$$
Since $k'^{\gamma}\ge1$, we know that $-\frac{1}{2\,k'^{\gamma+1}}\geq -\frac{1}{2k'}$ and $\frac{\alpha}{2^{1+1/\alpha}\,k'^2}\ge0$. Moreover, $2^{-1/\alpha}\le\frac12$ for $\alpha\in(0,1]$, which implies that $\frac{\alpha}{2^{1+1/\alpha}}\le\frac{\alpha}{2}$, and therefore that
$$
 \frac{1}{k'}-\frac{1}{2\,k'^{\gamma+1}}-\frac{\alpha}{2^{1+1/\alpha}\,k'} \ \ge\ \frac{1}{k'}-\frac{1}{2k'}-\frac{\alpha}{2k'} \ =\ \frac{1-\alpha}{2k'}\ \ge\ 0
$$
(with equality only when $\alpha=1$). Combining these estimates gives
$$
 \frac{1}{k'}+\Bigl(1-\frac{1}{k'}\Bigr)(1+A)^{-\alpha} \ \ge\ \frac{1}{2\,k'^{\gamma}}+\frac{1-\alpha}{2k'}\ \ge\ \frac{1}{2\,k'^{\gamma}}.
$$
Finally, $(k'+1)^\gamma\geq k'^\gamma$ implies that $\frac{1}{2\,k'^{\gamma}}\geq \frac{1}{2\,(k'+1)^\gamma}$, proving \eqref{eq:balance-half}.
\end{proof}

\noindent We can now prove the master lower bound on $V(\tau', k')$ by induction on $(k', \tau')$. The base cases are immediate. The inductive step plugs the one-variable minimum into the recurrence and then uses the balancing inequality.

\begin{lemma}[Solving the recurrence]\label{lem:solve}
For all $\tau'\ge0$ and $k'\ge1$,
\begin{equation}\label{eq:IH-half}
 V(\tau',k')\ \ge\ \frac{m}{(\alpha+1)\tau^\alpha}\ \frac{(\tau')^{\alpha+1}}{2\,(k'+1)^{\gamma}}.
\end{equation}
\end{lemma}
\begin{proof}
We prove this by induction on $(k',\tau')$. \textit{Base cases.} If $\tau'=0$ there is nothing to show. If $k'=1$, then we have
$$
 V(\tau',1)=\mathrm{OPT}_{\tau'}\ \ge\ \sum_{s=1}^{\tau'} m\Bigl(\frac{s}{\tau}\Bigr)^\alpha \ \ge\ \frac{m}{(\alpha+1)\tau^\alpha}(\tau')^{\alpha+1} \ \ge\ \frac{m}{(\alpha+1)\tau^\alpha}\ \frac{(\tau')^{\alpha+1}}{2\cdot 2^\gamma},
$$
as $2\cdot 2^\gamma\geq 1$. \textit{Inductive step.} Now assume \eqref{eq:IH-half} holds for all $(k''\!,\tau'')$ with $k''<k'$ and $0 \leq \tau'' \leq \tau'$. From \eqref{eq:rec}, we know that
$$
 V(\tau',k')\ \ge\ \frac{1}{k'}\,\mathrm{OPT}_{\tau'+k'}+\Bigl(1-\frac{1}{k'}\Bigr)\min_{0\leq t\leq \tau'}\!\left[\frac{m}{(\alpha+1)\tau^\alpha}\,t^{\alpha+1}+V(\tau'-t,k'-1)\right]\!.
$$

Using CEE and $m \leq f^\star(T)(\tau/T)^\alpha$, recall that we have
$$
f^\star(s) \geq f^\star(T) \left( \frac{s}{T} \right)^{\beta_I} \geq f^\star(T) \left( \frac{s}{T} \right)^\alpha = f^\star(T) \left( \frac{\tau}{T} \right)^\alpha \left( \frac{s}{\tau} \right)^\alpha \geq m \left( \frac{s}{\tau} \right)^\alpha
$$
for all $s \leq T$, which implies that
$$
 \frac{1}{k'}\,\mathrm{OPT}_{\tau'+k'}\ \ge\ \frac{1}{k'}\sum_{s=1}^{\tau'} m\Bigl(\frac{s}{\tau}\Bigr)^\alpha\ \geq \ \frac{m}{(\alpha+1)\tau^\alpha}\,\frac{(\tau')^{\alpha+1}}{k'}.
$$
Applying the inductive hypothesis to $V(\tau'-t,k'-1)$ and letting $y:=t/\tau'\in[0,1]$ gives
$$
 \begin{aligned}
 V(\tau',k')&\geq \frac{m}{(\alpha+1)\tau^\alpha}\,(\tau')^{\alpha+1}\left[\frac{1}{k'}+\Bigl(1-\frac{1}{k'}\Bigr)\min_{y\in[0,1]}\Bigl\{y^{\alpha+1}+\frac{1}{2\,k'^{\gamma}}(1-y)^{\alpha+1}\Bigr\}\right]\!.
 \end{aligned}
$$
By Lemma~\ref{lem:exact-min}, the minimum equals $\bigl(1+(2k'^{\gamma})^{1/\alpha}\bigr)^{-\alpha}$. Lemma~\ref{lem:balance} then yields \eqref{eq:IH-half}.
\end{proof}

\noindent Finally, we apply Lemma \ref{lem:solve} with $(\tau', k') = (T - k, k)$ to obtain the desired competitive ratio.

\begin{theorem}[Upper bound]\label{thm:upper}
Assume $T\geq 2k$. If we run $PTRR_\alpha$ with $\tau = T-k$ and $m := \frac{T - k}{T} f^{\star}(T)$, we obtain
$$
 \mathbb{E}[\mathrm{reward}] \ge \frac{1}{2^{\alpha + 3}(\alpha + 1)} \cdot \frac{\mathrm{OPT}_T}{(k + 1)^{\gamma}},
$$ where $\gamma=\frac{\alpha}{\alpha+1}$. 
The competitive ratio is $O(k^{\alpha / (\alpha + 1)})$.
\end{theorem}
\begin{proof} ({\bf Overview.})
Two simple properties drive our analysis. First, the optimal arm is never abandoned when $\alpha > \beta_I$. Once we encounter $f^*$, we stay on it forever. Second, if we ever abandon a non‑optimal arm at time $t$, the cumulative reward collected on that arm is at least the “area” under the threshold up to $t$
$$
 \sum_{s=1}^{t-1} f_i(s) \ge \frac{m}{(\alpha + 1) \tau^\alpha}(t - 1)^{\alpha + 1}.
$$
\noindent These two facts yield a value recurrence. Let $V(\tau', k')$ be the expected reward from a state with $k'$ untouched arms and $\tau'$ `free'' pulls left. With probability $1/k'$ the next random pick is $f^*$, after which the algorithm earns $\mathrm{OPT}_{\tau' + k'}$. Otherwise it spends $t$ pulls on a bad arm, banks the area above, and recurses on $(\tau' - t, k' - 1)$. Minimizing over feasible $t$ gives
$$
 V(\tau', k') \ge \frac{1}{k'} \mathrm{OPT}_{\tau' + k'} + \left(1 - \frac{1}{k'}\right) \min_{t \in [0, \tau']} \left\{ \frac{m}{(\alpha + 1) \tau^\alpha} t^{\alpha + 1} + V(\tau' - t, k' - 1) \right\}.
$$
Induct on $(\tau',k')$. The recurrence reduces to a one-variable balance
$$
 \min_{y \in [0,1]} \left\{ u\, y^{\alpha+1} + v(1 - y)^{\alpha+1} \right\},
$$
which yields
$$
 V(\tau',k') \geq \frac{m}{(\alpha+1)\tau^{\alpha}} \cdot \frac{(\tau')^{\alpha+1}}{2(k'+1)^{\alpha/(\alpha+1)}}.
$$
Finally, plug in $(\tau, k)$ with $\tau=T-k$, use $T \geq 2k$ so $\tau \ge T/2$, and upper-bound $\text{OPT}_T \le c_2 \,m\,(T/\tau)^{\alpha}T$ to obtain the desired result.

({\bf Full proof of final step.}) Applying Lemma~\ref{lem:solve} with $(\tau',k')=(T-k,k)$ (where the $k$ extra time accounts for the ``switching'' steps when the keep-test first fails) and $T \geq 2k$ gives
$$
 \mathbb{E}[\mathrm{reward}]\ \ge\ \frac{m}{(\alpha+1)\tau^\alpha}\cdot \frac{\tau^{\alpha+1}}{2\,(k+1)^\gamma} \ =\ \frac{m\,\tau}{2(\alpha+1)}\cdot \frac{1}{(k+1)^\gamma}.
$$
Since $\mathrm{OPT}_T\leq f^*(T) \cdot T \le c_2\,m\,(T/\tau)^{\beta_I}T \leq c_2\,m\,(T/\tau)^{\alpha}T$ (by the CEE), we know that $m \tau \ge \frac{(\tau / T)^{1 + \alpha}}{c_2} \, \text{OPT}_T$, and therefore that
$$
 \mathbb{E}[\mathrm{reward}] \ \ge\ \frac{1}{2(\alpha+1)c_2}\left(\frac{\tau}{T}\right)^{1+\alpha}\cdot \frac{\mathrm{OPT}_T}{(k+1)^\gamma} \geq \frac{1}{2^{\alpha + 3}(\alpha + 1)} \cdot \frac{\mathrm{OPT}_T}{(k + 1)^{\gamma}},
$$
as desired.
\end{proof}

\subsection{Lower Bound}
\label{appendix:lowerbd-proof}
To prove a matching lower bound, we construct a hard family of instances with $\beta_I = \beta$ on which every algorithm has
expected approximation factor at least on the order of $k^{\beta/(\beta+1)}$. The idea is straightforward and mirrors the construction in \cite{pmlr-v272-blum25a}: one arm $g$ is defined as $m (t/T)^\beta$,
while the others all match it up to a time $s$ and then flatten so that no deterministic sequence of pulls can safely discard one before investing $s$ pulls. We let $\mathcal{D}_s$ denote the distribution that picks one arm uniformly at random to be $g$. Lemma~\ref{lem:PF-index} shows every instance drawn from $\mathcal{D}_s$ has $\beta_I=\beta$. For any deterministic schedule, Lemma~\ref{lem:generous} upper-bounds the expected reward by balancing $\mathrm{OPT}_T$ with the flat value $Tg(s)$. Lemma~\ref{lem:opt-lb} gives $\mathrm{OPT}_T\ge mT/(\beta+1)$. Combining yields Lemma~\ref{lem:ratio}: with $x:=s/T$, 
$$
\frac{\mathbb{E}[\mathrm{ALG}_T]}{\mathrm{OPT}_T}\ \le\ h(x):=\frac{1}{k x}+(\beta+1)x^\beta.
$$
Lemma~\ref{lem:xstar} minimizes $h$ at $x^\star=[k\beta(\beta+1)]^{-1/(\beta+1)}$ and gives $h(x^\star)=\Theta\!\big(k^{-\beta/(\beta+1)}\big)$. Setting $s=\lfloor x^\star T\rfloor$ and assuming $T\ge 2/x^\star$, Lemma~\ref{lem:round} controls discretization within a constant factor. This proves Theorem~\ref{thm:LB} for deterministic schedules. Yao’s principle extends it to randomized algorithms with the same exponent. As before, a more detailed overview of the proof is included after the theorem.

\paragraph{Hard family.}
Fix $k\ge2$, $T\ge1$, and $m>0$. For $s\in\{1,\dots,T\}$ define a good arm as
$$ 
g(t):=m(t/T)^\beta,\quad t=1,\dots,T.
$$
Define for this $s$ a bad arm by
$$
f_{\mathrm{bad}}(t)=
\begin{cases}
g(t),&t\le s,\\
g(s),&t>s.
\end{cases}
$$

\noindent Let $\mathcal{D}_s$ be the distribution that chooses one arm uniformly at random to be $g$ and
sets all other arms to $f_{\mathrm{bad}}$.

\begin{lemma}[Membership and Concavity Envelope Exponent]\label{lem:PF-index}
Every instance drawn from $\mathcal{D}_s$ has $\beta_I=\beta$.
\end{lemma}

\begin{proof}
For $g$, we know that LE$(\beta)$ holds with equality. Since $$f_{\mathrm{bad}}(T)=g(s)\le g(T)=m,$$
we likewise know that $$f_{\mathrm{bad}}(t)\ge g(s)(t/T)^\beta=f_{\mathrm{bad}}(T)(t/T)^\beta$$ for all $t\le T$ and every $f_{\mathrm{bad}}$.
It follows that LE$(\beta)$ holds for each arm, and therefore that $\beta_I\le\beta$. 

Now suppose for the sake of contradiction that $\beta'<\beta$. Note that
$$m(t/T)^\beta< m(t/T)^{\beta'}=g(T)(t/T)^{\beta'},$$ which implies that $g$ violates LE$(\beta')$. It follows that $\beta_I=\beta$.
\end{proof}

To show that no algorithm can outperform $\textit{PTRR}_\beta$ on this distribution, we will consider the following `generous game.' Fix a deterministic sequence $S$ of $T$ arm pulls. Let $A$ denote the number of distinct arms that receive at least $s$ pulls under $S$, and note that $A \leq \lfloor\frac{T}{s}\rfloor.$ Draw an arbitrary instance from $\mathcal D_s$, play $S$ on this instance, and only afterwards reveal which arm was good. If $S$ gave arm $g$ at least $s$ pulls, pay out a reward of $\mathrm{OPT}_T$. Else provide the actual reward that $S$ obtained. Note that paying $\mathrm{OPT}_T$ in certain cases can only increase the payoff of the sequence, which implies that the payoff of this game upper-bounds the expected reward $\mathbb{E}_{\mathcal{D}_s} [\mathrm{ALG}_T]$.

\begin{lemma}[Generous evaluation]\label{lem:generous}
For any deterministic algorithm $\mathrm{ALG}$, let 
$S$ be the sequence played by $\mathrm{ALG}$ when all arms are $f_\text{bad}$, and let $A$ be the number of arms that receive at least $s$ pulls under $S$.
Then 
\begin{equation}\label{eq:gen}
\mathbb{E}_{\mathcal{D}_s}[\mathrm{ALG}_T]
\ \le\ \frac{A}{k}\,\mathrm{OPT}_T\ +\ \Bigl(1-\frac{A}{k}\Bigr)\,T\,g(s)
\ \le\ \frac{T}{k s}\,\mathrm{OPT}_T\ +\ T\,m\Bigl(\frac{s}{T}\Bigr)^\beta.
\end{equation}
\end{lemma}

\begin{proof}
Draw an instance from $\mathcal{D}_s$ and run $S$. With probability $A/k$, the good arm lies among those $A$ arms, and the reward is at most $\mathrm{OPT}_T$. Otherwise the good arm is pulled $<s$ times, which implies that any pull is worth at most $g(s)$, and therefore that the realized reward is at most $T \cdot g(s) = T \cdot m(s/T)^\beta$. Taking expectations gives the desired reward.
\end{proof}

\noindent We now bound $\text{OPT}_T$.
\begin{lemma}[Lower bound on the benchmark]\label{lem:opt-lb}
$$\mathrm{OPT}_T\ \ge\ \sum_{t=1}^{T} g(t)\ \ge\ \dfrac{mT}{\beta+1}.$$
\end{lemma}

\begin{proof}
Note that $$\mathrm{OPT}_T = \sum_{t=1}^{T} g(t)= m\sum_{t=1}^{T}(t/T)^{\beta} \ge mT^{-\beta}\!\int_{0}^{T} x^{\beta}\,dx = mT/(\beta+1),$$
as $x^{\beta}$ is non-decreasing on \([0,T]\) for \(\beta>0\).
\end{proof}

\noindent Combining the earlier generous bound with this lower bound gives a clean formula for the approximation ratio as a function of $x=s/T.$

\begin{lemma}[Ratio at a given $s$]\label{lem:ratio}
For any deterministic algorithm $\mathrm{ALG}$ and $x:=s/T\in\{1/T,\dots,1\}$,
$$
\frac{\mathbb{E}_{\mathcal{D}_s}[\mathrm{ALG}_T]}{\mathrm{OPT}_T}
\ \le\ h(x):=\frac{1}{k x}+(\beta+1)\,x^\beta.
$$
\end{lemma}

\begin{proof}
Divide \eqref{eq:gen} by Lemma~\ref{lem:opt-lb}.
\end{proof}

\noindent Simple calculus gives a unique minimizer $x^*$ and the exact value $h(x^*)$.

\begin{lemma}[Continuous minimizer]\label{lem:xstar}
$h$ has a unique minimizer on $(0,1]$ at
$$
x^\star\ :=\ \bigl[k\,\beta(\beta+1)\bigr]^{-1/(\beta+1)}\,,
$$
and
$$
h(x^\star)\ =\ (\beta+1)^2\,[\beta(\beta+1)]^{-\beta/(\beta+1)}\,k^{-\beta/(\beta+1)}.
$$
\end{lemma}

\begin{proof}
First, note that $$h'(x)=-1/(k x^2)+(\beta+1)\beta x^{\beta-1} = \frac{(\beta+1)\beta\,x^{\beta+1}-\tfrac1k}{x^{2}}$$ is strictly increasing in $x>0$, which implies that it has a unique minimum $$
(\beta+1)\beta\,(x^\star)^{\beta+1}=\frac1k
\quad\Longrightarrow\quad
x^\star=\big[k\,\beta(\beta+1)\big]^{-1/(\beta+1)}.
$$

\noindent At \(x^\star\), we use the first‑order condition to get
$$
h(x^\star)=\frac{1}{k x^\star}+(\beta+1)(x^\star)^\beta
=(\beta+1)\beta(x^\star)^\beta+(\beta+1)(x^\star)^\beta
=(\beta+1)^2(x^\star)^\beta.
$$
Since \((x^\star)^\beta=[k\,\beta(\beta+1)]^{-\beta/(\beta+1)}\), it follows that
\[
h(x^\star)=(\beta+1)^2\,[\beta(\beta+1)]^{-\beta/(\beta+1)}\,k^{-\beta/(\beta+1)}.
\]

\end{proof}

\noindent Recall that $s$ must be an integer. Following \cite{pmlr-v272-blum25a}, we take $s = \lfloor x^* T \rfloor$ and require $x^* T \geq 2$. This simple condition ensures $s/T$ lies between $x^*/2$ and $x^*$, inflating $h$ by at most a constant factor.
\begin{lemma}[Rounding]\label{lem:round}
Let $s:=\lfloor x^\star T\rfloor$. If $x^\star T\ge2$, then $s/T\in[x^\star/2,\,x^\star]$ and
$$
h\!\Bigl(\frac{s}{T}\Bigr)\ \le\ \frac{3}{2}\,h(x^\star).
$$
\end{lemma}

\begin{proof}
Since $x^\star T\ge2$, we know that $s/T\ge (x^\star T-1)/T\ge x^\star/2$, and trivially $s/T\le x^\star$.
Let $s/T=a x^\star$ with $a\in[1/2,1]$. Using $1/(k x^\star)=(\beta+1)\beta (x^\star)^\beta$, it follows that
$$
\frac{h(a x^\star)}{h(x^\star)}
=\frac{\beta/a + a^\beta}{\beta+1}
\ \le\ \frac{2\beta + 1}{\beta+1}
\ \le\ \frac{3}{2},
$$
since $1/a\le2$ and $a^\beta\le1$ for $a\in[1/2,1]$, $\beta\in(0,1]$.
\end{proof}

\noindent We now plug our choice of $s$ into the ratio bound. Because the bound holds for every deterministic schedule against $D_s$, it also holds for any randomized algorithm by linearity of expectation.
\begin{theorem}[Lower Bound]\label{thm:LB}
Fix $\beta\in(0,1]$ and $k\ge2$. Suppose
$
T\ \ge\ \frac{2}{x^\star}\ =\ 2\,[\beta(\beta+1)]^{\frac{1}{\beta+1}}\,k^{\frac{1}{\beta+1}}.
$
Then there exists a distribution on instances with $\beta_I=\beta$ such that for every (possibly randomized) algorithm,
$$
\frac{\mathbb{E}[\mathrm{ALG}_T]}{\mathrm{OPT}_T}
\ \le\ \frac{3}{2}\,(\beta+1)^2\,[\beta(\beta+1)]^{-\beta/(\beta+1)}\,k^{-\beta/(\beta+1)}.
$$

\end{theorem}

\begin{proof} ({\bf Overview.})
We construct a similar hard family to \cite{pmlr-v272-blum25a}. Define a ``good'' arm as the power curve $g(t) = m(t/T)^\beta$. Let the other $k - 1$ arms copy $g$ up to a breakpoint $s$ and then flatten. Every arm satisfies $\text{LE}(\beta)$, and the good arm violates any stricter floor, so $\beta_I = \beta$. Let $\mathcal{D}_s$ denote the distribution that picks one arm uniformly at random to be $g$.

Against any fixed schedule $S$ of length $T$, let $A$ denote the number of distinct arms that receive at least $s$ pulls under $S$. Note that a ``generous'' evaluation that pays $\text{OPT}_T$ if the good arm lies among those $A$ and otherwise pays at most $Tg(s)$ provides an upper-bound for the expected reward of $S$ under a random draw from $D_s$. Since $A \le \lfloor T/s \rfloor$ and $\mathrm{OPT}_T \ge mT/(\beta + 1)$, it follows that
$$
\frac{\mathbb{E}[\mathrm{ALG}_T]}{\mathrm{OPT}_T} \le \frac{1}{kx} + (\beta + 1)x^\beta \quad \text{where } x := s/T \in (0,1].
$$
Call the right-hand side $h(x)$. Simple calculus gives a unique minimizer
$$
x^* = \left[k\beta(\beta + 1)\right]^{-1/(\beta+1)}, \quad \text{ which evaluates to } h(x^*) = (\beta + 1)^2 \left[\beta(\beta + 1)\right]^{-\beta/(\beta+1)} k^{-\beta/(\beta+1)}.
$$
Let $s = \lfloor x^\star T \rfloor$ and $T \ge 2/x^\star$. Using $\frac{1}{k x^*} = \beta (\beta + 1) (x^*)^{\beta}$, we get $h(s/T) \leq \tfrac{3}{2} h(x^*),$ which implies that
$$
\frac{\mathbb{E}[\mathrm{ALG}_T]}{\mathrm{OPT}_T}\ \le\ \tfrac{3}{2} h(x^\star) = \frac{3}{2}\,(\beta+1)^2\,[\beta(\beta+1)]^{-\beta/(\beta+1)}\,k^{-\beta/(\beta+1)}.
$$
By Yao’s principle, the same bound holds for randomized algorithms.

({\bf Full proof of final step.}) Let $s=\lfloor x^\star T\rfloor$, and note that the condition on $T$ ensures $x^\star T\ge2$. Combine Lemmas~\ref{lem:ratio}, \ref{lem:xstar}, and \ref{lem:round} to bound the ratio for any deterministic schedule against $\mathcal{D}_s$. Since this bound holds for every deterministic schedule, it also holds for any randomized algorithm by linearity of expectation (by Yao’s principle). Lemma~\ref{lem:PF-index} guarantees $\beta_I=\beta$ for the constructed family.
\end{proof}

\subsection{Doubling Trick for Unknown $T$}\label{appendix:unknownT}
We remove the need to know $T$ by using the doubling schedule and exploration procedure of \cite{pmlr-v272-blum25a}. At a high level, we start with a guess $T_0 = 4k$, pretend this is the horizon, and spend $T'/2$ steps using the same adaptation of \textbf{Algorithm 2} of \cite{pmlr-v272-blum25a} that they use to estimate $\widehat{m}$, the reward of the best arm, for the relevant time scale. After this, we spend $T'/2$ steps exploiting (running $PTRR_\alpha$) with $\tau' := \frac{T'}{2} - k$ and $m=\frac{\widehat{m}}{2\cdot16^\alpha}$ (we shrink $\hat{m}$ to ensure we don't discard the best arm). 
If time remains after $T'$ steps, we double $T'$ and repeat. This yields the below result:

\begin{theorem}[Unknown-$T$ guarantee]\label{thm:unknownT-final}
Assume $\beta_I \in (0,1)$ and fix $\alpha \in (\beta_I, 1)$, $\gamma = \alpha / (\alpha + 1)$. If $T > 4k$, then the doubling trick above achieves
$$
\mathbb{E}[\mathrm{ALG}_T] \ge \frac{1}{2048 \cdot 16^\alpha (\alpha + 1) \log(128k)} \cdot \frac{\mathrm{OPT}_T}{(k + 1)^\gamma}.
$$
\end{theorem}
\begin{proof}
Consider $i$ such that $T' = 2^i \, T_0$ is the last iteration for which an explore / exploit cycle was  completed. Define $T'' \coloneqq \sum_{j = 0}^{i} 2^j T_0 = (2^{i+1} - 1) T_0 $. Then $T'' < T \le 2 T''\,.$ We will argue two things: first, we will show that spending $T'/2$ time on the procedure for estimating $m$ provides a sufficiently good estimate for $m$ for our purposes; second, we will quantify the gap between spending $T'/2$ collecting reward from the instance based on our estimated $m$ and $T$ time spent on the optimal arm.

Let $\tau := T'/2 - k$. Let $\hat{m}$ denote the estimated maximum value at horizon $\tau + k$ from running the estimation procedure for time $T'/2$ (from time $T''-T' $ to $T''-T'/2$). From the analysis in \cite{pmlr-v272-blum25a}, we know that $\frac{1}{2} \widehat{m} \le f^{\star}(\tau) \le 2\widehat{m}$ with probability at least $\frac{1}{\log(128k)}$, and that
$$
    \frac{T}{\tau}= \frac{T}{\frac{T'}{2}-k} \le \frac{T}{\frac{T'}{4}} = 4 \, \frac{T}{T''} \frac{T''}{T'} \le 4 \cdot 2 \cdot 2 = 16\,.$$
Let $m := \frac{\hat m}{2 \cdot 16^\alpha}$, and run $\textit{PTRR}_\alpha$ with $m$, $\tau$ for $\frac{T'}{2}$ steps.

Since $\tau/T \ge 1/16$ and $\widehat{m} \le 2 f^{\star}(\tau) \le 2 f^{\star}(T)$, we know that
$$
m \le \frac{2 f^{\star}(T)}{2 \cdot 16^\alpha} \le f^{\star}(T) \left( \frac{\tau}{T} \right)^{\alpha},
$$
and therefore that 
$$
m \left( \frac{t}{\tau} \right)^\alpha \le f^{\star}(T) \left( \frac{t}{T} \right)^\alpha \le f^{\star}(T) \left( \frac{t}{T} \right)^{\beta_I} \le f^{\star}(t)
$$ for all $t \le T$. It follows that $\textit{PTRR}_\alpha$ again does not switch away from the best arm $f^*$. 

Now, we proceed to the second part of our argument.
Applying Lemma~\ref{lem:solve} with $(\tau', k') = (\tau, k)$ gives
$$
\mathbb{E}[\text{exploit}] \ge \frac{m \tau}{2(\alpha + 1)(k + 1)^\gamma}, \qquad \gamma = \frac{\alpha}{\alpha + 1}.
$$
Since $m = \widehat{m}/(2 \cdot 16^\alpha)$ and $\tau / T \ge 1/16$, it follows that
$$
\mathbb{E}[\text{exploit}] \ge \frac{\widehat{m} \tau}{4(\alpha + 1) 16^\alpha (k + 1)^\gamma} \ge \frac{\widehat{m} T}{64(\alpha + 1) 16^\alpha (k + 1)^\gamma}.
$$

Since $f^{\star}(T) \le 16 f^{\star}(\tau) \le 32 \widehat{m}$, it further follows that $\mathrm{OPT}_T \le f^{\star}(T) T \le 32 \widehat{m} T$, and therefore that 
$$
\frac{\mathbb{E}[\text{exploit}]}{\mathrm{OPT}_T} \ge \frac{1}{2048 (\alpha + 1) 16^\alpha} \cdot \frac{1}{(k + 1)^\gamma}.
$$
Multiplying by the $1 / \log(128k)$ selection probability gives
$$
\mathbb{E}[\mathrm{ALG}_T] \ge \frac{1}{2048 (\alpha + 1) 16^\alpha \log(128k)} \cdot \frac{\mathrm{OPT}_T}{(k + 1)^\gamma}.
$$
\end{proof}

\section{Details for Sample Complexity Analysis}
We now provide some background from prior work and full proofs for the sample complexity of tuning the $\alpha$ parameter in our algorithm family $PTRR_\alpha$.

\subsection{Results from Prior Work \cite{sharma2025offline}} 
\label{appendix:ss25results}
\begin{definition}[Definition 1 in Arxiv verson of \cite{sharma2025offline}]
    Suppose the derandomized dual function $l_T^{I, \pmb{z}}(\rho)$ is a piecewise constant function. The derandomized dual complexity of $\mathcal{D}$ w.r.t. instance I is given by $\mathcal{Q}_\mathcal{D} \coloneqq \E{I \sim \mathcal{D}}{\E{z}{q\left(l_T^{I, \pmb{z}}(\cdot)\right)}}\,,$ where $q(\cdot)$ is the number of pieces over which the function is piecewise constant.
\end{definition}

\noindent Having defined this complexity measure that is capable of characterizing the complexity of an algorithm family of interest, we now present the main theorem we apply from \cite{sharma2025offline}, a uniform convergence guarantee on learning the best value of the parameter $\alpha$:
\begin{theorem}[Theorem 6.1 in arXiv Version of \cite{sharma2025offline}] \label{thm:ss25main}
    Consider the Hyperparameter Transfer setup for any arbitrary $\mathcal{D}$\footnote{To reiterate, we now consider $\mathcal{D}$ supported over $\mathcal{I} \times \Pi_k\,,$ where $\Pi_k = \{ \pi_k\}$ is the set of permutations of $[k]$.} and suppose the derandomized dual function $l_T^{I, \pmb{z}}(\alpha)$ is a piecewise constant function. For any $\epsilon, \delta > 0, N$ problems $\{I_i\}_{i = 1}^N$ sampled from $\mathcal{D}$ with corresponding random coins $\{\pmb{z}_i\}_{i = 1}^N$ such that $N = O\left( \left(\frac{H}{\epsilon}\right)^2 (\log \mathcal{Q}_\mathcal{D} + \log \frac 1 \delta )  \right)$ are sufficient to ensure that with probability at least $1-\delta,$ for all $\pmb{\alpha} \in \mathcal{P},$ we have that:
    
    \begin{equation*}        
    \left| \frac 1N \sum_{i = 1}^N l_T^{P_i, \pmb{z}_i}(\pmb{\alpha}) - \E{P \sim \mathcal{D}}{l_T^P(\pmb{\alpha})}    \right| < \epsilon
    \end{equation*}

\end{theorem}

Note that by the argument below, this guarantee solves the Hyperparameter Transfer Setting described in Definition~\ref{defn:hypertransf}.

\subsection{Uniform Convergence Implies Population Loss Near-Optimality}
\label{appendix:ucpoploss}
We recall the standard argument below for completeness. 

\begin{lemma}
    [Uniform convergence implies near-optimality in population loss]
    Suppose $\hat{\alpha}$ is the minimizer of $\frac 1N \sum_{i = 1}^N l(P_i, \alpha)$ for $N$ large enough to satisfy uniform convergence with error $\epsilon$ and $\alpha^\star \coloneqq \arg \min_{\alpha} \E{P \sim \mathcal{D}}{l_T(P, \alpha)}.$ Then, 
    $$
    \left |\E{P \sim \mathcal{D}}{l_T(P, \hat{\alpha})} - \E{P \sim \mathcal{D}}{l_T(P, \alpha^\star)} \right | < 2\epsilon\,.
    $$
\end{lemma}
\newcommand{\emploss}[1]{\frac 1N \sum_{i = 1}^N l_T(P_i, #1)}

\begin{proof}
    We can see this by adding and subtracting empirical losses:
    \begin{align}
        \left |\E{P \sim \mathcal{D}}{l_T(P, \hat{\alpha})} - \E{P \sim \mathcal{D}}{l_T(P, \alpha^\star)} \right | &= 
        \bigg |\E{P \sim \mathcal{D}}{l_T(P, \hat{\alpha})} - \emploss{\hat{\alpha}}\\ &+ \emploss{\hat{\alpha}} - \emploss{\alpha^\star} \label{eqn:neg}\\&+ \emploss{\alpha^\star} - \E{P \sim \mathcal{D}}{l_T(P, \alpha^\star)} \bigg | \\
        &\le \bigg |\E{P \sim \mathcal{D}}{l_T(P, \hat{\alpha})} - \emploss{\hat{\alpha}} \bigg | + \\
        &+ \bigg| \emploss{\alpha^\star} - \E{P \sim \mathcal{D}}{l_T(P, \alpha^\star)} \bigg | \\
        &\le 2\epsilon
    \end{align}
    Note that the term in Eqn.~\ref{eqn:neg} is negative, since $\hat{\alpha}$ is the minimizer of the empirical loss, and the last inequality follows from the uniform convergence guarantee. 
\end{proof}

\subsection{Average Regret} \label{appendix:avgregret}

Now, let us consider specific instantiations for an $H$-bounded loss function.

\begin{definition} \label{defn:avgregret}
    Suppose at time $t$, the reward collected by the algorithm is $r_t.$ Define the {\em average regret} as:

    $$
    R(T) \coloneqq\frac{1}{T} \left( \max_i \sum_{t = 1}^T f_i(t) -  \sum_{t = 1}^T r_t \right)\,.
    $$
    That is, the regret in hindsight is in reference to the best fixed arm.
\end{definition}

\begin{fact}
    The average regret is $H$-bounded for $H = m\,.$
\end{fact}

\noindent Thus, if we are interested in solving the problem with respect to average regret, we get the sample complexity below:
\begin{corollary}
        For the Hyperparameter Transfer setting for the improving multi-armed bandits problem optimizing for averaged regret, $N = O\left( \left(\frac{m}{\epsilon}\right)^2 (\log kT + \log \frac 1 \delta )  \right)$ instances drawn from $\mathcal{D}$ suffice to get the uniform convergence guarantee in Theorem~\ref{thm:ss25main}.
\end{corollary}

Note that the guarantee of this corollary is that the parameters found will produce a sequence of pulls that has regret that is close to the regret of the sequence of pulls induced by the {\em best set of parameters}.

\subsection{Computational Complexity of ERM}
\label{appendix:erm-alg}
In Section~3.3 we analyze the sample complexity of selecting $\alpha$ by ERM over
the family $\mathcal{A}=\{\text{PTRR}_\alpha:\alpha\in(0,1]\}$ in the
Hyperparameter Transfer Setting (Definition~2.6), using the uniform convergence
results from Appendix~B.1 and
Appendix~B.2. Here we add to that analysis by showing that, given offline
access to the full learning curves, the ERM minimizer over the
continuous domain $\mathcal{P}=(0,1]$ can be computed exactly. Our implementation is efficient for small $k$, and we leave open the question of efficient exact or approximate implementation for large $k$.

As in Section~3.3, we derandomize Algorithm~1 by augmenting each instance with a
permutation $z\in\Pi_k$ that fixes the order in which arms are sampled from $S$.
We therefore work with i.i.d.\ samples $(I_1,z_1),\dots,(I_N,z_N)\sim (D')^N$. For
fixed $(I,z)$ and $\alpha$, the execution of $\text{PTRR}_\alpha$ is
deterministic. Let $l^{I,z}_T(\alpha)$ denote the corresponding (derandomized)
dual loss value at horizon $T$ (Appendix~B.1). 

If $S=\emptyset$ while $t<T$, we halt (this does not occur for appropriate alpha, but ensures that the algorithm is well-defined for
all $\alpha$ and does not affect the arguments below).

\paragraph{Finite critical set for fixed $(I,z)$}.

Fix an augmented instance $(I,z)$ with reward curves
$\{f_i(t)\}_{i\in[k],\,t\in\{0,1,\dots,T\}}$, and recall that Algorithm~1 makes $\alpha$-dependent decisions through the keep-test
\begin{equation}
\label{eq:ptrr_keep_test_app}
f_i(t_i)\ \ge\ m\Big(\frac{t_i}{\tau}\Big)^\alpha,
\qquad t_i\in\{0,1,\dots,T\}.
\end{equation}
For each pair $(i,t)$ with $i\in[k]$ and $t\in\{1,\dots,T\}\setminus\{\tau\}$,
define $c_{i,t}\in(0,1]$ to be the unique solution (if it exists) to the equality
\begin{equation}
\label{eq:critical_point_def_app}
f_i(t)\ =\ m\Big(\frac{t}{\tau}\Big)^\alpha,
\qquad \alpha\in(0,1].
\end{equation}
Uniqueness holds because for fixed $t\neq\tau$, the map $\alpha\mapsto m(t/\tau)^\alpha$
is strictly monotone. Let the critical set be
$$
\mathcal{C}(I)
\ :=\
\big\{c_{i,t}\in(0,1] : i\in[k],\ t\in\{1,\dots,T\}\setminus\{\tau\},
\text{ and \eqref{eq:critical_point_def_app} has a solution in }(0,1]\big\}.
$$
Then clearly $|\mathcal{C}(I)|\le kT$.

\begin{lemma}[Constancy between critical values]
\label{lem:constancy_between_critical}
Fix $(I,z)$. The function $\alpha\mapsto l^{I,z}_T(\alpha)$ is constant on each
connected component of $(0,1]\setminus \mathcal{C}(I)$.
Moreover, any change in $l^{I,z}_T(\alpha)$
can occur only at $\alpha\in\mathcal{C}(I)$.
\end{lemma}

\begin{proof}
For each $i\in[k]$ and $t\in\{0,1,\dots,T\}$, define the predicate
$$
\text{Test}_{i,t}(\alpha)
\ :=\
\mathbf{1}\!\left[\, f_i(t)\ \ge\ m\Big(\frac{t}{\tau}\Big)^\alpha \,\right].
$$
When $t=\tau$, the right-hand side equals $m$ and is independent of $\alpha$.
When $t\neq\tau$, strict monotonicity of $\alpha\mapsto m(t/\tau)^\alpha$ implies that
$\mathsf{Test}_{i,t}(\alpha)$ can change value only at the unique equality solution
$c_{i,t}$ (if it exists). Therefore, on any connected component
$U\subset(0,1]\setminus\mathcal{C}(I)$, every predicate $\mathsf{Test}_{i,t}(\alpha)$
is constant for all $\alpha\in U$.

Now fix $\alpha,\alpha'\in U$ and couple the executions of $\text{PTRR}_\alpha$ and
$\text{PTRR}_{\alpha'}$ on the same augmented instance $(I,z)$.
Because $z$ fixes the order of arms sampled from $S$, the only remaining branching
in Algorithm~1 is the outcome of the keep-test \eqref{eq:ptrr_keep_test_app} at the
current arm and current pull-count. Each such branch is determined by some
$\mathsf{Test}_{i,t}(\cdot)$, and all of these predicates agree on $U$.
Therefore both executions take identical branches at every step and therefore generate
the same pull sequence and the same accumulated reward, which implies that $l^{I,z}_T(\alpha)=
l^{I,z}_T(\alpha')$. This proves constancy on $U$ and that trace changes can occur
only at $\alpha\in\mathcal{C}(I)$.
\end{proof}

Because Algorithm~1 uses a non-strict inequality in \eqref{eq:ptrr_keep_test_app},
the value of $l^{I,z}_T(\alpha)$ at $\alpha\in\mathcal{C}(I)$ can differ from
the constant values on the adjacent open intervals. Therefore, any exact ERM
procedure must treat critical points as candidates.

\paragraph{Exact ERM over $\alpha\in(0,1]$}.

Given i.i.d.\ samples $(I_1,z_1),\dots,(I_N,z_N)\sim(D')^N$, define the empirical
objective
\[
\widehat{L}_N(\alpha)
\ :=\
\frac{1}{N}\sum_{j=1}^N l^{I_j,z_j}_T(\alpha),
\qquad \alpha\in(0,1].
\]
We show how to compute an exact minimizer
$\widehat{\alpha}\in\arg\min_{\alpha\in(0,1]} \widehat{L}_N(\alpha)$. For each sample $j$, form $\mathcal{C}(I_j)$ and sort its distinct elements as
$$
0 < c_{j,1} < \cdots < c_{j,q_j} \le 1,
\qquad q_j \le kT.
$$
Define the open intervals induced by these breakpoints:
\[
U_{j,0}:=(0,c_{j,1}),\quad
U_{j,r}:=(c_{j,r},c_{j,r+1})\ (r=1,\dots,q_j-1),\quad
U_{j,q_j}:=(c_{j,q_j},1),
\]
(with the convention that if $q_j=0$ then $U_{j,0}=(0,1)$).
By Lemma~\ref{lem:constancy_between_critical}, we know that $l^{I_j,z_j}_T(\alpha)$ is constant
on each $U_{j,r}$, but may take a different value at each $\alpha=c_{j,r}$. Define the global critical set
$$
\mathcal{C}_{\mathrm{all}}
\ :=\
\bigcup_{j=1}^N \mathcal{C}(I_j),
\qquad |\mathcal{C}_{\mathrm{all}}|\le NkT,
$$
and let its distinct elements be
$0 < b_1 < \cdots < b_M < 1$ (possibly $M=0$).

\begin{theorem}[Running time for implementing exact ERM for $\text{PTRR}_\alpha$ over the derandomized loss function]
\label{thm:exact_erm_ptrr}
Given offline access to full reward curves for each training instance, an exact
empirical minimizer
$\widehat{\alpha}\in\arg\min_{\alpha\in(0,1]} \widehat{L}_N(\alpha)$
can be computed in time
$$
O\!\left(NkT^2\ +\ NkT\log(NkT)\right).
$$
\end{theorem}

\begin{proof}
We first show that it suffices to minimize $\widehat{L}_N(\alpha)$ over a finite
candidate set. By Lemma~\ref{lem:constancy_between_critical},
for each $j$, the function $l^{I_j,z_j}_T(\alpha)$ is constant on each open interval
of $(0,1]\setminus\mathcal{C}(I_j)$. Therefore the empirical average
$\widehat{L}_N(\alpha)$ is constant on each open interval of
$(0,1]\setminus\mathcal{C}_{\mathrm{all}}$. The only points at which
$\widehat{L}_N(\alpha)$ can differ from these open-interval constants are the
breakpoints in $\mathcal{C}_{\mathrm{all}}$ and the endpoint $\alpha=1$, which implies that an exact minimizer is attained by some $\alpha$ in the finite set
$$
\mathcal{A}_{\mathrm{cand}}
\ :=\
\Big(\{b_1,\dots,b_M,1\}\Big)
\ \cup\
\Big\{\text{one representative }\tilde\alpha_r\in(b_r,b_{r+1})
:\ r=0,\dots,M\Big\}
$$
(where we set $b_0:=0$ and $b_{M+1}:=1$).

We now describe an evaluation procedure for $\mathcal{A}_{\mathrm{cand}}$ and bound
its runtime. For each sample $j$, we compute
(i) the constant value of $l^{I_j,z_j}_T(\alpha)$ on each local interval $U_{j,r}$
by simulating $\text{PTRR}_\alpha$ at one representative point in $U_{j,r}$, and
(ii) the value at each local breakpoint $c_{j,r}$ by simulating at $\alpha=c_{j,r}$,
and also simulate once at $\alpha=1$.
Each simulation performs at most $T$ pulls and therefore runs in $O(T)$ time.
Since each sample has at most $q_j\le kT$ breakpoints and $q_j+1\le kT+1$ open
intervals, this requires $O((2q_j+2)\cdot T)=O(kT^2)$ time per sample and
$O(NkT^2)$ time total.

Next, we sort the global breakpoint set $\mathcal{C}_{\mathrm{all}}$ in
$O(NkT\log(NkT))$ time. Sweep $\alpha$ from left to right across the sorted list.
On each global open interval $(b_r,b_{r+1})$, compute $\widehat{L}_N(\tilde\alpha_r)$
using the current local-interval value for each sample. At each breakpoint $b_r$,
compute $\widehat{L}_N(b_r)$ by using the precomputed breakpoint value for those
samples that have a local breakpoint at $b_r$, and otherwise using the current
local-interval value. Track the best value encountered across all candidates in
$\mathcal{A}_{\mathrm{cand}}$, including $\alpha=1$. This yields an exact minimizer
over $(0,1]$.

The sweep itself is linear in the number of breakpoint events, $O(NkT)$, and is
dominated by the $O(NkT^2)$ simulation cost and the $O(NkT\log(NkT))$ sorting cost.
\end{proof}

This is a purely computational result, which shows that the ERM
minimizer assumed in Appendix~B.2 can be computed exactly over the continuous
parameter space $\alpha\in(0,1]$ from offline learning curves in time
$$
O\!\left(NkT^2\ +\ NkT\log(NkT)\right).
$$

\noindent We can use this result for implementing exact ERM over the derandomized dual loss function to implement exact ERM over the true dual loss function (which involves expectation over the randomization of PTRR$_\alpha$) by taking an average over $k!$ permutations of the arms.

\subsection{An Algorithm for Finding a Suitable Value of the Parameter}
\label{appendix:approx-learning-alpha}

In this section, we provide a natural algorithm for finding a suitable value of the parameter $\alpha$ for the PTRR algorithm based on access to instances. We then analyze the sample complexity, runtime complexity, and performance of the algorithm. \par

At a high level, the algorithm starts by (approximately) identifying the underlying value of the CEE, $\beta_i$ for each instance $i$ in the training set. Then, we aggregate estimated values $\hat{\beta}_i$ by taking the maximum and that aggregated value (or 1, if it is bigger than 1) is returned as a value of $\alpha$ to use in future deployments. We now present a very general version of the result.

\begin{algorithm}
\caption{Approximate Learner}\label{alg:learn-alpha}
\begin{algorithmic}[1]
\REQUIRE $n$ samples from distribution $\mathcal{D}$
\FOR{each instance $I_i$}
    \STATE $\hat{\beta}_i \leftarrow \texttt{CEE\_Oracle}(I_i)$
\ENDFOR
\STATE $\beta \gets \max_i  \hat{\beta}_i $
\RETURN $\beta$
\end{algorithmic}
\end{algorithm}



For both steps, we can adjust the sophistication of the protocols. For the first step, we start by assuming we have an oracle that returns $CEE(I)$ for any instance $I\,.$ We can implement an approximate oracle by dividing the interval $(0, 1]$ into subintervals of size $\epsilon$ and binary searching over that discrete set for the smallest satisfying value. Similarly, for the second step, it is easiest to analyze the aggregation protocol that simply takes the maximum. However, this has the awkward property that with more samples, the predicted parameter increases toward the supremum of the support of the induced distribution over $\beta$s. Instead, we ought to use a $1-\delta$ order statistic that concentrates around its mean. This implies discarding some probability mass as part of the failure probability.

At a high level, in order to argue that this algorithm provides interesting, non-vacuous guarantees on the performance of PTRR with the learned parameter on a new instance drawn from the same distribution, we argue that the learned value of the parameter is good for a good fraction of the support of the distribution. To argue about this, we first define the induced distribution over $\beta$:

\begin{definition}
    For each instance $I_i$ drawn from the distribution $\mathcal{D}\,,$ suppose we compute $\beta_i = CEE(I).$ We call the distribution over the $\beta_i$ the {\em induced distribution over } $\beta$s.
\end{definition}

Next, we define the optimal parameter $\alpha^\star$ for the distribution. 

\begin{definition}
    Define $\alpha^\star$ as the maximizer of the expected competitive ratio over the distribution of instances. In particular, $\alpha^\star \coloneqq \arg \max_{\alpha \in (0, 1]} \E{I_\text{test} \sim \mathcal{D}}{\frac{Rew(PTRR_\alpha, I_\text{test})}{OPT(I_\text{test})}}\,.$
\end{definition}

Now, assuming the induced distribution over $\beta$ for a distribution $\mathcal{D}$ has some density and cumulative density, we can show that, for a fixed number of samples, with quantifiable probability (approaching 1 as the number of samples goes to infinity), we can bound the reward ratio in terms of the learned parameter. We formalize this in the theorem below.

\begin{theorem}
    Suppose the induced distribution over $\beta$s has density $f(\beta)$ and cumulative density $F(\beta)\,.$ Suppose we run Algorithm~\ref{alg:learn-alpha} with a fixed number of samples $n > 10$ and receive $\hat{\alpha}$ as the learned value of the parameter. If $\hat{\alpha} \ge1\,,$ we run $PTRR_1$ and achieve $1/\sqrt{k}$ fraction of the reward. Otherwise, 
    with probability $\left( \frac{1}{n+1} \right)^{1/n}\left( \frac{n}{n+1}\right)\,$ (which approaches 1 as $n \rightarrow \infty)$ over the test sample and the training samples, the reward of running $PTRR_{\hat{\alpha}}$ on the test sample is at least $1/k^{\hat{\alpha}/(\hat{\alpha}+1)}$ fraction of running $PTRR_{\alpha^\star}\,$ on that instance. That is:
    \begin{equation} \label{eqn:reward-ratio}
    \frac{Rew(PTRR_{\hat{\alpha}}, I_\text{test})}{Rew(PTRR_{\alpha^\star}, I_\text{test})} \ge \Omega \left(\frac{1}{k^{\hat{\alpha}/(\hat{\alpha}+1)}} \right)\,.
    \end{equation}
\end{theorem}


\begin{remark}
    Note that this is different from the PAC guarantee given by ERM in two important ways. First, we cannot choose arbitrarily small failure probability and error values and draw a number of samples that is a function of those. Secondly, we are providing a guarantee for a single fixed test instance drawn from the distribution, not for the ``average'' test instance. In particular, we are using the fact that PTRR with the best value of the parameter for the distribution cannot get more reward on the instance than the optimal reward achievable on that instance. 
\end{remark}

\begin{proof}

In order to prove this result, at a high-level, we will show that even if we resign ourselves to failure on a small fraction of the distribution (i.e., bad approximation ratio), we can still achieve a good approximation with high probability on the rest. To show this, we proceed in three steps:
\begin{enumerate}
    \item First, we define two events that are crucial to our guarantee: first, that the learned $\hat{\alpha}$ is sufficiently large; second, that the test example is sufficiently nice.
    \item We argue that for any test example for which the CEE is at most $\tau\,,$ running $PTRR_{\hat{\alpha}}$ obeys the reward ratio in Eqn.~\ref{eqn:reward-ratio} as long as $\hat{\alpha} \ge \tau$.
    \item We argue that with good probability over the training sample, $\hat{\alpha} \ge \tau\,,$ and with good probability over the test instance, $CEE(I_\text{test}) \le \tau\,.$ Thus, we can combine the probabilities of the relevant events to show that with the stated probability, the reward ratio guarantee holds.
\end{enumerate}

\paragraph{Relevant Events} At a high level, our strategy will be to resign ourselves to failure on a small fraction of the support of the induced distribution over $\beta$s in the interest of guaranteeing an approximation on the rest of the support. To that end, we define the following two events:
\begin{align}
    \mathcal{E}_1 &\coloneqq \{  \hat{\alpha} \ge \tau  \} \\
    \mathcal{E}_2 &\coloneqq \{  CEE(I_\text{test}) \le \tau \}
\end{align}

We know from the analysis in Section~\ref{sec:intro3} that for an instance with CEE $\beta$\,, running $PTRR_\alpha$ for any $\alpha \ge \beta$ will achieve $1/k^{\alpha/(\alpha+1)}$ fraction of the optimal reward. Thus, on a new test example, provided the value of the PTRR parameter we use is larger than the CEE of the instance, we will accrue sufficient reward. We formalize this subsequently.

\paragraph{Reward Ratio Holds}. Now, recall from Theorem~\ref{thm:ptrr-alpha} that for an instance with CEE $\beta_I\,,$ $PTRR_\alpha$ for any $\alpha > \beta_I$ accrues $O(OPT/k^{\alpha/(\alpha+1)})$ reward. Thus, if $\hat{\alpha} > \tau > \beta_\text{test} \coloneqq CEE(I_\text{test})\,,$ the $Rew(PTRR_{\hat{\alpha}}, I_\text{test}) \ge O(OPT/k^{\alpha/(\alpha+1)})\,.$ Finally, we observe that by the definition of $OPT\,,$ we have that $OPT \ge Rew(PTRR_\alpha, I_\text{test}) \, \forall \alpha\,,$ so in particular this holds for $\alpha = \alpha^\star\,,$ the value of $\alpha$ for which the expected competitive ratio is maximized. 
Thus, we have that if $\hat{\alpha} \ge \tau$ and $CEE(I_\text{test}) \le \tau\,,$ then:
$$
\frac{Rew(PTRR_{\hat{\alpha}}, I_\text{test})}{Rew(PTRR_{\alpha^\star}, I_\text{test})} \ge O \left(\frac{1}{k^{\hat{\alpha}/(\hat{\alpha}+1)}} \right)\,.
$$

Thus, it remains to argue that the conditions in the previous statement occur with good probability.

\paragraph{Probability}. First, let us consider event $\Pr{}{\mathcal{E}_2}\,.$ This is simply the cumulative density at $\beta=\tau\,, F(\tau)\,.$ Next, we analyze $\Pr{}{\mathcal{E}_1}:$
\begin{align}
    \Pr{}{\mathcal{E}_1} &= \Pr{}{\hat{\alpha} \ge \tau} \\
    &= \Pr{}{\max_{i}\hat{\beta}_i \ge \tau} \\
    \text{ under oracle } &= 1 - \Pr{}{\forall \, i \, \beta_i \le \tau} \\
    &= 1 - F(\tau)^n\,.
\end{align}
The probability of success is the probability of the intersection of these two events, so the probability of success is $F(\tau)(1-F(\tau)^n)\,.$ Now, it suffices to show that we can pick a $\tau$ that gives us a good probability of success. We can therefore pick the $\tau$ that maximizes the probability of success, namely $\arg \max_{\tau} F(\tau)(1-F(\tau)^n)\,.$
\begin{align}
    \max_{\tau} F(\tau)(1-F(\tau)^n) \Leftrightarrow f(\tau) - (n+1) F(\tau)^n f(\tau) &= 0 \\
    f(\tau)\left( 1 - (n+1) F(\tau)^n \right) &= 0 \\
    \frac{1}{n+1} &= F(\tau)^n \\
    \tau = F^{-1}\left(  \left(   \frac{1}{n+1} \right)^{1/n}  \right)\,.
\end{align}

Plugging this back into the probability of success, we get the stated result.

\end{proof}

\subsection{Full Empirical Evaluation}
\label{appendix:empirical}
The main phenomenon highlighted by the theory is that the parameter $\alpha$ controls the exploration--abandonment tradeoff, and therefore that the best $\alpha$ can depend on the underlying instance. The purpose of this empirical section is to provide evidence on real learning-curve data that different instances prefer different $\alpha$. We use sample-wise learning curves from the Learning Curve Database (LCDB~1.1), specifically the CC-18 benchmarks.
LCDB stores error-rate curves at multiple training-set sizes (``anchors'') for a fixed set of learners, together with repeated evaluations under a nested cross-validation protocol. 

Each dataset $d$ defines one improving-bandit instance.
Arms correspond to learners, and the time index corresponds to LCDB anchor points.
We convert error rates to rewards through
$
r_{i,d}(t) \;=\; 1 - \mathrm{err}_{i,d}(t),
$
so that larger rewards correspond to better performance.
LCDB stores multiple repetitions per (dataset, learner, anchor) due to nested cross-validation.
We average across these repetitions (ignoring missing values) to obtain a single mean learning curve per (dataset, learner).
This yields deterministic reward functions $r_{i,d}(\cdot)$, matching our setting.

LCDB curves can terminate early for some dataset--learner pairs (in the stored tensors this appears as NaNs).
For a fixed horizon $T$, we restrict to datasets for which all included learners have finite values for anchors $t=1,\dots,T$.
This ensures that each dataset corresponds to a well-defined bandit instance with a common horizon.

For each usable dataset $d$ and each $\alpha$ on a fixed grid, we run $\mathrm{PTRR}_\alpha$ for a horizon of $T$ pulls.
The only randomness in our implementation is the random ordering in which arms are first considered. We average results over 200 random seeds. We evaluate performance using the normalized cumulative reward
$$
\frac{\mathbb{E}[\mathrm{ALG}(d;\alpha)]}{\mathrm{OPT}(d)}
\qquad\text{where}\qquad
\mathrm{OPT}(d)=\max_i \sum_{t=1}^T r_{i,d}(t),
$$
which is the reciprocal of the competitive ratio $\mathrm{OPT}/\mathbb{E}[\mathrm{ALG}]$ from Definition~2.2. Note that $\mathrm{OPT}(d)$, the best fixed-arm policy on $d$ in hindsight, is not necessarily the global best policy here due to the absence of monotonicity.

\paragraph{Main experiment: $T=44$, $k=22$.}
Our primary experiment uses horizon $T=44$ and a learner set of size $k=22$ (we exclude two learners with $\geq 50 \%$ ill-behavior in the CC-18 file).
Under this choice, we obtain 27 datasets with complete prefixes of length $T$ across all $k$ learners.


\begin{figure}[H]
    \centering
    \begin{subfigure}[t]{0.49\linewidth}
        \centering
        \includegraphics[width=\linewidth]{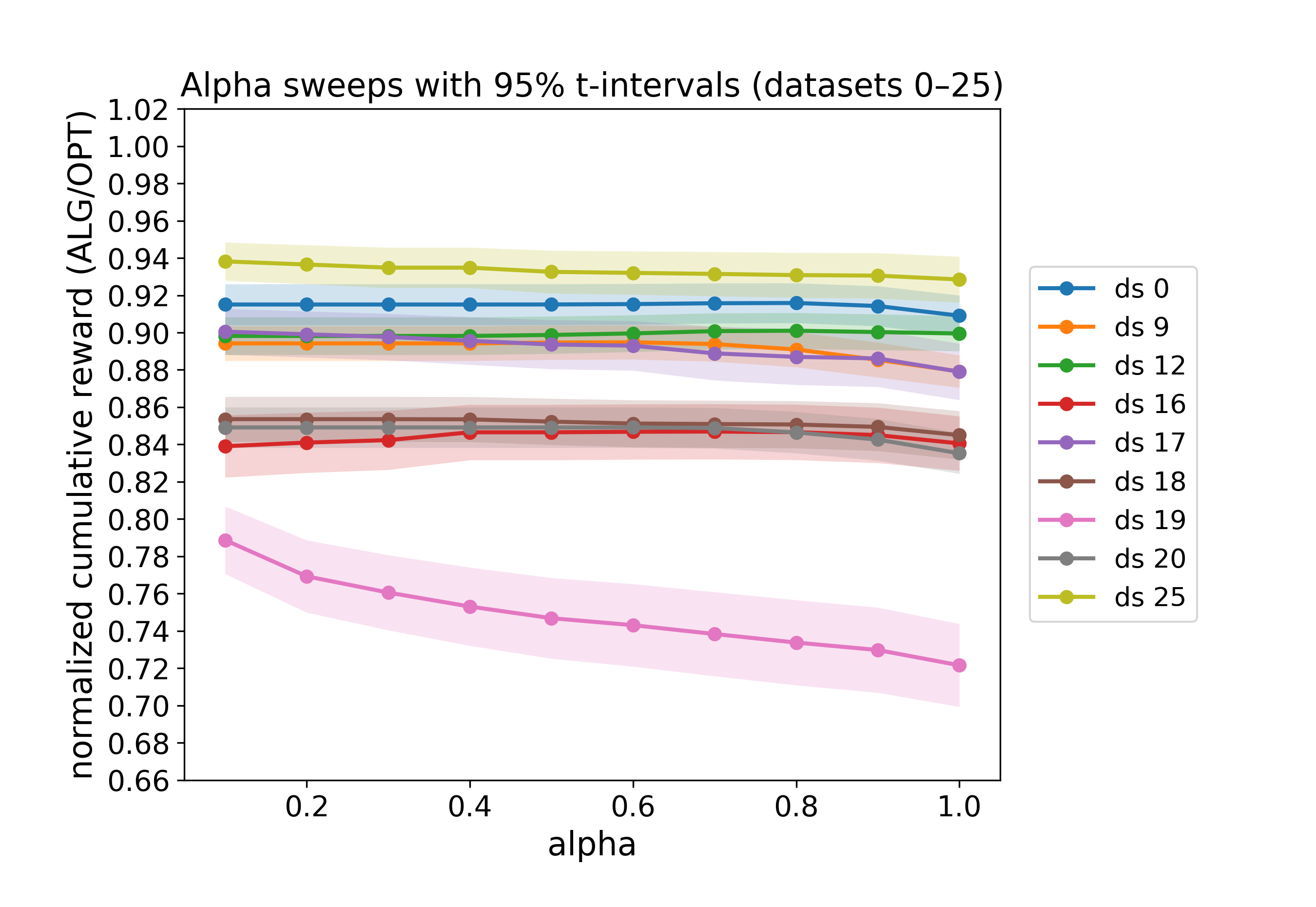}
        \label{fig:alpha_sweep_block1_ci_bigfont}
    \end{subfigure}\hfill
    \begin{subfigure}[t]{0.49\linewidth}
        \centering
        \includegraphics[width=\linewidth]{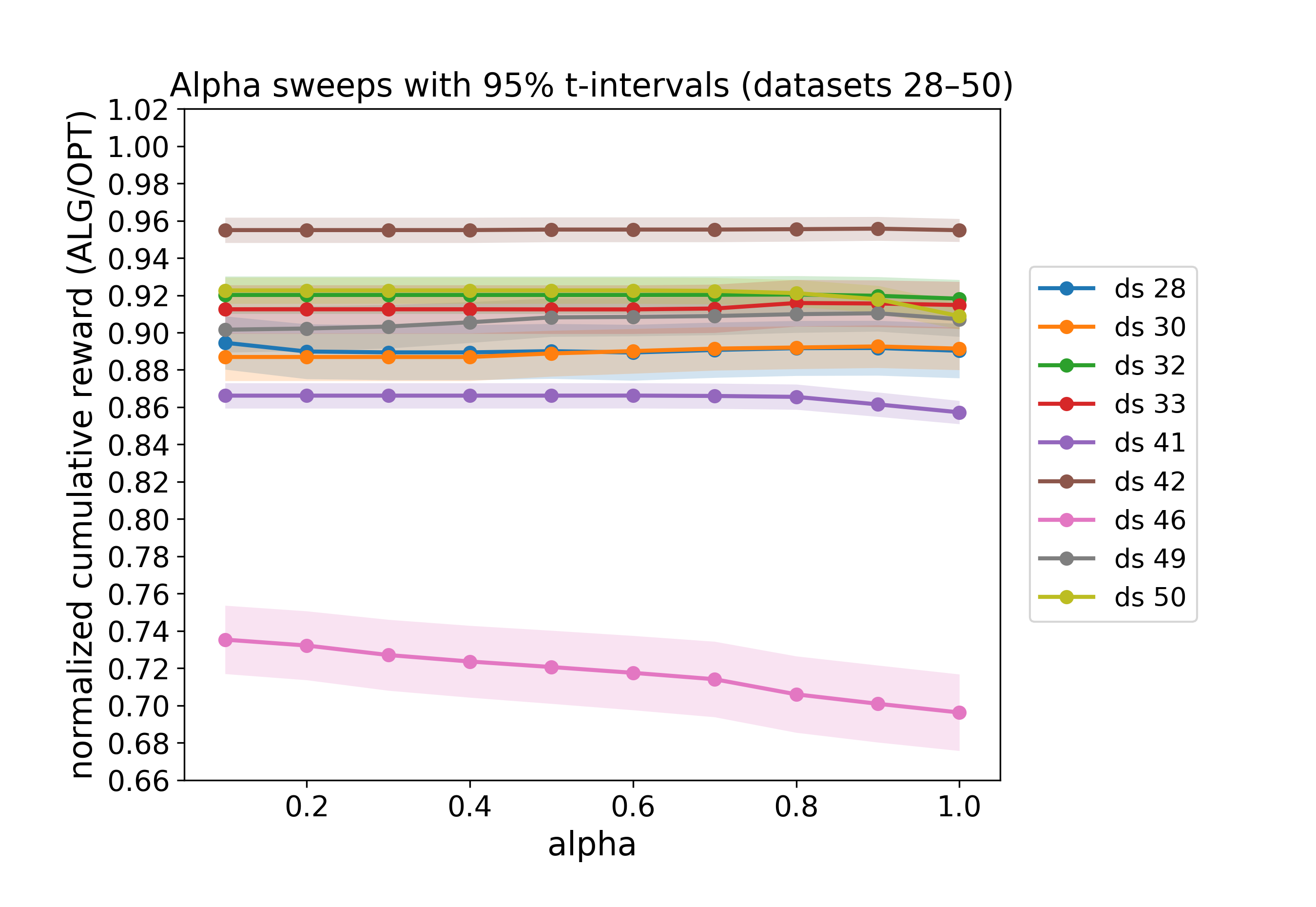}
        \label{fig:alpha_sweep_block2_ci_bigfont}
    \end{subfigure}

    \vspace{0.6em}

    \begin{subfigure}[t]{0.49\linewidth}
        \centering
        \includegraphics[width=\linewidth]{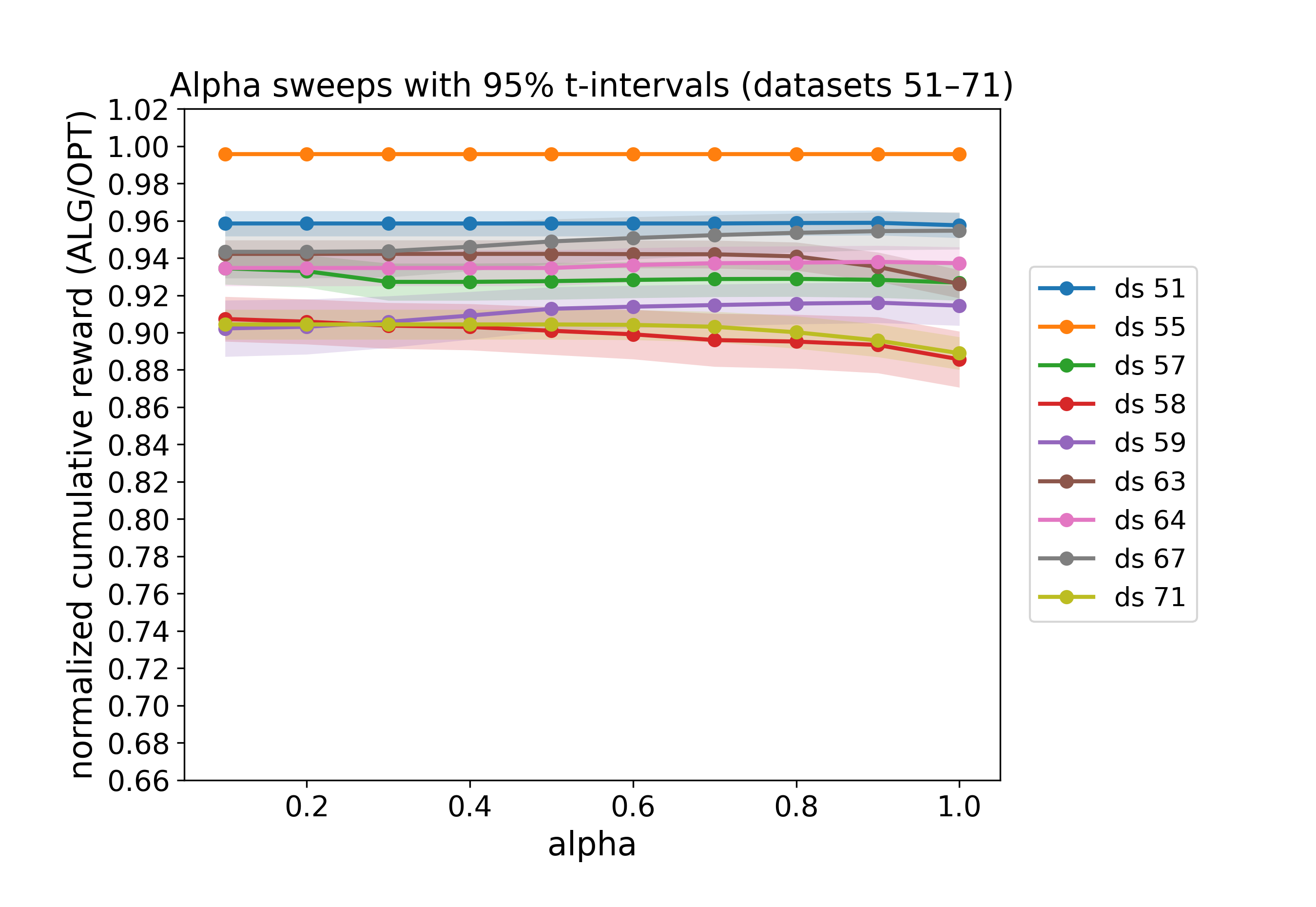}
        \label{fig:alpha_sweep_block3_ci_bigfont}
    \end{subfigure}

    \caption{\textbf{Sensitivity of $\mathrm{PTRR}_\alpha$ to $\alpha$ on all LCDB instances ($T=44$, $k=22$).} Each curve corresponds to one CC-18 dataset $d$ from LCDB~1.1 and reports the normalized cumulative reward $\mathbb{E}[\mathrm{PTRR}_\alpha(d)]/\mathrm{OPT}(d)$ as a function of $\alpha\in\{0.1,0.2,\dots,1.0\}$, where $\mathrm{OPT}(d)=\max_i\sum_{t=1}^T r_{i,d}(t)$ is the cumulative reward of the best fixed-arm policy in hindsight under the same horizon (which does not necessarily correspond to the global optimal policy due to the absence of monotonicity) 
and $r_{i,d}(t)=1-\mathrm{err}_{i,d}(t)$ is the mean (over cross-validation) reward at anchor $t$ for arm $i$. For each $(d,\alpha)$, $\mathbb{E}[\mathrm{PTRR}_\alpha(d)]$ is estimated by averaging over $200$ random arm orderings. The shaded regions are pointwise 95\% Student-$t$ confidence intervals across the 200 runs (mean $\pm\, t_{0.975,199}\cdot \mathrm{sd}/\sqrt{200}$). For most datasets, performance differences across $\alpha$ are small relative to the confidence intervals, while a minority show a significant trend across $\alpha$ on this grid.}
    \label{fig:alpha_sweep_all_blocks_ci_bigfont}
\end{figure}

We sweep $\alpha \in \{0.1,0.2,\dots,1.0\}$.
Figure \ref{fig:alpha_sweep_all_blocks_ci_bigfont} shows all of the per-dataset performance curves $\alpha \mapsto \mathbb{E}[\mathrm{ALG}]/\mathrm{OPT}$, demonstrating that the maximizer varies across certain instances. Across the 27 datasets, the best $\alpha$ is widely distributed: 11 datasets select $\alpha=0.1$, while many others select $\alpha$ in the range $[0.8,1.0]$ (it is worth noting that 4 of the 11 cases are actually flat across several small $\alpha$ values). Mechanistically, smaller $\alpha$ corresponds to more aggressive early abandonment in $\mathrm{PTRR}_\alpha$, so these datasets are those for which aggressive abandonment is (on this grid) empirically most favorable.

\begin{figure}[H]
    \centering
    \includegraphics[width=1\linewidth]{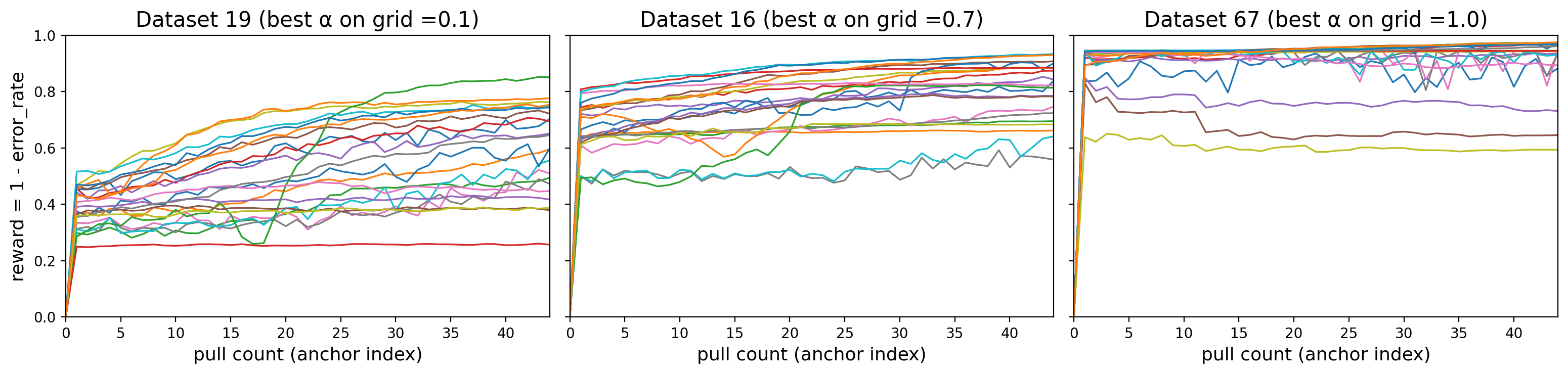}
    \caption{\textbf{Mean LCDB reward curves for three datasets with distinct best $\alpha$ values on the grid.} Each panel overlays the mean reward curves $r_{i,d}(t)=1-\mathrm{err}_{i,d}(t)$ across anchors $t$ for all $k=22$ arms on a single CC-18 dataset $d$. The title of each panel reports the value of $\alpha\in\{0.1,0.2,\dots,1.0\}$ that maximizes the estimated normalized cumulative reward $\mathbb{E}[\mathrm{PTRR}_\alpha(d)]/\mathrm{OPT}(d)$ at horizon $T=44$ on that dataset. 
    We selected these datasets to illustrate the reward dynamics underlying distinct best $\alpha$ values on the grid. \looseness-1
    \label{fig:learning curves appendix}}
\end{figure}

Naturally, real learning curves in LCDB are not guaranteed to satisfy the monotonicity/concavity assumptions used in this work (and largely don't). To connect $\alpha$ to concrete learning dynamics, we plot mean learning curves with fixed axes for three datasets with distinct best $\alpha$'s.
Figure \ref{fig:learning curves appendix} displays all learners for dataset 19 (which selects $\alpha=0.1$), dataset 16 (which selects $\alpha = 0.7)$ and dataset 67 (which selects $\alpha=1.0$). Qualitatively, these plots suggest a mechanism consistent with the influence of $\alpha$ on $\mathrm{PTRR}_\alpha$, where datasets preferring smaller $\alpha$ tend to exhibit early separation between `good' and `bad' arms (making aggressive abandonment beneficial), while datasets preferring larger $\alpha$ exhibit closer early performance among many learners (making aggressive abandonment riskier). We emphasize that this interpretation is qualitative and based directly on the plotted mean curves; the purpose of these plots is merely to contextualize the observed instance dependence.

\section{Best-of-both-worlds for Maximizing Cumulative Reward}

In this section, we consider the problem of getting the best policy regret possible if it is sublinear and reverting to optimal competitive ratio if it is not. In the main body, we considered a similar formulation for best-arm identification. Here, we instead hybridize between an algorithm known to get good policy regret on good instances from \cite{metelli_stochastic_2022} and an algorithm known to get optimal competitive ratio on worst-case instances from \cite{pmlr-v272-blum25a}. We take a data-driven approach to identifying exactly how to hybridize, i.e., exactly when to switch from the policy-regret-optimizing algorithm to the competitive-ratio-optimizing algorithm. We note that as a result of this approach, we do {\em not} get per-instance worst-case guarantees. This is an interesting direction for future work.

\subsection{Motivating Examples}
\cite{pmlr-v272-blum25a} give a randomized algorithm for improving bandits (also known as deterministic rested rising multi-armed bandits) that achieves the optimal $\sqrt{k}$ competitive ratio on worst-case instances. Although it is well known that a sublinear regret is not attainable for worst-case instances of improving bandits, under some niceness assumptions, UCB-style algorithms~\cite{metelli_stochastic_2022} can be shown to obtain sublinear policy regret (regret w.r.t.\ the best single arm pulled for the entire time horizon). Here we present examples that show that both the above algorithms from the literature may be sub-optimal and there are instances where each is dominated by the other.

\begin{example}
{\it The randomized robin algorithm of \cite{pmlr-v272-blum25a} may suffer linear policy regret on instances where the UCB-based algorithm of ~\cite{metelli_stochastic_2022} achieves sublinear regret.} We set the reward function for the best arm as $f_{i^*}(t)=1$ for all $t$, and for any other arm $i\ne i^*$ as $f_{i}(t)=t/T$, where $T$ is the time horizon. Now the randomized robin algorithm selects the optimal arm first with probability $1/k$, and otherwise, it keeps playing a sub-optimal arm. The expected total reward is $T\cdot\frac{1}{k}+\frac{T}{2}\cdot\frac{k-1}{k}=\frac{T}{2}\left(1+\frac{1}{k}\right)$, which corresponds to  $\Omega(T)$ regret. On the other hand, the UCB-based algorithm has a logarithmic upper bound on its regret. Indeed, after $O(\log T)$ exploratory pulls of any sub-optimal arm, it will always prefer the optimal arm, implying an upper bound of $O(k\log T)$ on the policy regret.
\end{example}

\begin{example}
{\it The UCB-based algorithm of ~\cite{metelli_stochastic_2022} may have $\Omega(k)$ competitive ratio on some instances.} We use the example used in the lower bound construction of \cite{pmlr-v272-blum25a}. Set $f_{i^*}(t)=t/T$ for all $t$ for the optimal arm $i^*$, and for any other arm $i\ne i^*$ as $f_{i}(t)=\min\left\{t/T,\frac{1}{\sqrt{k}}\right\}$. Due to the exploration term, while the arm rewards are identical, each arm gets pulled an equal number of times. By time $T$, each arm gets pulled $T/k$ times for a total reward of $T/2k$, sub-optimal by a factor of $k$.
\end{example}

\subsection{Data-driven Hybrid Approach}

Our goal, therefore, is to take a step toward devising algorithms that can achieve sublinear regret for nice instances but fall back to optimal competitive ratio for general instances.
To this end, we describe an algorithm that interpolates between a UCB-based algorithm by \cite{metelli_stochastic_2022} that gets sublinear regret on nice instances and the algorithm of \cite{pmlr-v272-blum25a} which achieves the optimal worst-case competitive ratio.

\begin{algorithm}
\caption{Regret-optimizing Hybrid Algorithm, i.e. \texttt{Regret-Hybrid}$_B$} \label{alg:reg-hybr}
\begin{algorithmic}
    \STATE Run Algorithm 1 (R-ed-UCB) of \cite{metelli_stochastic_2022} for B time steps 
    \STATE Run PTRR for T-B time steps
\end{algorithmic}
\end{algorithm}

We define a family of algorithms parameterized by $B$:

\begin{definition} \label{defn:regret-hybrid}
    Define the family of algorithms \texttt{Regret-Hybrid} $\coloneqq \{ \texttt{Regret-Hybrid}_B : B \in [T] \} \,, $ where $\texttt{Regret-Hybrid}_B$ is Algorithm~\ref{alg:reg-hybr}.
\end{definition}

Now, as before, we analyze $Q_D$ and then instantiate the result for a loss function of interest.

\begin{lemma}
    For the family $\texttt{Regret-Hybrid}$ defined in Defn.~\ref{defn:regret-hybrid}, the improving multi-armed bandits problem, and any piecewise constant loss function, $Q_D \le k\,T^2\,.$
\end{lemma}
\begin{proof}
    $B$ takes on $T$ discrete values, and for each fixed $B\,,$ by Lemma~\ref{lemma:bdqd}, there are at most $k\,T$ behaviors for the PTRR family on an instance.
\end{proof}

In order to study cumulative reward via regret, we can use the average regret loss function defined in Defn.~\ref{defn:avgregret}. Then, extending Theorem~\ref{thm:hybrid-alg-sc}, we get the following corollary:

\begin{corollary}
        For the Hyperparameter Transfer setting for the improving multi-armed bandits problem optimizing for averaged regret over the algorithm family described in Defn.~\ref{def:fullyhybridfam}, $N = O\left( \left(\frac{m}{\epsilon}\right)^2 (\log kT + \log \frac 1 \delta )  \right)$ instances drawn from $\mathcal{D}$ suffice to get the uniform convergence guarantee in Theorem~\ref{thm:ss25main}.
\end{corollary}

Note that this result provides an algorithm that is competitive with the best possible algorithm in the family on instances from distribution $\mathcal{D}.$ Thus, that algorithm may not provide the worst-case fallback guarantee on a {\em fixed} instance. As mentioned earlier, this is an interesting consideration for future work.

\section{BAI Comparison with Prior Work and Proof Details}\label{appendix:PR}

We provide below a comparison with prior work on best arm identification along with complete proof details for Section \ref{sec:bothworlds}.

\subsection{Comparison With Prior Work}

First, let us compare the BAI task in \cite{mussi2024best} to ours. In their work, they study the task of identifying the arm with the greatest single pull in the instance (this equals the pull of that arm at time $T$ due to reward monotonicity). In our work, we study the task of identifying the arm with the greatest cumulative reward in the instance. While these tasks are slightly different, we note that both are important and relevant tasks, and the best arm for one task will be a 2-approximation (that is, not be more than a factor of two suboptimal) for the other task. 
On the one hand, consider a setting in which each arm is a technology, and investing in developing the technology increases its utility. If we are interested in identifying one out of many technologies that will have the highest utility after the investment period, we are interested in identifying the arm with the highest single pull, corresponding to \cite{mussi2024best}'s setting. On the other hand, consider a setting in which each arm is an advertisement, and the reward associated with a pull is the amount of money a user spends as a result of that advertisement. A corporation would care to identify the arm that has the highest cumulative reward and then use that arm for future users. Thus, both are interesting and valid goals albeit slightly different from each other. \par

Having defined the tasks, we will now see more clearly that the difference in the ``niceness'' conditions relates to this, as well. We consider the deterministic variant of the setting which is the focus of this work (that is, set the stochasticity in the instances to zero). Now, we study the conditions and terms in Theorem 4.1 of \cite{mussi2024best}. Note that setting the stochasticity to zero implies $\epsilon, \sigma = 0\,.$ Then, $y$ is no longer a function of $a\,,$ and thus Eqn. 4 is satisfied for {\em all} $a\,,$ so we can achieve BAI with probability 1 if the conditions are met. This is the same guarantee we achieve, and so we are ready to compare the conditions more carefully. \par

In \cite{mussi2024best}, they study $y_i\,,$ the number of pulls of arm $i$ until the first difference is smaller than the average gap between the best single pull over all arms and the best single pull of this arm. This quantity corresponds to our quantity $h_i(\epsilon)\,,$ which is the number of pulls of arm $i$ after which the amount of cumulative reward we could achieve if we continued growing with the same slope is bounded by $\epsilon.$ Thus, both $y_i$ and $h_i$ study how quickly arms ``converge'' to their final value, $y_i$ directly in the value (of individual pull) sense and $h_i$ in the cumulative sense. \par

Finally, we note that \cite{mussi2024best} compares the sum of these ``convergence times'' to $(1-\iota)T\,,$ whereas we compare to a parameter $B\,.$ In both cases, the smaller the comparison value, the more value we can extract from the best arm in the remaining $\iota T$ or $T-B$ time. Thus, the overall flavor of the conditions in both works is the same, namely that if arms ``converge quickly,'' then BAI is possible. The exact notions of the words in quotes are what differ.

\subsection{Difficulty in Corralling Improving Bandit Algorithms}\label{appendix:corralling}

A natural approach towards obtaining best-of-both-worlds guarantees in the above examples would be to design a meta-algorithm that can potentially switch between \cite{pmlr-v272-blum25a} and \cite{mussi2024best}. We examine whether it is possible in the improving bandits setting to perform corralling~\cite{agarwal2017corralling,arora2021corralling,luo2022corralling}, a recently developed paradigm for meta-learning bandit algorithms in a fully online setting. Our main result here is an impossibility result which rules out the ability to use corralling to achieve best-of-both-worlds in a fully online improving bandits setting. We show that in the improving bandits setting, no matter what the meta-algorithm is, it is possible to suffer linear regret relative to the best base algorithm.

In the meta-learning setting, we have a collection of $M$ {\it base} algorithms $\mathcal{B}=\{B_1,\dots,B_M\}$ for the IMAB problem and the goal is to (nearly) recover the guarantees of the best algorithm on any given instance. Each base algorithm can give its own prediction for which arm to play next, given the history of arm pulls and rewards so far and the  meta-learner can decide which arm to actually pull based on these predictions. For best arm identification, we define the sub-optimality ratio of the meta-algorithm as

$$\overline{R}_T(\mathcal{M},\mu,\mathcal{B}):=\left[\max_{j}R_T^{B_j}(\mu)\right]/R_T^{\mathcal{M}}(\mu),$$
\noindent where $\mu$ denotes the IMAB instance, $R_T^{B_j}$ denotes the cumulative reward of the best arm selected by base algorithm $B_j$, and $R_T^{\mathcal{M}}$ denotes the cumulative reward of the best arm selected by the corralling meta-algorithm $\mathcal{M}$.

\begin{theorem}[Lower Bound for Best Arm Identification in Corralling Improving Bandits]
Consider deterministic multi-arm rested bandits with rising concave reward sequences. Let $\mathcal{B}=\{B_1,\dots,B_M\}$ be any finite class of (possibly randomized) base algorithms. A meta-algorithm $\mathcal M$ selects at each round one base algorithm whose recommended arm is executed.
Then 
\[
\inf_{\mathcal M}
\sup_{\mu,\mathcal{B}}
\overline{R}_T(\mathcal{M},\mu,\mathcal{B})
\;\ge\;
2.
\]
\label{theorem:corralling-bai}
\end{theorem}

\begin{proof}
Let $k=3$. Fix horizon $T$ and let $L=T/2$. Choose $\Delta>0$ sufficiently small. Construct three deterministic concave rising environments $E^j$ as follows.

For $s \le L$, define
\[
\mu_1(s)=\mu_2(s)=\mu_3(s)=s\Delta.
\]
That is, the first $L$ pulls of each arm yield identical rewards in all environments.

For $s>L$, define:

\medskip
\noindent
\textbf{Environment $E^j$:}
\[
\mu_j(s)=L\Delta+(s-L)\Delta,
\qquad
\mu_{i}(s)=L\Delta, \text{ for }i\ne j.
\]

\medskip
All reward sequences are nondecreasing and concave.
Observe that the  environments are identical until some arm is pulled more than $L$ times. In $E^j$, arm $j$ is uniquely optimal after the divergence. Set the base algorithm class $\mathcal{B}=\{B_1,B_2,B_3\}$, where $B_j$ always pulls arm $j$ and outputs it as the best arm.

Consider any meta-algorithm $\mathcal M$. Note that the meta-algorithm can pull at most one base algorithm more than $L$ times, say $B_1$ (WLOG). Since rewards are identical for the first $L$ pulls of each arm, $\mathcal M$ cannot distinguish between environments $E^2$ and $E^3$. So, no matter what best arm is selected by $\mathcal M$, we can always select an environment where it picks a sub-optimal arm. In any environment the optimal arm gets cumulative reward twice as much as any sub-optimal arm. 
Therefore,
\[
\sup_{\mu\in\{E^j\}}
\overline{R}_T(\mathcal{M},\mu,\mathcal{B})
\ge
2.
\]
\end{proof}

\noindent For cumulative reward maximization, we seek to bound the meta-regret with respect to the best algorithm among the base algorithms


$$\tilde{R}_T(\mathcal{M},\mu,\mathcal{B}):=\max_{j}R_T^{B_j}(\mu) - R_T^{\mathcal{M}}(\mu),$$

\noindent where $\mu$ denotes the IMAB instance, $R_T^{B_j}$ denotes the cumulative reward if algorithm $B_j$ is used exclusively to decide the arm pulls for $T$ rounds, and $R_T^{\mathcal{M}}$ denotes the cumulative reward collected by the corralling meta-algorithm $\mathcal{M}$.

\begin{theorem}[Minimax Impossibility of Corralling Improving Bandits]
Consider deterministic two-arm rested bandits with rising concave reward sequences and horizon $T$. Let $\mathcal{B}=\{B_1,\dots,B_M\}$ be any finite class of (possibly randomized) base algorithms. A meta-algorithm $\mathcal M$ selects at each round one base algorithm whose recommended arm is executed.
Then there exists a universal constant $c>0$ such that\looseness-1
\[
\inf_{\mathcal M}
\sup_{\mu,\mathcal{B}}
\tilde{R}_T(\mathcal{M},\mu,\mathcal{B})
\;\ge\;
cT.
\]
\label{theorem:corralling}
\end{theorem}

\begin{proof}[Proof Sketch]
    We construct two environments that are identical for the first $T/2$ pulls of each arm but diverge thereafter, with opposite optimal arms. Any meta-algorithm attempting to compete with the best run-alone base must effectively determine which environment it faces before the divergence point. However, since the two instances are indistinguishable during the prefix, identifying the correct environment requires linear exploration. Consequently, linear meta-regret is unavoidable on at least one of the two instances. A full proof is located in Appendix \ref{appendix:corralling}.
\end{proof}

\noindent This theorem establishes a  minimax barrier for corralling in improving bandits.  The obstruction is information-theoretic: the two environments are indistinguishable for a linear prefix yet demand opposite commitments thereafter. This motivates the need to consider alternative approaches towards achieving best-of-both-worlds guarantees.

We provide below a complete proof for Theorem \ref{theorem:corralling}.

\begin{proof}
Fix horizon $T$ and let $L=T/2$. Choose $\Delta>0$ sufficiently small. Construct two deterministic concave rising environments $E^+$ and $E^-$ as follows.

For $s \le L$, define
\[
\mu_1(s)=\mu_2(s)=s\Delta.
\]
That is, the first $L$ pulls of either arm yield identical rewards in both environments.

For $s>L$, define:

\medskip
\noindent
\textbf{Environment $E^+$:}
\[
\mu_1(s)=L\Delta,
\qquad
\mu_2(s)=L\Delta+(s-L)\Delta.
\]

\medskip
\noindent
\textbf{Environment $E^-$:}
\[
\mu_2(s)=L\Delta,
\qquad
\mu_1(s)=L\Delta+(s-L)\Delta.
\]

\medskip
Both reward sequences are nondecreasing and concave.
Observe that the two environments are identical until some arm is pulled more than $L$ times. In $E^+$, arm $2$ is uniquely optimal after the divergence; in $E^-$, arm $1$ is uniquely optimal. Set the base algorithm class $\mathcal{B}=\{B_1,B_2\}$, where $B_j$ always pulls arm $j$.

Consider any meta-algorithm $\mathcal M$. Since rewards are identical for the first $L$ pulls of each arm, the distribution over the first $L$ actions of $\mathcal M$ is identical under $E^+$ and $E^-$. Let $N_1$ be the expected number of pulls of arm $1$ by round $L$.
Without loss of generality, suppose $N_1 \le L/2$ (otherwise swap the roles of the arms).

Consider environment $E^-$, where arm $1$ becomes uniquely optimal after the divergence. Let $j^*$ be an index maximizing $R_T^{B_j}(E^-)$. Base algorithm $B_1$ must obtain $\Theta(T\Delta)$ additional reward by exploiting arm $1$ after round $L$.

Because $\mathcal M$ allocates at most $L/2$ pulls to arm $1$ during the prefix, it under-invests in the arm that later becomes uniquely optimal. Due to concavity, the cumulative post-divergence reward is linear in the number of additional pulls. Thus, there exists a constant $c>0$ such that
\[
R_T^{B_{1}}(E^-)
-
R_T^{\mathcal M}(E^-)
\ge
cT.
\]

If instead $N_1>L/2$, the same argument applied to $E^+$ yields an identical bound. Therefore,
\[
\sup_{\mu\in\{E^+,E^-\}}
\tilde{R}_T(\mathcal{M},\mu,\mathcal{B})
\ge
cT.
\]
\end{proof}

\subsection{Proofs for results in Section~\ref{sec:hybridguarantees}} \label{appendix:hybridproofs}
We prove the lemma used in  Section~\ref{sec:hybridguarantees} to justify the quantities $L_i, U_i$ used in Algorithm~\ref{alg:hybrid}.

\begin{lemma}\label{lem:envelope} 
For every arm $i$ and every $t \leq T$, we have $L_i(t) \leq f_i(T) \leq U_i(t)$.
\end{lemma}
\begin{proof}
By concavity, we know that $\gamma_i(s)$ is non-increasing, and therefore that for any $x \ge t$, we have
$$
f_i(x) - f_i(t) = \sum_{s=t}^{x-1} \left(f_i(s+1) - f_i(s)\right) \le \sum_{s=t}^{x-1} \gamma_i(t) = (x - t)\gamma_i(t).
$$
Setting $x = T$ gives $f_i(T) \le f_i(t) + (T - t)\gamma_i(t) = U_i(t).$ By monotonicity, we likewise know that $L_i(t) = f_i(t)\leq F_i(T)$.  
\end{proof}

We include below a complete proof for Theorem \ref{thm:hybrid-single2}.

\noindent{\it Proof (of Theorem \ref{thm:hybrid-single2}.)}

\begin{proof}[Proof of 1. (certificate correctness)]
Suppose an instance $I$ satisfies \(\mathrm{GCC}(\theta)\). Since $\triangle_i(t)$ is non-increasing for all $i$ and
$\sum_{i=1}^k h_i(\Delta_I/3) \le \theta$, we know that there are at most $\theta$ total pulls (across arms) such that $\triangle_i(t_i) > \Delta_I/3$. Note that we can never have $L_{i}(t_i)  > \max_{j \ne i} U_j(t_j)$ for some $i \neq i^*$, as we know by concavity and monotonicity that $L_i(t) \le f_i(T) \le U_i(t)$, so this would imply $f_{i^*} (T) \leq U_{i^*}(t_{i^*}) < L_i (t_i) \leq f_{i}(T)$ (a contradiction). Since $B>\theta$ and Stage $1$ always pulls the arm $i$ that maximizes $(U_i - L_i)$, it follows that the algorithm will reach some point (namely $t = \theta$) where $\triangle_i(t_i) \leq \Delta_I/3$ for all $i$. At this point, for each $j \ne i^\star$, we have
$$
U_j(t_j) \le f_j(T) + \frac{\Delta_I}{3} \le f^*(T) - \Delta_I + \frac{\Delta_I}{3} = f^*(T)  - \frac{2\Delta_I}{3}.
$$
Moreover, $i^*$ satisfies
$$
L_{i^\star}(t_{i^*})  \ge f^*(T) - \frac{\Delta_I}{3}.
$$ It follows that $L_{i^\star}(t_{i^*})  > \max_{j \ne i^\star} U_j(t_j)$, and therefore that the algorithm returns $i^*$ in Stage $1$.
\end{proof}

\begin{proof}[Proof of 2. (approximation fallback)]
Write $g^\star(h):=\max_i g_i(h)=\max_i f_i(t_i+h)$.
Let $T_\text{rem} := T-B$, and let $\tau' = T_\text{rem}-k$. Suppose $m' := (\tau'/T)f^{\star}(T)$. Since 
$g^{\star}(T) = \max_i f_i(T) \ge f^{\star}(T)$, we know that
$$
m' = \frac{\tau'}{T} f^{\star}(T) \le f^{\star}(T) \left( \frac{\tau'}{T} \right)^{\alpha} \le g^{\star}(T) \left( \frac{\tau'}{T} \right)^{\alpha}.
$$
Using monotonicity and the fact that $g^{\star}(T) \le 2 f^{\star}(T)$, we also know that $g^{\star}(\tau') \le 2 f^{\star}(T) = 2(T/\tau') m'$, and therefore that \ $m' \ge (\tau'/(2T)) g^{\star}(\tau')$. Since 
$B \le T/2$ and $T \ge 4k$, it follows that
$$
\frac{1}{8} g^{\star}(\tau') \le m' \le g^{\star}(T) \left( \frac{\tau'}{T} \right)^{\alpha}.
$$
Run \textit{PTRR}\(_\alpha\) for $T_\text{rem}$ steps with parameters $m'$ and $\tau'$. By Theorem \ref{thm:ptrr-alpha}, we know that
$$
\mathbb{E}\big[\mathrm{ALG}_{T_\text{rem}} \big]\ \ge\ 
\frac{\mathrm{OPT}^{\mathrm{res}}_{T_\text{rem}}}{C_\alpha c_2(k+1)^{\alpha/(1+\alpha)}},
$$
where $\mathrm{OPT}^{\mathrm{res}}_{T_{\mathrm{rem}}}$ denotes the optimal cumulative reward for $\{g_i\}$ over $T_{\mathrm{rem}}$ rounds and $C_\alpha = 2^{\alpha+2}(\alpha+1)$.

Now note that $\mathrm{OPT}^{\mathrm{res}}_{T_{\mathrm{rem}}}\ge \tfrac12g^\star(T_{\mathrm{rem}})T_{\mathrm{rem}}$, and that the algorithm’s maximum single‑pull reward dominates its average reward (Facts B.1 and B.3 in the appendix of \cite{pmlr-v272-blum25a}). Let $\hat i$ denote the arm that achieves this maximum single-pull reward, and note that

$$
\mathbb{E}\Big[\max_{t\le T_{\mathrm{rem}}}\text{reward}_t\Big]\ \ge\ \frac{1}{T_{\mathrm{rem}}}\mathbb{E}\big[\mathrm{ALG}_{T_{\mathrm{rem}}}\big]\ \ge\ \frac{g^\star(T_{\mathrm{rem}})}{2C_\alpha c_2(k+1)^{\alpha/(1+\alpha)}}.
$$

Since $f_{\hat i}(T)\ge \max_{t\le T_{\mathrm{rem}}}\text{reward}_t$ by monotonicity, we know that taking expectations gives
$$
\mathbb{E}\big[f_{\hat i}(T)\big]\ \ge\ \frac{g^\star(T_{\mathrm{rem}})}{2C_\alpha c_2(k+1)^{\alpha/(1+\alpha)}}.
$$
Monotonicity and concavity give $g^\star(T_{\mathrm{rem}})\ge (T_{\mathrm{rem}}/T)f^\star(T)\ge \tfrac12 f^\star(T)$, as $T_{\mathrm{rem}}\ge T/2$. Combining with $c_2\le 8$ from Step~1, it follows that
$$
\mathbb{E}\big[f_{\hat i}(T)\big]\ \ge\ \frac{1}{2C_\alpha c_2(k+1)^{\alpha/(1+\alpha)}}\cdot \frac{1}{2}f^\star(T)\ \ge\ \frac{1}{2^{\alpha+7}(\alpha+1)}(k+1)^{-\alpha/(1+\alpha)}f^\star(T),
$$
as desired.
\end{proof}

\subsection{Proof of Lemma \ref{lemma:bdqd-hybrid}}

\begin{proof}
    We argue this by applying the proof of Lemma~\ref{lemma:bdqd}. For a fixed (augmented) instance (i.e., fixed instance and fixed random permutation), we are interested in computing the number of different behaviors as we vary $B, \alpha\,.$ To understand this, let us start by fixing $B$. Then, by Lemma~\ref{lemma:bdqd}, we know that there are at most $kT$ possible behaviors as we vary $\alpha\,.$ Now, for a fixed value of $B,$ for either algorithm, the sequence of pulls is determined, which exactly determines the loss in Stage 1. Thus, each value of $B$ corresponds to at most 1 new value of the loss. This implies that there are at most $kT$ possible behaviors in both stages for a fixed value of $B.$ Since there are at most $T$ values of $B,$ we have that there are at most $kT^2$ possible behaviors.
\end{proof}

\subsection{Cumulative Reward Best Arm Identification}
In the below sections, we provide comparable BAI guarantees to Section~\ref{sec:hybridguarantees} for cumulative reward instead of maximum reward. As before, we work in the standard improving‑bandits setting with $k$ arms, a known horizon $T$, and non‑decreasing concave reward functions $f_i$. Each algorithm Hybrid$_{\alpha,B}$ has two stages. Stage 1 uses a UCB-style envelope: At each step, the algorithm computes a lower bound $L_i(n)$ and terminal upper bound $U_i(n)$ on the final accumulated reward of every arm and pulls the arm with the largest optimistic estimate $U_i$. If the lower bound of one arm dominates the terminal upper bound of every other arm, Hybrid$_{\alpha,B}$ commits to this arm. 
If no commit occurs by time $B$, Stage 2 runs PTRR$_\alpha$ and finds an arm whose expected terminal reward is at least a substantial fraction of the best arm’s.\looseness-1 \par

For the terminal envelope, we define $L_i(t) := F_i(t) + (T - t) f_i(t),
\triangle_i(t) := \frac{(T - t)(T - t + 1)}{2} \, \gamma_i(t - 1),$ and $
U_i(t) := L_i(t) + \triangle_i(t), 
$
where $F_i(T) \coloneqq \sum_{t = 1}^T f_i(t)$ and $\gamma_i (t-1) := f_i(t) - f_i (t-1)$. We set $U_i(0) : = \infty$ to ensure first pulls. Using concavity and monotonicity, it is again straightforward to prove that $L_i(t) \leq F_i(T) \leq U_i(t)$ for all $i,t$.

\begin{algorithm}[h]
\caption{$\textit{Cumulative Hybrid}_{\alpha, B}$}\label{alg:cumhybrid}
\begin{algorithmic}[1]
\REQUIRE $m$
\STATE \textbf{Stage 1}: $t\gets0$ 
\FOR{each arm $i$}
\STATE $t_i\gets0$, $F_i\gets0$, $L_i\gets0$, $U_i\gets+\infty$
\ENDFOR
\WHILE{$t<B$}
  \FOR{each $i$ with $t_i\ge1$}
     \STATE $L_i \gets F_i + (T-t_i)\,f_i(t_i)$ \hfill 
     \STATE $\gamma_i \gets f_i(t_i)-f_i(t_i-1)$ \hfill 
     \STATE $U_i \gets L_i + \frac{(T-t_i)(T-t_i+1)}{2}\,\cdot \gamma_i$ 
  \ENDFOR
  \STATE $\hat i \gets \arg\max_i L_i$,\quad $U_{\mathrm{next}}\gets \max_{j\ne \hat i} U_j$
  \IF{$L_{\hat i} > U_{\mathrm{next}}$} 
     \RETURN $\hat i$.
  \ENDIF
  \STATE $i' \gets \arg\max_i (U_i - L_i)$
  \STATE \textbf{pull} $i'$;\; $t_{i'}\gets t_{i'}+1$, $t\gets t+1$, $F_{i'}\gets F_{i'}+f_{i'}(t_{i'})$
\ENDWHILE
\STATE \textbf{Stage 2}:
 $\tau' \gets (T - B) - k$, $m' \gets \left(\frac{\tau'}{T}\right) \cdot m$
\FOR{each $i$}
\STATE $g_i(s)\gets f_i(t_i + s)$
\ENDFOR
\RETURN $\hat i \gets \text{arm returned by } \textit{PTRR}_\alpha$ with parameters $(m', \tau')$ on $\{g_i\}$ for $T-B$ steps.

\end{algorithmic}
\end{algorithm}

\subsection{
Best-of-Both-Worlds Best arm identification guarantees}
In this section, we show that \textit{Cumulative Hybrid} contains algorithms that simultaneously (i) guarantee best arm identification on sufficiently benign instances, and (ii) preserve tight (up to constants) multiplicative bounds for approximating the best arm
on adversarial instances.\footnote{With additional information about stronger concavity, we achieve sharper bounds by using the corresponding version of PTRR.}

We start by defining a class of `sufficiently benign' instances and prove that our algorithm is guaranteed to return the best arm on all members of this class. If an instance is not in this class, Stage~2 pursues {\em approximate} BAI and runs $PTRR_\alpha$ with $\alpha$ dependent on the strength of concavity ($\alpha = 1$ works for {\em all} instances). Having reverted to an approximate goal, we identify an arm $\hat{i}$ whose final reward satisfies
$\mathbb{E}[f_{\hat{i}}(T)] \ge \Omega\left(k^{\frac{-\alpha}{1+\alpha}}\right) f^{\star}(T)$.  

As before, we assume for simplicity that both $T$ and $f^*(T)$ are known to the algorithm, which we use to set $\tau' = (T-B)-k$ and $m' = \frac{\tau'}{T} \cdot f^*(T)$. As in Section~\ref{sec:sharper CR}, these assumptions can be removed with only $O(\log k)$ overhead (see also Appendix \ref{appendix:unknownT}). 

We will now define a condition under which we can guarantee best arm identification.

\begin{definition}[Per-arm terminal budget, $h_i$]
For any arm $i$ and $\epsilon>0$, the terminal budget $h_i$ of arm $i$ is defined as
$
h_i(\epsilon) := \min\{\,n \in \{2,\ldots,T\} : \triangle_i(n) \le \epsilon\,\},
$
where $\triangle_i(t) := \frac{(T - t)(T - t + 1)}{2} \, \gamma_i(t - 1)$. 
\end{definition}

\begin{definition}[Best Arm Gap, $\Delta_I$]
For any instance $I$, its Best Arm Gap $\Delta_I$ is defined as
 $\Delta_I := \text{OPT}_T - \max_{j \ne i^*} F_j(T).$
\end{definition}

\begin{definition}[Gap Clearance Condition, $GCC(B)$]\label{def:GCC}
For any $B\le T$, we say that an instance satisfies the Gap Clearance Condition  $GCC(B)$ if $\Delta_I > 0$ and
$
\sum_{i=1}^K h_i(\Delta_I / 3) \;\le\; B.
$
\end{definition}

Under concavity, $\triangle_i(n)$ denotes the worst-case remaining terminal mass for arm $i$ after $n$ pulls: it is the most reward an adversary can ``hide in the tail'' by continuing with the last observed slope. 
The per-arm budget $h_i(\epsilon)$ is therefore the minimal number of pulls needed to ensure its optimistic terminal value is close to current lower envelope. 
Taking $\epsilon = \Delta_I/3$, once a suboptimal arm has been pulled this many times, its best possible continuation still loses to the best arm’s lower envelope. 
Thus $\mathrm{GCC}(B)$ states that the total work needed to certify the best arm fits within the mid-horizon budget $B$. 
If the sum exceeds $B$, concavity allows at least one suboptimal arm to remain plausibly optimal by time $B$, so no sound mid-horizon certificate can be guaranteed.

\begin{theorem}[best-of-both-worlds guarantees] \label{thm:hybrid-single}
Suppose an instance $I\in\mathcal{I}$ has Concavity Envelope Exponent $\beta_I\in(0,1]$. Algorithm \ref{alg:cumhybrid} with $\alpha \in(\beta_I, 1]$  satisfies the following properties:
\begin{enumerate}
[leftmargin=*,topsep=0pt,partopsep=1ex,parsep=1ex]\itemsep=-4pt
\item If the instance further satisfies $\Delta_I>0$, and $GCC(\theta)$ holds for $\theta\le T/2$, then the algorithm with $B\in (\theta,T/2]$  identifies and commits to the best arm $i^\star$ in Stage~1.
    \item If Stage~1 does not certify a best arm by time $B$, then Stage~2 finds an approximate best arm $\hat i$ such that 
    $\mathbb{E}\big[F_{\hat i}(T)\big]\ \ge\ \Omega{\left(k^{\frac{-\alpha}{1+\alpha}}\right)}\; F^\star(T),$ where the expectation is  over the randomness of the algorithm. 
\end{enumerate}
\end{theorem}

\begin{proof}[Proof Sketch]
We argue (1) and (2) separately. (1) $\mathrm{GCC}(\theta)$ implies that the per–arm budgets $h_i(\epsilon)$ sum to at most $\theta$. As long as the certificate fails, Stage~1 pulls an arm with slack $>\Delta_I/3$, so these budgets are met within $B\ge\theta$ pulls. Then every suboptimal arm has $U_j\le \mathrm{OPT}_T- \frac{2\Delta_I}{3}$ and the best arm has $L_{i^\star}\ge \mathrm{OPT}_T- \frac{\Delta_I}{3}$, which implies that $L_{i^\star}\ge\max_{j\ne i^\star}U_j$. We commit to $i^\star$.
(2) If no certificate fires by $B\le T/2$, running $\textit{PTRR}_\alpha$ for $T_{\mathrm{rem}}$ steps yields an average reward within a $k^{\alpha/(1+\alpha)}$ factor of the residual optimum (up to constants). Using the fact that the best single pull is at least the average, and $\mathrm{OPT}^{\mathrm{res}}$ is at least a constant times $g^\star(T_{\mathrm{rem}})\,T_{\mathrm{rem}}$ by concavity, we get $\mathbb{E}[f_{\hat i}(T)]\ge \Omega\!\big(k^{-\alpha/(1+\alpha)}\big)\,f^\star(T)$. 
\end{proof}

\begin{proof}[Proof of 1. (certificate correctness)]
Suppose an instance $I$ satisfies \(\mathrm{GCC}(\theta)\). Since $\triangle_i(t)$ is non-increasing for all $i$ and
$\sum_{i=1}^k h_i(\Delta_I/3) \le \theta$, we know that there are at most $\theta$ total pulls (across arms) such that $\triangle_i(t_i) > \Delta_I/3$. Note that we can never have $L_{i}(t_i)  > \max_{j \ne i} U_j(t_j)$ for some $i \neq i^*$, as we know that $L_i(t) \le F_i(T) \le U_i(t)$, so this would imply $F_{i^*} (T) \leq U_{i^*}(t_{i^*}) < L_i (t_i) \leq F_{i}(T)$ (a contradiction). Since $B>\theta$ and Stage $1$ always pulls the arm $i$ that maximizes $(U_i - L_i)$, it follows that the algorithm will reach some point (namely $t = \theta$) where $\triangle_i(t_i) \leq \Delta_I/3$ for all $i$. At this point, for each $j \ne i^\star$, we have
$$
U_j(t_j) \le F_j(T) + \frac{\Delta_I}{3} \le \mathrm{OPT}_T - \Delta_I + \frac{\Delta_I}{3} = \mathrm{OPT}_T - \frac{2\Delta_I}{3}.
$$
Moreover, $i^*$ satisfies
$$
L_{i^\star}(t_{i^*})  \ge F_{i^\star}(T) - \Delta_I/3 = \mathrm{OPT}_T - \Delta_I/3.
$$ It follows that $L_{i^\star}(t_{i^*})  > \max_{j \ne i^\star} U_j(t_j)$, and therefore that the algorithm returns $i^*$ in Stage $1$.
\end{proof}

\begin{proof}[Proof of 2. (approximation fallback)]
Write $g^\star(h):=\max_i g_i(h)=\max_i f_i(t_i+h)$.
Let $T_\text{rem} := T-B$, and let $\tau' = T_\text{rem}-k$. Suppose $m' := (\tau'/T)f^{\star}(T)$. Since 
$g^{\star}(T) = \max_i f_i(T) \ge f^{\star}(T)$, we know that
$$
m' = \frac{\tau'}{T} f^{\star}(T) \le f^{\star}(T) \left( \frac{\tau'}{T} \right)^{\alpha} \le g^{\star}(T) \left( \frac{\tau'}{T} \right)^{\alpha}.
$$
Using monotonicity and the fact that $g^{\star}(T) \le 2 f^{\star}(T)$, we also know that $g^{\star}(\tau') \le 2 f^{\star}(T) = 2(T/\tau') m'$, and therefore that \ $m' \ge (\tau'/(2T)) g^{\star}(\tau')$. Since 
$B \le T/2$ and $T \ge 4k$, it follows that
$$
\frac{1}{8} g^{\star}(\tau') \le m' \le g^{\star}(T) \left( \frac{\tau'}{T} \right)^{\alpha}.
$$
Run \textit{PTRR}\(_\alpha\) for $T_\text{rem}$ steps with parameters $m'$ and $\tau'$. From the analysis of \textit{PTRR}\(_\alpha\) (Theorem \ref{thm:ptrr-alpha}), we know that
$$
\mathbb{E}\big[\mathrm{ALG}_{T_\text{rem}} \big]\ \ge\ 
\frac{\mathrm{OPT}^{\mathrm{res}}_{T_\text{rem}}}{C_\alpha c_2(k+1)^{\alpha/(1+\alpha)}},
$$
where $\mathrm{OPT}^{\mathrm{res}}_{T_{\mathrm{rem}}}$ denotes the optimal cumulative reward for $\{g_i\}$ over $T_{\mathrm{rem}}$ rounds and $C_\alpha = 2^{\alpha+2}(\alpha+1)$.

Now note that $\mathrm{OPT}^{\mathrm{res}}_{T_{\mathrm{rem}}}\ge \tfrac12g^\star(T_{\mathrm{rem}})T_{\mathrm{rem}}$, and that the algorithm’s maximum single‑pull reward dominates its average reward (Facts B.1 and B.3 in the appendix of \cite{pmlr-v272-blum25a}). Let $\hat i$ denote the arm that achieves this maximum single-pull reward, and note that

$$
\mathbb{E}\Big[\max_{t\le T_{\mathrm{rem}}}\text{reward}_t\Big]\ \ge\ \frac{1}{T_{\mathrm{rem}}}\mathbb{E}\big[\mathrm{ALG}_{T_{\mathrm{rem}}}\big]\ \ge\ \frac{g^\star(T_{\mathrm{rem}})}{2C_\alpha c_2(k+1)^{\alpha/(1+\alpha)}}.
$$

Since $f_{\hat i}(T)\ge \max_{t\le T_{\mathrm{rem}}}\text{reward}_t$ by monotonicity, we know that taking expectations gives
$$
\mathbb{E}\big[f_{\hat i}(T)\big]\ \ge\ \frac{g^\star(T_{\mathrm{rem}})}{2C_\alpha c_2(k+1)^{\alpha/(1+\alpha)}}.
$$
Monotonicity and concavity give $g^\star(T_{\mathrm{rem}})\ge (T_{\mathrm{rem}}/T)f^\star(T)\ge \tfrac12 f^\star(T)$, as $T_{\mathrm{rem}}\ge T/2$. Combining with $c_2\le 8$ from Step~1, it follows that
$$
\mathbb{E}\big[f_{\hat i}(T)\big]\ \ge\ \frac{1}{2C_\alpha c_2(k+1)^{\alpha/(1+\alpha)}}\cdot \frac{1}{2}f^\star(T)\ \ge\ \frac{1}{2^{\alpha+7}(\alpha+1)}(k+1)^{-\alpha/(1+\alpha)}f^\star(T),
$$
as desired.

Finally, in terms of cumulative reward of the chosen vs. best arms:
$$
\E{}{F_{\hat{i}}(T)} \ge \frac T2 \E{}{f_{\hat{i}}(T)} \ge \frac T2 \frac{1}{2^{\alpha+7}(\alpha+1)}(k+1)^{-\alpha/(1+\alpha)}f^\star(T) \ge \mathrm{OPT}_T \frac{1}{2^{\alpha+8}(\alpha+1)}(k+1)^{-\alpha/(1+\alpha)}
$$
\end{proof}

\subsection{Sample Complexity for tuning $\alpha,B$ }
\label{sec:hybridcomplexity}

In the previous sections, we described two hybrid algorithms which each proceed as follows: first, they spend a portion of the time attempting to solve exact best-arm identification (BAI) and then they revert to approximate best-arm identification if it is not successful. These algorithms differ only in the implementation of lines 6,7, and 8, i.e., in how the tracked variables are defined.  In general, instances could lie anywhere along the spectrum of easy-to-identify to worst-case. Is there a better way to decide how {\em much} time to allocate to attempting exact BAI and when to switch to worst-case approximate BAI? Suppose the instances we see have some stationarity. In particular, if we see historical examples of instances arising from a certain task, and we are expected to complete that task with a good algorithm in the future, we might model the instances as having been drawn from a distribution. In that case, we could hope to {\em learn} to what degree it is worth attempting exact BAI. In particular, we could hope to learn the time after which we should switch from the exact goal to the relaxed goal. We show that we can pick a switch time and a PTRR parameter with polynomially many samples from the distribution over instances that is almost as good as the optimal such time and parameter {\em on average} over the instances. Here, unlike in the previous sections, we require a distributional assumption to hold and access to samples from the distribution, and we achieve an {\em on average} rather than per instance guarantee. However, we adapt to the data distribution. \par

We study the fully hybrid algorithm family defined in Definition~\ref{def:fullyhybridfam}. Each algorithm in this family has two parameters: the first is the time $B$ at which it switches from the exact BAI goal to the approximate BAI goal, and the second specifies which $PTRR_\alpha$ algorithm it runs {\em after} switching to the approximate BAI goal. Our goal is to learn from a set of offline instances which value of $B$ and $\alpha$ gets the best performance on a new instance from the same distribution.
In order to analyze this, we can apply the same set of data-driven algorithm design tools as before. Once again, we need to understand for a fixed instance, how many possible behaviors an algorithm from the family could have in terms of loss. To this end, it remains to (1) bound $Q_\mathcal{D}$ for this problem and (2) investigate various reasonable $H$-bounded losses. We do this next.

\vspace{-4mm}

\paragraph{Bounding derandomized dual complexity in our setting.}
We bound $Q_\mathcal{D}\,$ by computing the number of possible behaviors of algorithms from $\mathcal{H}\,.$ 

\begin{lemma} \label{lemma:bdqd-hybrid}
    For the family $\mathcal{B}$ defined in Defn.~\ref{defn:alg-fam}, the improving multi-armed bandits problem, and {\em any} piecewise-constant loss function, $Q_\mathcal{D} \le k\,T^2$\,.
\end{lemma}

\vspace{-4mm}

\paragraph{Sample complexity results.}
As in Section~\ref{sec:sample-complexity}, we combine Thm.~\ref{thm:ss25main} and Lemma~\ref{lemma:bdqd-hybrid} to derive sample complexity results for generic $H$-bounded loss functions. We then instantiate it for both BAI and regret loss functions. 

\begin{theorem} \label{thm:hybrid-alg-sc}
    For the Hyperparameter Transfer setting for tuning $\alpha,B$ in Algorithm \ref{alg:hybrid} 
    with generic $H$-bounded loss, $O\left( \left(\frac{H}{\epsilon}\right)^2 (\log kT + \log \frac 1 \delta )  \right)$ instances drawn from $\mathcal{D}$ suffice to get the uniform convergence guarantee in Theorem~\ref{thm:ss25main}.\looseness-1
\end{theorem}

Finally, we instantiate the loss function for BAI. We can study ``maximum pull regret,'' and this loss is $H$-bounded for $H = m\,.$ Note that since $m$ will affect the sample complexity, we need an upper bound on it to know how many samples to draw.

\begin{definition}
    Define  ``maximum pull regret'' as $R_{mp}(T) \coloneqq \max_{i \in [k]} f_i(T) - \max_{i\in [k]} f_i(t_i)\,,$ where $t_i$ is the number of rounds for which the algorithm played arm $i\,.$\looseness-1
\end{definition}

Thus, with knowledge of an upper bound for the best pull of the best arm, we have that:

\begin{corollary}
    For the Hyperparameter Transfer setting for the improving multi-armed bandits problem optimizing for averaged regret, $N = O\left( \left(\frac{m}{\epsilon}\right)^2 (\log kT + \log \frac 1 \delta )  \right)$ instances drawn from $\mathcal{D}$ suffice to get the uniform convergence guarantee in Theorem~\ref{thm:ss25main} for the $Hybrid_{\alpha, B}$ algorithm family.\looseness-1
\end{corollary}

Thus, with polynomially many samples from a distribution over instances, we can achieve near-optimal loss on a new instance drawn from the same distribution.\looseness-1

\end{document}